\documentclass{article}

  \usepackage[preprint,nonatbib]{neurips_2026}

\usepackage[utf8]{inputenc}
\usepackage[T1]{fontenc}
\usepackage[round]{natbib}
\usepackage{amsmath,amssymb,amsthm}
\usepackage{mathtools}
\usepackage{bm}  \usepackage{xspace}  \usepackage{xcolor}
\usepackage{graphicx}
\graphicspath{{../}{../../}{../../../}{./}}
\usepackage{booktabs}
\usepackage{hyperref}
\usepackage{algorithm}
\usepackage{algorithmic}
\usepackage{subcaption}
\usepackage{multirow}

\newtheorem{theorem}{Theorem}

\newtheorem{corollary}[theorem]{Corollary}
\newtheorem{definition}[theorem]{Definition}

\newcommand{\kl}{\mathrm{KL}}
\newcommand{\rlct}{\lambda}
\newcommand{\fisher}{F}
\newcommand{\reals}{\mathbb{R}}
\newcommand{\expect}{\mathbb{E}}

 \newif\ifexpanded
\expandedfalse  

\newif\ifbodyhasproofs
\bodyhasproofsfalse

\newif\iftheoryonly
\theoryonlyfalse

\newif\ifsupp
\suppfalse
 \newif\ifanonymise
\anonymisetrue  

\newif\ifanonymousfriendly
\anonymousfriendlytrue  

\newcommand{\theorycite}{ \citet{TheoryRefNamed}}
\newcommand{\theorycitep}{ \citealp{TheoryRefNamed}}

\newcommand{\theorysrc}{ the theory paper\xspace}
\newcommand{\Theorysrc}{ The theory paper\xspace}

\newcommand{\theorypdf}{ the theory paper}
\newcommand{\theorytag}{ theory paper}

\newcommand{\lnkernelcitep}{ \citep{shirodkar2026algebraicdeaddirectionslayernorm}}
 
  \anonymisefalse
  \expandedfalse
  \anonymousfriendlyfalse

\title{Dead-Direction Signatures: A Cheap Spectral Reading of Singular Complexity}

  \author{
    Tejas Pradeep Shirodkar \\
    IIIT, Hyderabad \\
    \texttt{tejas.shirodkar@research.iiit.ac.in}
    \And
    P. J. Narayanan \\
    IIIT, Hyderabad \\
    \texttt{pjn@iiit.ac.in}
  }

\begin{document}

\maketitle

\begin{abstract}
Singular learning theory characterises the complexity of a deep network through the geometry of its loss singularities. The local learning coefficient (LLC), the standard estimator of Watanabe's real log canonical threshold (RLCT, $\lambda$), reads this geometry as an integrated Bayesian scalar through SGLD, which needs per-task calibration and $10^4$-$10^6$ forward-backward passes per checkpoint. We introduce \emph{Dead-Direction Signatures} (DDS), a family of cheap closed-form spectral readings of singular structure: each reads a network's activation matrix or per-sample-gradient Fisher-Gram at a chosen layer, replacing the SGLD posterior chain with spectral linear algebra. The readings rest on a dead-direction framework that predicts a structural correlation between activation- and Fisher-side spectra at any singular minimum, and a rank-multiplicative volume identity that single-eigenvalue monitors cannot produce: the active-volume $\log\det^{+}(G)$ slope counts the dead directions, tracking the rank-deficit $r$ across $r \in \{1,2,3,4\}$ (slope ratios $2.0, 3.1, 4.0$ at $r{=}2,3,4$ against the predicted $2,3,4$), where the smallest eigenvalue is rank-blind. On reduced-rank regression with closed-form $\lambda$, calibrated LLC recovers $\lambda$ at $99\%$ mean and the DDS observables rank-track it at the framework-predicted sign; on a non-linear modular-addition transformer DDS separates $d_{\mathrm{model}}$ across eighteen orders of magnitude where calibrated LLC at the protocol budget is rank-flat. Complementary to LLC's integrated posterior reading, DDS gives a directional, layer-local handle on a network's dead directions, read in closed form from its activation and gradient spectra.
 \end{abstract}

\section{Introduction}
\label{sec:intro}

\begin{figure}[t]
  \centering
  \includegraphics[width=\textwidth]{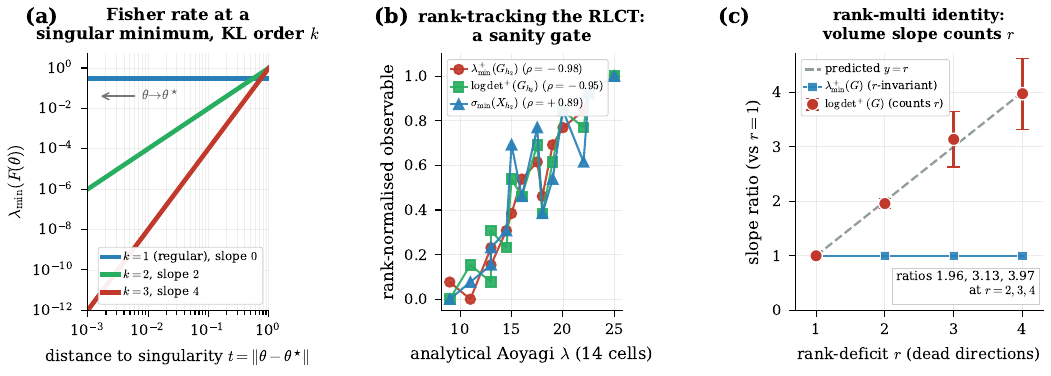}
  \caption[DDS overview.]{\textbf{Dead-Direction Signatures (DDS): the dead-direction primitive and how DDS reads it.} \textbf{(a)} The framework's central object: at a singular minimum with KL order $k$, the smallest Fisher eigenvalue decays as $\lambda_{\min}(F) = \Theta(t^{2(k-1)})$ along a dead direction; $k{=}1$ is regular (rate $0$), $k{=}2,3$ are degenerate (rates $2, 4$). \textbf{(b)} Rank-tracking on closed-form RLCT ground truth (\S\ref{sec:exp:aoyagi}): on the $14$-cell Aoyagi 2005 anchor each rate-chain DDS observable rank-tracks the analytical Watanabe RLCT $\lambda$ with the framework-predicted sign at $|\rho| \ge 0.89$ ($\rho{=}{-}0.98$ for $\lambda_{\min}^+$, $-0.95$ for $\log\det^+$, $+0.89$ for $\sigma_{\min}$). This cross-cell rank-correlation is a sanity gate: a naive $H{\cdot}r$ proxy clears it too (Fig.~\ref{fig:sanity_gate}). \textbf{(c)} The discriminating test, the rank-multiplicative volume identity (\S\ref{sec:exp:rank_multi}): the $\log\det^+(G)$ slope ratio (vs $r{=}1$) counts the dead directions, tracking the rank-deficit across $r\in\{1,2,3,4\}$ ($1.96, 3.13, 3.97$ at $r{=}2,3,4$ against the predicted $2,3,4$) while $\lambda_{\min}^+(G)$ stays rank-invariant; per-cell and per-layer detail in Fig.~\ref{fig:volume_identity}. Single-rank spectral observables are rank-blind to this identity by construction.}
  \label{fig:dds_overview}
\end{figure}
 
Singular learning theory locates the complexity of a deep network in the local geometry of its loss-singular set: the real log canonical threshold (RLCT) is the leading invariant, and it has been argued to be the right complexity measure on overparameterised models \citep{Watanabe09, Watanabe18}. In the empirical literature, the local learning coefficient \citep[LLC;][]{LauFurmanWangMurfetWei25} has emerged as the canonical estimator of Watanabe's RLCT in practice, the established choice for cross-model complexity ranking, posterior-Bayesian readings of singular structure, and developmental-stage detection during training (\S\ref{sec:related}). It is also expensive: LLC requires $10^4$--$10^6$ forward--backward passes per checkpoint via SGLD sampling, an offline budget at LLM scale and a substantial cost even on parametric singular models.

We give a complementary readout. The framework of \theorycite{} derives a directional Fisher decay rate at any singular minimum: along a dead direction of KL order $k$, the directional Fisher decays as $\Theta(t^{2(k-1)})$ along the parametric approach. The selection rule (Theorem~3) makes the bridge to Watanabe's RLCT (denoted $\lambda$) explicit: the rate exponent $2(k-1)$ recovers the directional RLCT contribution $1/(2k)$ on smooth singular fibres, in original parameter coordinates, without resolution of singularities. From this primitive a cluster of cheap spectral observables follows, drawn from two natural matrices at each layer $\ell$: the activation matrix $X_\ell \in \mathbb{R}^{N \times h}$ (rows are per-sample post-activation hidden states, $N$ samples by hidden width $h$), and the per-sample-gradient Fisher-Gram $G_\ell \in \mathbb{R}^{h \times h}$ (the empirical covariance of layer-$\ell$ gradients). The DDS observables are, by framework role: the \emph{rate} observable $\lambda_{\min}^+(G_\ell)$ (smallest positive Fisher-Gram eigenvalue, predicted by the Fisher decay theorem); the \emph{volume} observable $\log\det^+(G_\ell)$ (active-spectrum log-volume, the curvature--volume class); and the \emph{activation-side dual} $\sigma_{\min}(X_\ell)$ (the A--G dual of the rate observable, with layer-dependent sign). All three are framework-derived: the framework predicts a sign and an order of vanishing for each at any singular minimum. We collect these as \emph{Dead-Direction Signatures} (DDS): a forward-pass-cheap reading of singular structure off a network's activations and Fisher-Gram matrices in closed form. The four framework results we invoke (Fisher decay, selection rule + RLCT recovery, multi-layer KFAC bridge with A--G duality, curvature--volume rate chain) are restated with compact proof sketches in App.~\ref{app:experiments:proof_sketches}, so the empirical claims below can be checked against an explicit derivation without opening \theorypdf.

Three tests put pressure on different parts of the DDS claim; the contributions below carry the headline numbers and \S\ref{sec:experiments} the full protocols. The throughline separates the \emph{discriminating} tests, the detection and rank-deficit counting that single-eigenvalue monitors cannot run, from the cross-cell rank-tracking against closed-form RLCT that turns out to be a sanity gate a naive capacity proxy also clears; a static structural correlation then adds a cross-testbed robustness reading, and an off-the-anchor probe carries DDS to one non-linear transformer where the structural sign survives but the trajectory rate breaks.

DDS trades LLC's posterior chain for closed-form spectral reads. Calibrated LLC needs per-task SGLD calibration and a posterior chain per checkpoint ($4{,}400$ forward--backward passes at our operating budget; the literature recommends $10^4$--$10^6$). The activation-side $\sigma_{\min}(X_\ell)$ is one forward pass and an SVD, the cheap real-time read; the Fisher-side $\lambda_{\min}^+(G_\ell)$ and $\log\det^+(G_\ell)$ assemble the per-sample-gradient Fisher-Gram and take an $O(h^3)$ eigendecomposition, orders of magnitude below an LLC chain on the parametric testbeds but the costliest DDS read as width climbs. The structural correlation closes that gap at scale, recovering the Fisher-side ordering from $\sigma_{\min}^2$ at activation-side cost (\S\ref{sec:experiments}).

\paragraph{Contributions.}
\begin{enumerate}\itemsep=3pt
\item \textbf{Detection and rank-deficit counting, the framework's discriminating tests.} At a layer carrying a dead direction the smallest positive Fisher-Gram eigenvalue $\lambda_{\min}^+(G)$ collapses by orders of magnitude ($\sim\!246\times$ at the bottleneck on the closed-form anchor, while the dimension-fixed boundary layer stays flat), detecting and localising it from a single backward pass. The active Fisher-volume $\log\det^+(G_\ell)$ then \emph{counts} the dead directions: its slope scales linearly in the rank-deficit $r$ while $\lambda_{\min}^+(G_\ell)$ is $r$-invariant (Prop.~8 multi-direction generalisation). Across $7$ noisy-bridge configurations ($L \in \{4,6,8\}$, $D \in \{20,50\}$, full and mini-batch SGD, $5$ seeds, $r{=}1..4$) the $\log\det^+$ slope ratio tracks the rank-deficit (cell-wise means $2.0, 3.1, 4.0$ at $r{=}2,3,4$ against the predicted $2,3,4$), a strict prefactor-free identity invisible to single-rank spectral monitors (\S\ref{sec:exp:rank_multi}, \S\ref{sec:exp:aoyagi}).
\item \textbf{Closed-form RLCT ground truth, with the cross-cell rank-correlation as a sanity gate.} On the Aoyagi 2005 anchor (analytical $\lambda$ over $[9,25]$), a target-$\lambda$-calibrated LLC recovers $\lambda$ at $99\%$ mean and DDS rate-chain observables rank-track $\lambda$ at $|\rho| \in [0.95, 0.98]$ with the predicted sign; but calibrated LLC ($+0.98$) and a naive $H{\cdot}r$ capacity proxy ($+0.99$) clear the same bar, so the cross-cell correlation is a sanity gate, not a discriminator. On the $24$-cell Aoyagi 2024 deep-linear sweep $\lambda$ is constant above saturation ($200 = d^2/2$); DDS observables track that constancy while locked-config LLC reads the drifting local-LLC at $w^\star$, distinct from the global RLCT \citep{LauFurmanWangMurfetWei25} (\S\ref{sec:exp:aoyagi}, App.~\ref{app:experiments:dln_aoyagi_anchor}).
\item \textbf{Off-the-anchor extension and a cross-testbed robustness reading.} On a Nanda modular-addition width sweep at AdamW$+$CE ($101$ of $120$ grokked cells, $4$ widths $\times\, 30$ seeds) DDS observables span $5$--$15$ orders of magnitude and stay sign-coherent at $|\rho| \in [0.62, 0.96]$ vs $d_{\mathrm{model}}$, where calibrated LLC at the protocol budget is rank-flat and a $9\times$ SGLD budget leaves it flat. The static structural correlation $\lambda_{\min}^+(G_\ell) \propto \sigma_{\min}(X_\ell)^2$ holds at $\rho \ge +0.83$ on canonical-aligned testbeds (Barak post-grok SGD$+$MSE; Nanda SGD$+$MSE no-grok), $\rho{=}{+}1.000$ on the analytic $L{=}2$ bridge, and $\rho{=}{+}0.75$ at the AdamW$+$CE boundary (true-MC Fisher), where the trajectory-rate prediction fails but the static sign survives. The transformer probe here is single-architecture / single-task, with broader transformer coverage left as open empirical work (\S\ref{sec:exp:nanda}, \S\ref{sec:exp:universality}, App.~\ref{app:scope_map_full}).
\end{enumerate}

\paragraph{Scope.} DDS is validated on parametric singular models with closed-form RLCT and one non-linear transformer width sweep. The directional Fisher-decay theorem and the selection rule are imported from \theorysrc as cited prior results; calibrated LLC remains the canonical estimator for the posterior-Bayesian readings (singular fluctuations, WAIC, developmental-stage plateaus) that DDS does not provide.
  \section{Background and notation}
\label{sec:background}

For an analytic statistical family $\{p_\theta\}$ with Fisher information $\fisher(\theta)$ and $p_{\theta_0} = p^*$, the singular set $\Sigma = \{\theta : \det \fisher(\theta) = 0\}$ is where $\fisher$ loses rank. Singular learning theory \citep{Watanabe09} shows minima of $K(\theta) := \kl(p^* \| p_\theta)$ in singular models lie on $\Sigma$, and that $K$ vanishes along a unit direction $u$ at $\theta_0$ at a rate controlled by an integer KL order $k \ge 1$: $K(\theta_0 + tu) = c\,t^{2k} + O(t^{2k+1})$. \emph{Reading $k$:} $k{=}1$ is an ordinary quadratic minimum (Fisher full rank along $u$, the regular case); $k{=}2$ is a quartic-flat dead direction along which $\fisher$ vanishes as $\Theta(t^2)$; larger $k$ means flatter, more degenerate. The rate exponent $2(k-1)$ thus controls how fast $\fisher$ collapses as the parameter approaches the singular minimum along $u$. The RLCT $\rlct$ decomposes into a directional $1/(2k)$ contribution from the transversal dead direction plus a tangential term; the selection rule (Thm.~3) recovers that $1/(2k)$ piece from our rate exponent $2(k-1)$ in original coordinates, with the full normal-form derivation in App.~\ref{app:experiments:proof_sketches}.

For a layer with weights $W_\ell$, the KFAC factorisation \citep{MartensGrosse15} gives $F_\ell \approx A_\ell \otimes G_\ell$ with $A_\ell := \expect[X_{\ell-1} X_{\ell-1}^\top]$ and $G_\ell := \expect[\delta_\ell \delta_\ell^\top]$; thus $\lambda_{\min}(F_\ell) = \lambda_{\min}(A_\ell)\,\lambda_{\min}(G_\ell)$. We use $A_\ell, G_\ell$ as canonical names. Other recurring symbols: $L$ depth, $\ell$ layer index, $h$ hidden width; $X_\ell \in \reals^{N \times h}$ activation matrix; $u$ dead direction (Def.~\ref{def:dead_direction}).
 \subsection{Related work}
\label{sec:related}

\paragraph{Singular learning theory and LLC.} Watanabe's singular learning theory \citep{Watanabe09, Watanabe18} locates the complexity of an overparameterised model in the local geometry of its loss-singular set, with the real log canonical threshold (RLCT) as the leading invariant. The local learning coefficient \citep[LLC;][]{LauFurmanWangMurfetWei25} is the canonical empirical estimator: SGLD sampling on the posterior near a singular minimum recovers Watanabe's RLCT in practice. LLC has been extended to track developmental phase transitions during training \citep{HooglandWangFarrugiaRoberts24} and to decompose into per-component restricted-LLC \citep{WangHoogland24}. We position DDS as a complementary readout on the same singular structure: where LLC integrates over the local posterior to return a Bayesian-WBIC-anchored quantity, DDS reads the spectra of the activations and Fisher-Gram in closed form at a single checkpoint. The two approaches answer different questions about the same singular geometry, and the cost ordering (the activation-side $\sigma_{\min}$ orders of magnitude below calibrated LLC per checkpoint on the testbeds here) lets DDS run at scales and cadences SGLD cannot match.

\paragraph{Closed-form RLCT lineage.} Aoyagi and Watanabe's closed-form RLCT formulas for reduced-rank regression \citep{AoyagiWatanabe05} and the recent extension to deep-linear nets \citep{Aoyagi24} provide ground-truth $\lambda$ values on parametric singular families, enabling validation against analytical RLCT rather than another estimator. The selection rule of \theorycitep{} (Thm.~3) is the formal bridge from our directional Fisher exponent to Watanabe's RLCT: the rate $2(k-1)$ recovers the directional contribution $1/(2k)$ on smooth singular fibres in original parameter coordinates without resolution of singularities. Our use of Aoyagi-style closed-form RLCT testbeds in \S\ref{sec:exp:aoyagi} is the empirical anchor for this bridge.

\paragraph{Spectral monitoring on activations and weights.} A parallel line summarises activation- or Fisher-side spectra as generalisation-correlated signatures: participation ratio on FFN activation covariances \citep{Jha2026NerVE}, streaming spectral entropy and top spectral gap on windowed activations \citep{Ettori2025EigenTrack}, the normalised Shannon entropy of \citet{Truong2026SpectralEntropy}, and the parameter-update top consecutive-singular-value ratio of \citet{Xu2026SpectralEdge}. Marchenko--Pastur deviations on pretrained weights \citep{StaatsThammRosenow24} and rank-collapse on pure-attention chains \citep{DongCordonnierLoukas21, NociAnagnostidisBiggio22} read the same spectra at still other points; \citet{BoixAdseraLittwinAbbe23} report training-trajectory rank evolution, a different object from the singular-minimum rank deficit DDS targets. DDS reads the Fisher-Gram bottom directly through the rate-chain observables ($\lambda_{\min}^+(G)$, $\log\det^+(G)$) and the activation-side dual $\sigma_{\min}(X_\ell)$, with framework-derived sign and per-layer prescription tied to the directional RLCT contribution $1/(2k)$ via the Fisher decay theorem and A--G duality (Theorem~\ref{thm:structural_correlation}, full proof App.~\ref{app:structural_correlation_proof}). The prior signatures correlate spectra with generalisation, and we report them as complementary observation points without a quantitative head-to-head. Random-matrix characterisations of the Fisher / Hessian spectrum at large width \citep{PenningtonWorah18, KarakidaAkahoAmari19, KarakidaAkahoAmari21} address the bulk at random init under mean-field or Gaussian-entry assumptions; the DDS observables read the trained-checkpoint bottom spectrum at a chosen layer, used both as static per-checkpoint structural identities and as directional rates along the singular-minimum approach. We use random-matrix arguments defensively for the $n/d$ calibration of the measurement protocol.

\paragraph{Algebraic dead directions in LayerNorm.} An instantiation of the same dead-direction framework gives the cheapest read of the primitive: for LayerNorm models the dead direction is fixed algebraically by the normalization affine, $v^\star = \gamma^{-1}/\|\gamma^{-1}\|$, recoverable from the layer parameter with no forward pass \lnkernelcitep. That algebraic read is exact for LayerNorm models. DDS's spectral observables ($\sigma_{\min}(X_\ell)$, the Fisher-Gram spectrum) are architecture-general and read the same primitive at a forward-pass or Fisher cost, two points on a cost--generality spectrum for the same dead-direction object.

\paragraph{KFAC and the Fisher estimand.} Our measurement protocol uses the KFAC factorisation \citep{MartensGrosse15, GrosseMartens16, EschenhagenImmerTurner23} with the Fisher estimand matched to the task: the empirical (true-label) Fisher on the MSE testbeds, where it coincides with the model Fisher at the minimum, and the true-MC (model-sampled) Fisher on the cross-entropy testbed, where the empirical Fisher is an unreliable proxy \citep{KunstnerHennigBalles19} that collapses at zero loss. \citet{GeorgeLaurentBouthillier18} introduced the eigenbasis variant. The KFAC factorisation's analytical use (the A--G duality, Cor.~25) is established in \theorycitep{}; we use it as a measurement primitive only.
 
\section{The structural correlation}
\label{sec:predictions}

At any singular minimum the activation-side and Fisher-side bottoms of the spectrum vanish together at a predictable ratio. The coupling is the single structural prediction DDS rests on, stated below. The supporting framework (Fisher decay theorem, selection rule, multi-layer KFAC bridge, A--G duality) lives in \theorysrc and is used here as a cited prior.

\begin{definition}[Dead direction]
\label{def:dead_direction}
A unit direction $u \in \reals^d$ is a \emph{dead direction} at $\theta_0$ if $u^\top \fisher(\theta(t)) u \to 0$ as $t \to 0$ along $\theta(t) := \theta_0 + tu$. The KL order along $u$ is the integer $k \ge 1$ with $K(\theta(t)) = c\,t^{2k} + O(t^{2k+1})$, $c > 0$. (Unrelated to the activation-level ``dead ReLU'' phenomenon despite the lexical overlap.)
\end{definition}

\begin{theorem}[Structural correlation between Fisher and activation spectra]
\label{thm:structural_correlation}
\label{pred:structural_correlation}
At any singular minimum of an $L$-layer network with KFAC-decomposable Fisher information $F_\ell \approx A_\ell \otimes G_\ell$, along a canonical-aligned approach $\theta(t)$ to the minimum:
\[
\lambda_{\min}^+(G_\ell(\theta(t))) \;=\; \Theta(t^{2(L-\ell)}), \qquad \sigma_{\min}(X_\ell(\theta(t)))^2/N \;=\; \Theta(t^{2\ell}),
\]
at every layer $\ell$. Both quantities vanish as $t \to 0$ as positive powers of $t$, so they are co-monotonic along the approach and the Spearman rank correlation between them tends to $+1$ in the asymptotic regime.
\end{theorem}

\begin{proof}[Sketch]
By the multi-layer KFAC bridge (Theorem~21 of \theorysrc), the smallest non-zero eigenvalue of the Fisher-Gram at layer $\ell$ satisfies $\lambda_{\min}^+(G_\ell(\theta(t))) = \Theta(t^{2(L-\ell)})$ along a canonical-aligned approach. By the A--G duality corollary (Cor.~25), the activation-side dual is $\lambda_{\min}(A_{\ell+1}(\theta(t))) = \Theta(t^{2\ell})$, and $\sigma_{\min}(X_\ell)^2/N \to \lambda_{\min}(A_{\ell+1})$ almost surely by the strong law (the rows of $X_\ell$ are iid under the Gaussian-isotropic input assumption of Theorem~2). Both Fisher and activation bottoms are positive powers of $t$ with strictly positive exponents; rank correlation between any two such monotone trajectories tends to $+1$ in the asymptotic limit. Full proof in App.~\ref{app:structural_correlation_proof}.
\end{proof}

The three DDS observables (smallest activation singular value $\sigma_{\min}(X_\ell)$, smallest positive Fisher-Gram eigenvalue $\lambda_{\min}^+(G_\ell)$, log-volume $\log\det^+(G_\ell)$) read this coupled geometry from the activation or the gradient side. \S\ref{sec:experiments} validates the correlation across regime cells, with a deterministic limit on the analytic $L{=}2$ bridge and a boundary check on AdamW$+$CE where canonical alignment is violated.
  \section{Experiments}
\label{sec:experiments}

DDS detects a dead direction and counts how many directions are dead, two reads a single-eigenvalue monitor cannot give. \S\ref{sec:exp:aoyagi} pins detection and rank-tracking to closed-form RLCT ground truth, where the cross-cell rank-correlation turns out to be a sanity gate that even a naive capacity proxy clears. \S\ref{sec:exp:rank_multi} is the discriminating quantitative test, the counting identity. \S\ref{sec:exp:nanda} carries the observables off the anchor to a non-linear transformer at AdamW$+$CE, and \S\ref{sec:exp:universality} reports the static structural correlation as a robustness reading. Tab.~\ref{tab:dds_summary} summarises the per-testbed Spearman $\rho$; per-claim assumptions and validation loci are in App.~\ref{app:experiments:theory_scope_map}.

\begin{table}[t]
\centering\small
\caption{DDS observables on three testbeds, with framework role and cross-cell Spearman $\rho$ vs ground-truth complexity (Aoyagi $\lambda$ on the closed-form testbeds; $d_{\mathrm{model}}$ on Nanda). Role: \emph{rate} (Fisher decay theorem), \emph{vol} (active-spectrum log-volume / curvature class), \emph{dual} (A--G duality, layer-dependent sign by construction). All three observables are framework-derived: the framework predicts a sign at the dimension-fixed boundary layer for each. On the DLN sweep the analytical $\lambda$ is constant above saturation (in both $h$ and rank deficit), so its column is a tied-rank artifact ($\rho \approx 0$); that sweep tests constancy and the local-vs-global LLC distinction. Per-cell hyperparameters and reproducibility notes in App.~\ref{app:experiments:dln_aoyagi_anchor}.}
\label{tab:dds_summary}
\setlength{\tabcolsep}{6pt}
\begin{tabular}{l c r r r}
\toprule
Observable & Role & Aoyagi & DLN & Nanda \\
& & ($14$ cells) & ($24$ cells; $\lambda$ const.) & ($n{=}101$) \\
\midrule
$\lambda_{\min}^+(G)$ & rate & $-0.98$ & $-0.06$ & $-0.88$ \\
$\log\det^+(G)$ & vol  & $-0.94$ & $-0.09$ & $-0.97$ \\
$\sigma_{\min}(X_\ell)$ & dual & $+0.89$ & $-0.21$ & $-0.65$ \\
\bottomrule
\end{tabular}
\end{table}
 
\subsection{Closed-form RLCT anchor (Aoyagi reduced-rank regression)}
\label{sec:exp:aoyagi}

If DDS tracks Watanabe's RLCT, the first thing to demand is agreement with the analytical $\lambda$ where ground truth is closed-form. Reduced-rank regression (RRR) is the canonical singular family with closed-form RLCT: it factors a teacher matrix $M^\star \in \mathbb{R}^{N \times M}$ of rank $r$ through a hidden bottleneck of width $H$, so the trained model is $W_2 W_1$ with $W_1 \in \mathbb{R}^{H \times M}$, $W_2 \in \mathbb{R}^{N \times H}$. The Aoyagi--Watanabe closed form \citep{AoyagiWatanabe05} gives the RLCT $\lambda$ as a function of $(M, N, H, r)$ alone (input dim $M$, output dim $N$, bottleneck $H$, truth rank $r$), providing analytical ground truth against which any RLCT estimator can be calibrated. Our \emph{Aoyagi 2005 anchor} fixes $M{=}10, N{=}5$ and sweeps $H \in \{2,3,4,5\}$ and $r \in \{1, \ldots, \min(N, H)\}$, giving $14$ realisable cells (truth rank $\le$ model capacity) at $\lambda \in [9, 25]$. The \emph{Aoyagi 2024 deep-linear net} (DLN) sweep \citep{Aoyagi24} extends this to a $3$-matrix product $W_3 W_2 W_1$ with shared input/output dim $d{=}20$ and hidden width $h \in \{16, 20, 24, 32, 64, 128\}$; the truth rank is $20 - \mathrm{rd}$ where the \emph{rank deficit} $\mathrm{rd} \in \{1,2,3,4\}$ counts how many directions the truth is degenerate in. The \emph{saturation threshold} is $h \ge \mathrm{rank}(M^\star)$: at and above it the model has enough capacity to express the truth and the closed-form RLCT becomes provably independent of $h$, fixed by the truth's rank alone. With $d{=}20$, the $h \in \{20, 24, 32, 64, 128\}$ cells sit at or above saturation across all $\mathrm{rd}$; the $h{=}16$ cell sits below it for $\mathrm{rd}<4$. The constancy-in-$h$ comparison below uses the above-saturation cells, where the analytical RLCT is constant by design; the unrealisable $h{=}16, \mathrm{rd}<4$ subset is excluded from the constancy claim. Calibrated LLC is retrofitted on every cell with a locked SGLD config (per-testbed $5\times 5$ grid + $9\times$ saturation gate); see App.~\ref{app:experiments:dln_aoyagi_anchor} for the lock procedure.

\begin{figure}[t]
  \centering
  \includegraphics[width=0.82\textwidth]{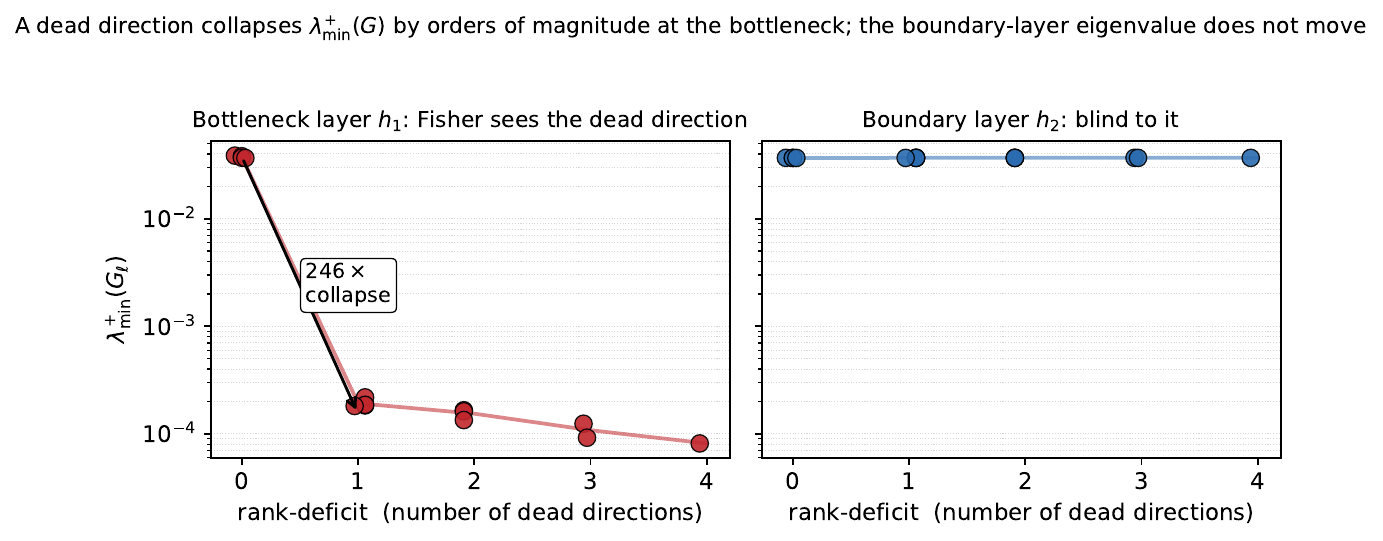}
  \caption[Dead-direction detection cliff.]{\textbf{A dead direction collapses the Fisher-Gram eigenvalue at the bottleneck.} On the Aoyagi 2005 reduced-rank-regression anchor ($14$ cells, $\sigma{=}0.1$), the smallest positive Fisher-Gram eigenvalue $\lambda_{\min}^{+}(G)$ at the bottleneck layer $h_1$ drops $\sim\!246\times$ the instant a dead direction is present (rank-deficit $\ge 1$), while at the dimension-fixed boundary layer $h_2$ it stays flat to $0.3\%$. The collapse detects the dead direction and localises it to the bottleneck, read from a single backward pass.}
  \label{fig:detection_cliff}
\end{figure}
 
Detection comes first. At the bottleneck layer the smallest positive Fisher-Gram eigenvalue $\lambda_{\min}^+(G)$ collapses $\sim\!246\times$ the instant a dead direction is present, while the dimension-fixed boundary layer stays flat to $0.3\%$ (Fig.~\ref{fig:detection_cliff}); the collapse is one backward pass, and it localises the dead direction to the bottleneck. The magnitudes are trustworthy against the closed form: under a target-$\lambda$ SGLD calibration, calibrated LLC recovers the analytical $\lambda$ at $99\%$ mean across the $14$ cells (App.~\ref{app:experiments:llc_calibration}). \emph{Cross-cell rank-tracking is a sanity gate.} The rate-chain observables $\lambda_{\min}^+(G)$ and $\log\det^+(G)$ rank-track $\lambda$ at $|\rho| \in [0.95, 0.98]$ on the dimension-fixed boundary layer with the framework-predicted sign, and $\sigma_{\min}(X_\ell)$ reads the truth-rank at $\rho = +0.89$; but calibrated LLC ($+0.98$) and a naive $H{\cdot}r$ capacity proxy ($+0.99$) clear the same bar (Fig.~\ref{fig:sanity_gate}), so the cross-cell correlation is a passing bar every complexity-monotone observable meets, not a discriminator (the rankings hold across a $4\times$ $\sigma$-noise range, App.~\ref{app:experiments:dln_aoyagi_anchor}). On the Aoyagi 2024 DLN sweep, $\lambda$ is constant in $h$ above saturation and DDS observables track that constancy, while calibrated LLC at the locked-config budget reads the local-LLC at the trained $w^\star$, which \citet{LauFurmanWangMurfetWei25} distinguish from the global RLCT.

\begin{figure}[t]
  \centering
  \includegraphics[width=0.82\textwidth]{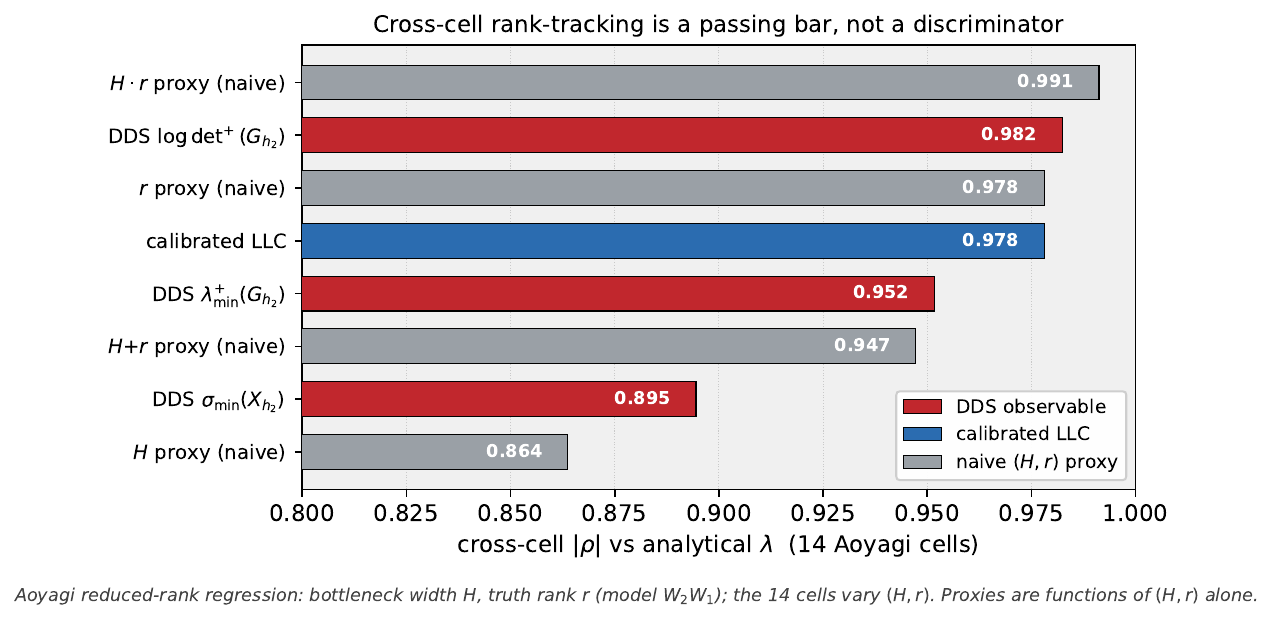}
  \caption[Cross-cell Spearman is a sanity gate.]{\textbf{The cross-cell rank-correlation is a sanity gate, not a discriminator.} On the Aoyagi 2005 anchor ($14$ cells), every complexity-monotone observable clears the cross-cell $|\rho|$ against the analytical $\lambda$: the DDS observables, calibrated LLC, and a naive $H{\cdot}r$ capacity proxy, which in fact scores highest ($0.99$). The discriminating evidence is the rank-multiplicative volume identity (Fig.~\ref{fig:volume_identity}), which the proxy and single-eigenvalue monitors cannot reproduce. $H$ is the bottleneck width, $r$ the truth rank.}
  \label{fig:sanity_gate}
\end{figure}
 
\subsection{Rank-multiplicative volume identity (discriminating test)}
\label{sec:exp:rank_multi}

Sign and rank-correlation tests are a passing bar; the discriminating test is a quantitative identity any monotone observable would fail. The framework provides one: at a singular minimum with rank-deficit $r$, $\log\det^+(G_\ell)$ slope scales linearly in $r$ while $\lambda_{\min}^+(G_\ell)$ slope is $r$-invariant (Prop.~8 multi-direction generalisation). The slope at rank-deficit $r$ is $r$ times the rank-$1$ slope, with no free prefactor. Single-eigenvalue spectral monitors are rank-blind to this identity by construction.

\begin{figure}[t]
  \centering
  \includegraphics[width=\textwidth]{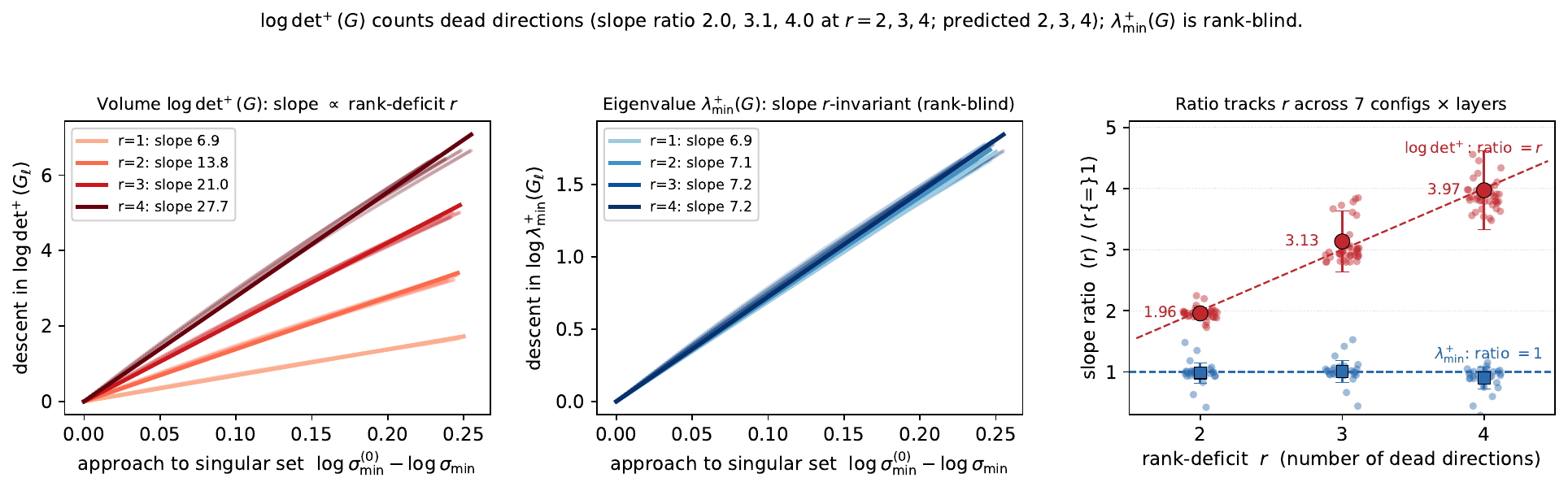}
  \caption[Rank-multiplicative volume identity.]{\textbf{The volume observable counts dead directions; the smallest eigenvalue cannot.} \textbf{(A)} Normalised descent of $\log\det^{+}(G_\ell)$ toward the singular set at an interior layer (deep-linear noisy bridge, $L{=}4$, $D{=}20$, layer $h_1$), for rank-deficit $r{=}1,\dots,4$: the slope fans to $1\times/2\times/3\times/4\times$ the rank-$1$ slope ($6.9/13.8/21.0/27.8$). \textbf{(B)} The smallest positive eigenvalue $\log\lambda_{\min}^{+}(G_\ell)$ descends at the same rate for every $r$ (rank-blind, slopes $\approx 7$). \textbf{(C)} Across $7$ noisy-bridge configurations $\times$ layers ($L \in \{4,6,8\}$, $D \in \{20,50\}$, full and mini-batch SGD, $5$ seeds, $r{=}1..4$), the $\log\det^{+}$ slope ratio tracks the rank-deficit (means $2.0, 3.1, 4.0$ at $r{=}2,3,4$ against the predicted $2,3,4$) while the $\lambda_{\min}^{+}$ ratio stays at $1$. The ratio is a strict, prefactor-free identity; a single-eigenvalue monitor is rank-blind to it by construction.}
  \label{fig:volume_identity}
\end{figure}
 
We measure across $7$ noisy-bridge configurations ($L \in \{4, 6, 8\}$, $D \in \{20, 50\}$, full and mini-batch SGD, $\sigma{=}0.1$, $5$ seeds each), running each trajectory at rank-deficit $r{=}1,\dots,4$ ($M^\star = \mathrm{diag}(1,\ldots,1,0,\ldots)$ with the last $r$ entries zero). Cells use depth-controlled init $t_0(L) = (1/16)^{1/L}$ to hold the initial dead-direction product $t_0^L = 1/16$ constant across $L$, isolating the rate prediction from the asymptotic-regime accessibility (Rem.~11). The $\log\det^+$ slope ratio tracks the rank-deficit across all four ranks (Fig.~\ref{fig:volume_identity}): cell-wise means $2.0, 3.1, 4.0$ at $r{=}2,3,4$ against the predicted $2,3,4$, the $r{=}2$ ratio inside the $\pm 10\%$ falsification band, while the matching $\lambda_{\min}^+$ slopes stay $r$-invariant (rate-chain $\rho{=}0.985$). Per-cell and per-layer breakdowns are in App.~\ref{app:experiments:bridge_structural_correlation}.

\subsection{Off-the-anchor extension on a non-linear transformer}
\label{sec:exp:nanda}

The closed-form anchor pins DDS to ground truth in the canonical-aligned regime; the natural next question is what survives off it. We run the same observables on a non-linear architecture trained with AdamW$+$CE, where the canonical-alignment hypothesis is violated. This is the cleanest single-axis stress test. We use the Nanda modular-addition setup \citep{NandaChanLieberum23} at $d_{\mathrm{model}} \in \{32, 64, 128, 256\}$ ($30$ seeds per width, $101$ of $120$ grokked) under AdamW$+$CE, with per-width SGLD calibration locked via a $5\times 5$ grid plus $9\times$ saturation gate. The framework predicts (Theorems~2, 9, 25): wider model $\to$ more degenerate Fisher $\to$ smaller $\lambda_{\min}^+(G)$, more negative $\log\det^+(G)$, smaller $\sigma_{\min}$. Every DDS observable spans $5$--$15$ orders of magnitude across the four widths, remains monotone in $d_{\mathrm{model}}$, and is sign-coherent at $|\rho| \in [0.62, 0.96]$, with $\log\det^+(G_{h_2})$ at the high end ($|\rho|{=}0.965$, $35\%$ rel-std). Calibrated LLC at the protocol $4{,}400$-step budget is rank-flat on the same sweep ($\rho{=}{-}0.07$), and a $9\times$-budget re-run does not lift it (App.~\ref{app:experiments:dln_aoyagi_anchor}). The dynamic range and sign-coherence are what the $4$-distinct-widths setup can test; the cross-cell Spearman has heavy tied-rank dependence, and this is one architecture on one task. Broadening the rate and structural-correlation reads (attention-only chains, deeper transformers, more widths, other algorithmic tasks) is open empirical work.

\begin{figure}[t]
  \centering
  \includegraphics[width=0.72\textwidth]{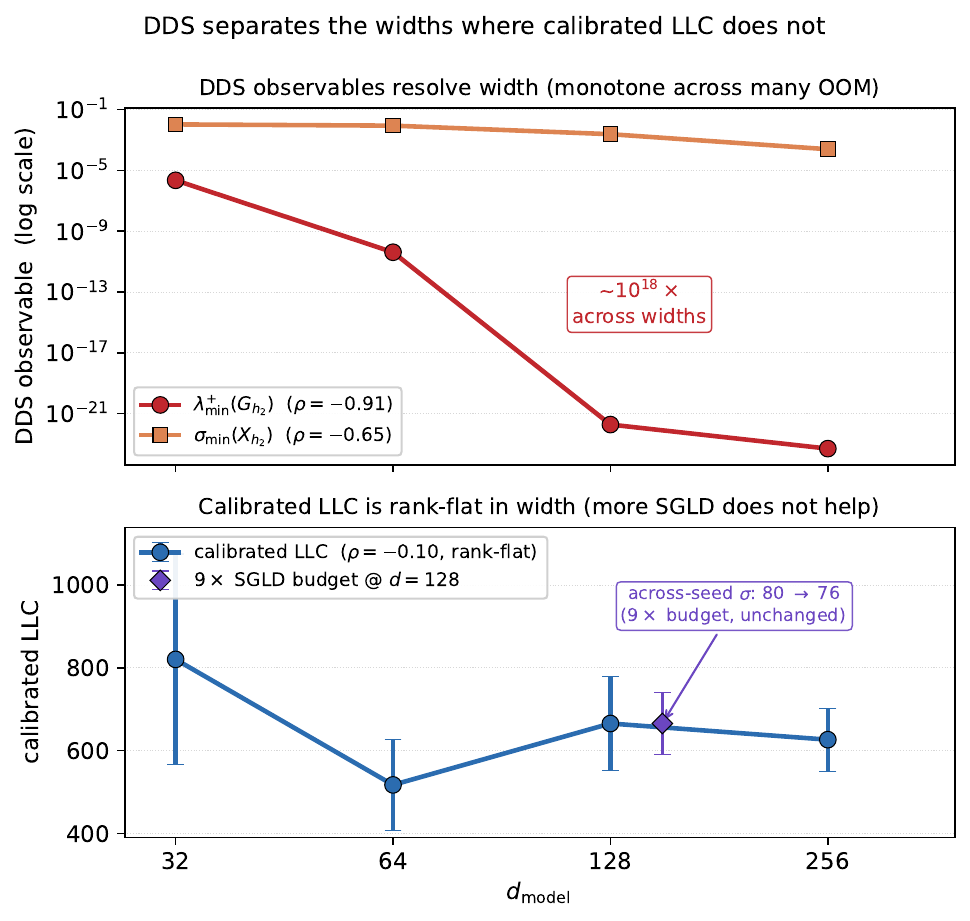}
  \caption[DDS resolves width where LLC is flat.]{\textbf{DDS separates the widths where calibrated LLC does not.} Nanda modular-addition width sweep (AdamW$+$CE, $101$ grokked cells of $4$ widths $\times\, 30$ seeds). \textbf{(Top)} DDS observables descend monotonically with $d_{\mathrm{model}}$: $\lambda_{\min}^{+}(G)$ spans $\sim\!10^{18}$ ($\rho{=}{-}0.91$), $\sigma_{\min}$ is monotone ($\rho{=}{-}0.65$). \textbf{(Bottom)} Calibrated LLC at the $4{,}400$-step protocol budget is rank-flat ($\rho{=}{-}0.10$); a $9\times$ SGLD budget at $d{=}128$ leaves the across-seed spread essentially unchanged ($\sigma\,80 \to 76$), pinning the flatness to per-seed initialisation variance rather than an under-sampled chain.}
  \label{fig:width_vs_llc}
\end{figure}
 
\subsection{Static structural correlation (robustness reading)}
\label{sec:exp:universality}

\begin{figure}[t]
\centering
\includegraphics[width=\textwidth]{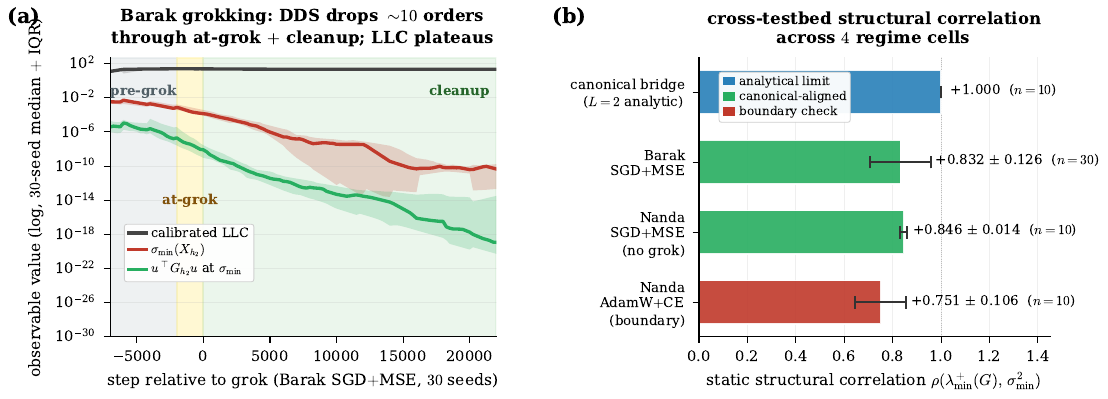}
\caption{\textbf{DDS resolves singular structure dynamically; the static identity holds across regime cells.} \textbf{(a)} Barak sparse-parity grokking trajectory ($30$ seeds, SGD$+$MSE, $240$k steps; $30$-seed median $+$ IQR band). Through the val\_acc-anchored phases (pre-grok, at-grok, cleanup) the activation-side $\sigma_{\min}(X_{h_2})$ drops $\sim\!7$ orders and the Fisher-side $u^{\top} G_{h_2} u$ at the dead direction drops $\sim\!14$ orders, while calibrated LLC (top trace) plateaus. \textbf{(b)} Static structural correlation $\rho(\lambda_{\min}^{+}(G), \sigma_{\min}^{2})$ across $4$ regime cells (canonical bridge, Barak SGD$+$MSE, Nanda SGD$+$MSE no-grok, Nanda AdamW$+$CE; per-cell $\rho$ and seed counts in the text and Tab.~\ref{tab:dds_summary}). All four read at $\rho \ge {+}0.75$, the AdamW$+$CE cell a boundary check at the canonical-alignment violation.}
\label{fig:nanda_universality}
\end{figure}

The static structural identity $\lambda_{\min}^+(G_\ell) \propto \sigma_{\min}(X_\ell)^2$ couples the Fisher-side and activation-side spectra at any singular minimum (Theorem~21, Cor.~25; full proof in App.~\ref{app:structural_correlation_proof}). It supplies the same Fisher-side ordering as $\lambda_{\min}^+(G_\ell)$ at activation-side cost, which matters at LLM widths where the Fisher-side full-spectrum is expensive (cost ordering below). The discriminating quantitative test lives on the volume side (\S\ref{sec:exp:rank_multi}); here we report the static identity as a robustness reading across testbeds. \emph{Two canonical-aligned non-trivial testbeds.} $\rho = +0.832 \pm 0.126$ on Barak sparse parity \citep{BarakEdelmanGoelKakadeMalachZhang22} SGD$+$MSE \emph{Phase A} ($30$ seeds; post-grok descent into the singular minimum), and $\rho = +0.846 \pm 0.014$ on the Nanda $1$-block transformer SGD$+$MSE no-grok cell ($10$ seeds). \emph{Deterministic check.} On the analytic $L{=}2$ canonical bridge (linear, deterministic GF, MSE; $10$ seeds), $\rho = +1.000$, the analytical limit where $\rho$ reduces to algebra. The theorem predicts $\rho \to +1$ in the strict asymptotic limit; finite-$t$ deviations and finite-$N$ Marchenko--Pastur noise on $\widehat{A}_{\ell+1}$ shrink the observed correlation, and the ordering matches: deterministic at the limit, Barak loosened by finite $N/D \approx 2$ and the $30$-seed dispersion. \emph{Boundary-of-applicability check.} Nanda $1$-block transformer AdamW$+$CE (groks; $10$ seeds, true-MC Fisher) gives $\rho = +0.751 \pm 0.106$. On AdamW$+$CE the canonical-alignment hypothesis is violated (Adam's diagonal preconditioner is non-equivariant under CE's logit-shift symmetry; Rem.~80), so the trajectory-level prediction does not strictly apply; the residual coupling rests on the KFAC factorisation $F_\ell \approx A_\ell \otimes G_\ell$ at the singular minimum itself, independent of the optimiser. \emph{Slope-2 prediction.} On Barak Phase A, $4$ of $5$ seeds give trajectory-rate slope $\bar x = 2.06 \pm 0.20$ ($R^2 \in [0.87, 0.97]$), agreeing with the noisy-bridge sweep (\S\ref{sec:exp:rank_multi}); higher $n/D$ compresses the geometric event below save-cadence resolution (App.~\ref{app:experiments:grokking}).

\paragraph{Cost ordering.} The four observables sit at different training-loop cadences. $\sigma_{\min}(X_\ell)$ is one forward pass plus an SVD ($O(Nh^2)$) and runs at \emph{real-time} cadence at any width we tested. $u^\top G u$ at a known direction is one FBP ($O(h)$), \emph{checkpoint} cadence. The Fisher-side full-spectrum observables $\lambda_{\min}^+(G_\ell)$ and $\log\det^+(G_\ell)$ assemble the per-sample-gradient Fisher-Gram and take its $O(h^3)$ eigendecomposition, \emph{periodic} cadence on small models and \emph{offline} at LLM width. Calibrated LLC's $4{,}400$-FBP SGLD chain is offline at any scale we tested. On a Barak post-grok checkpoint ($h{=}500$), in wall-clock, the Fisher-side observables are $\sim\!300\times$ cheaper than calibrated LLC and $\sigma_{\min}$ is $\sim\!130\times$ cheaper (App.~\ref{app:experiments:wall_clock}). At LLM widths the ordering shifts: $\sigma_{\min}$ and $u^\top G u$ scale roughly linearly in $h$ and stay cheap, while the Fisher-side eigendecomposition scales as $O(h^3)$ and becomes the costliest DDS read. The structural correlation $\rho(\lambda_{\min}^+(G), \sigma_{\min}^2)$ (\S\ref{sec:exp:universality}) recovers the Fisher-side ordering through $\sigma_{\min}$ at activation-side cost; per-cadence recipes and per-width wall-clock numbers are in App.~\ref{app:experiments:observable_protocols}.
 \section{Discussion}
\label{sec:discussion}

\paragraph{What we showed.} DDS detects and counts dead directions through cheap closed-form spectral reads of a network's activations and Fisher-Gram, complementary to calibrated LLC; the per-checkpoint cost ordering and how it shifts with width are in \S\ref{sec:experiments}.

The three DDS observables read singular structure from complementary sides of the same per-layer object. $\sigma_{\min}(X_\ell)$ reads activation rank-deficiency, $\lambda_{\min}^+(G_\ell)$ reads gradient-Gram rank-deficiency, and $\log\det^+(G_\ell)$ reads the active gradient-Gram log-volume. The framework's structural correlation states that activation and gradient rank-deficiency reflect the same singular geometry coupled through the KFAC factorisation, which anchors the three observables to one another. The volume-side rank-multiplicative identity is the quantitative structural prediction this anchoring produces; it cannot be reduced to sign or rank correlation, and \S\ref{sec:experiments} bears this out as the discriminating test where the cross-cell rank-tracking is only a sanity gate.

\paragraph{Alternative explanations considered.} \emph{(i) Calibrated-LLC width-flatness / drift is a budget artefact.} Ruled out by a $9\times$-larger SGLD budget at $d{=}128$ on the Nanda sweep: within-cell SGLD std falls $1.50\times$ but across-seed std falls only $1.05\times$, so per-seed initialisation variance dominates SGLD measurement variance at this calibration cell, and the drift on the DLN sweep is consistent with the App.~I local-vs-global RLCT distinction \citep{LauFurmanWangMurfetWei25} (App.~\ref{app:experiments:dln_aoyagi_anchor}). \emph{(ii) Transformer cross-architecture extension is contingent on grokking.} The sign-coherent $|\rho| \in [0.62, 0.96]$ reading is on the $101$ grokked cells; whether DDS extends to non-grokked transformers remains open.

\paragraph{Where DDS applies and where it doesn't.} \emph{(a) Static observables apply at any singular minimum, conditional on one being present.} The three DDS readings are algebraic identities of the per-layer geometry; they are testbed-broad across every optimiser, loss, and architecture we tested. They go quiet on data without rank-deficient teachers or training that does not reach a singular minimum: there is nothing dead to read. The geometry is a property of the problem the network is solving. \emph{(b) Trajectory rate-fits additionally require an actual descent.} The rank-multiplicative identity and the per-layer rate ladder need canonical alignment, a theorem-compatible preconditioner (SGD on $G$-invariant metrics), and the trajectory entering the asymptotic regime where $t \to 0$. The third precondition is data-and-training-dependent: a network that never settles into the singular minimum (training stopped short, regularisation absent, optimiser parked in a non-singular basin) leaves the rate undefined. \emph{(c) Architectural reach tested.} The rank-multi identity is validated on $7$ noisy-bridge configurations; the transformer cross-architecture probe is one task on a single-block transformer with $4$ widths. Width sweeps on attention-only chains, deeper transformers, and other small-algorithmic tasks are open empirical work. \emph{(d) Posterior-Bayesian summaries.} The local-posterior summaries (Watanabe's multiplicity $m$, the singular fluctuation $\nu$, posterior-WAIC, and the developmental-stage trajectory plateaus \citep{HooglandWangFarrugiaRoberts24}) integrate over the posterior and read a complementary slice of the same singular structure. \emph{(e) Methodological caveats.} The dimension-fixed boundary layer prescription (App.~\ref{app:experiments:dln_aoyagi_anchor}) was identified on the Aoyagi anchor and applied across the DLN sweep, the Nanda extension, and the structural-correlation testbeds; on architectures with no obvious bottleneck the rule must be tested before being trusted.

\paragraph{Falsifiers.} Three checks would force revision. \emph{(F1)} An Aoyagi-anchor cell where rate-chain observables carry the wrong sign against $\lambda$; we find none across the $14$ cells where $\lambda$ varies. \emph{(F2)} A canonical-aligned noisy-bridge cell at $L \ge 4$ where the $\log\det^+$ slope ratio at rank-deficit $r$ deviates from $r$ by more than $\pm 10\%$; across $r\in\{1,2,3,4\}$ the $7$ configurations read $1.96, 3.13, 3.97$ at $r{=}2,3,4$ against the predicted $2,3,4$. \emph{(F3)} A canonical-aligned-regime architecture where $\rho(\lambda_{\min}^+(G_\ell), \sigma_{\min}(X_\ell)^2)$ across checkpoints reads $\le 0$.

\paragraph{Outlook.} DDS gives the first directional, parameter-local handle on singular complexity. Three follow-ons sit close: rank-collapse pretraining monitors, LoRA placement along measured dead directions, and complexity ranking on architectures where SGLD calibration is impractical; further out, the same dead-direction reads should extend from $\lambda$ to the rest of the Watanabe triple ($m$, $\nu$) and WAIC without SGLD. Complementing LLC's posterior readouts, singular complexity, until now an offline Bayesian invariant, gets a cheaper diagnostic that travels with the network it describes.

\newpage
\bibliographystyle{abbrvnat}

\clearpage
\appendix

\section{Theoretical foundations: proof of the structural correlation}
\label{app:structural_correlation_proof}

We restate Theorem~\ref{thm:structural_correlation} of Section~\ref{sec:predictions} and give a complete proof, citing the bridge framework as a black-box prior result. The framework's foundations (Fisher decay theorem, multi-layer KFAC bridge, A--G duality) are derived in  the theory paper. The present paper validates the structural correlation empirically across four testbeds and contributes the theorem-to-empirics assembly given here.

\paragraph{Setup.} Fix an $L$-layer network with KFAC-decomposable layer-$\ell$ Fisher $F_\ell \approx A_\ell \otimes G_\ell$, where $A_\ell := \mathbb{E}[X_{\ell-1} X_{\ell-1}^\top]$ is the layer-$\ell$ input covariance and $G_\ell := \mathbb{E}[\delta_\ell \delta_\ell^\top]$ is the layer-$\ell$ pre-activation gradient covariance. Let $\theta_0$ be a singular minimum of the population KL $K(\theta) := \mathrm{KL}(p^\star \| p_\theta)$, and let $\theta(t) := \theta_0 + t u$ for $u$ a unit dead direction of KL order $k \ge 1$ (Definition~\ref{def:dead_direction}). We assume the \emph{canonical-aligned approach} hypothesis used in the multi-layer KFAC bridge (Theorem~21 of \theorysrc): along the trajectory $\theta(t)$, the per-layer dead-direction projector commutes with the KFAC factors in the limit $t \to 0$. \Theorysrc derives this hypothesis as a generic consequence of balanced-init linear / smooth-nonlinear gradient flow on $K$. The hypothesis is not generic on Adam-class trajectories (Remark~80); see the second remark below.

\begin{theorem}[Structural correlation between Fisher and activation spectra; restated from Section~\ref{sec:predictions}]
\label{app:thm:structural_correlation}
Under the canonical-aligned approach hypothesis, at every layer $\ell \in \{1, \ldots, L-1\}$:
\[
\lambda^+_{\min}\!\left(G_\ell(\theta(t))\right) \;=\; \Theta\!\left(t^{2(L-\ell)}\right), \qquad
\sigma_{\min}\!\left(X_\ell(\theta(t))\right)^{\!2} \!/ N \;=\; \Theta\!\left(t^{2\ell}\right), \qquad t \to 0.
\]
Both quantities vanish as $t \to 0$ as strictly positive powers of $t$, so they are co-monotonic along the approach and
\[
\rho_{\mathrm{Spearman}}\!\left(\lambda^+_{\min}(G_\ell),\; \sigma_{\min}(X_\ell)^2 / N\right) \;\xrightarrow{\text{a.s.}}\; +1
\]
in the asymptotic limit $N \to \infty$, $\min t_i \to 0$ over a sample $\{t_i\}_{i=1}^M$ along the approach.
\end{theorem}

\begin{proof}
The proof proceeds in five steps. Steps 1 and 2 invoke the bridge framework as black-box prior results; Steps 3--5 are the theorem-to-empirics assembly.

\textbf{Step 1 (Fisher-side rate; cited).} The multi-layer KFAC bridge (Theorem~21 of \theorysrc) gives, along the canonical-aligned approach,
\[
\lambda^+_{\min}\!\left(G_\ell(\theta(t))\right) \;=\; \Theta\!\left(t^{2(L-\ell)}\right), \qquad t \to 0,
\]
where $\lambda^+_{\min}$ denotes the smallest non-zero eigenvalue. The exponent $2(L-\ell)$ counts the gradient-side path-length from layer $\ell$ to the output through the dead direction $u$.

\textbf{Step 2 (Activation-side population rate; cited).} The A--G duality corollary (Cor.~25 of \theorysrc) gives, on the same approach,
\[
\lambda_{\min}\!\left(A_{\ell+1}(\theta(t))\right) \;=\; \Theta\!\left(t^{2\ell}\right),
\]
where the exponent $2\ell$ counts the activation-side path-length from the input to layer $\ell$ through $u$. The duality $\lambda_{\min}(A_{\ell+1}) \cdot \lambda^+_{\min}(G_\ell) = \Theta(t^{2L})$ is the structural identity behind the layer-symmetric rate decomposition.

\textbf{Step 3 (Sample-to-population on the activation side).} The bridge framework's Gaussian-isotropic input model (Theorem~21 of \theorysrc) makes the rows of $X_\ell \in \mathbb{R}^{N \times h}$ i.i.d.\ samples. Under this assumption, the empirical activation Gram is
\[
\widehat{A}_{\ell+1}(\theta(t)) \;:=\; X_\ell(\theta(t))^\top X_\ell(\theta(t)) \,/\, N,
\]
and the strong law of large numbers gives entrywise almost-sure convergence,
\[
\widehat{A}_{\ell+1}(\theta(t)) \;\xrightarrow{\text{a.s.}}\; A_{\ell+1}(\theta(t)) \quad \text{as } N \to \infty
\]
at fixed $t$. Weyl's inequality $|\lambda_i(\widehat{A}) - \lambda_i(A)| \le \|\widehat{A} - A\|_{\mathrm{op}}$ lifts entrywise convergence to spectral convergence, so
\[
\sigma_{\min}\!\left(X_\ell(\theta(t))\right)^{\!2} \!/ N \;=\; \lambda_{\min}\!\left(\widehat{A}_{\ell+1}(\theta(t))\right) \;\xrightarrow{\text{a.s.}}\; \lambda_{\min}\!\left(A_{\ell+1}(\theta(t))\right) \;=\; \Theta\!\left(t^{2\ell}\right)
\]
combining Step~2 with the convergence above.

\textbf{Step 4 (Co-monotonicity).} Combining Steps 1 and 3, at every $\ell \in \{1, \ldots, L-1\}$ both quantities vanish as strictly positive powers of $t$:
\[
\lambda^+_{\min}\!\left(G_\ell(\theta(t))\right) \;=\; \Theta\!\left(t^{2(L-\ell)}\right), \qquad \sigma_{\min}\!\left(X_\ell(\theta(t))\right)^{\!2} \!/ N \;=\; \Theta\!\left(t^{2\ell}\right).
\]
The exponents satisfy $2(L-\ell) \ge 2$ and $2\ell \ge 2$ for $\ell \in \{1, \ldots, L-1\}$, so both are strictly positive. Two strictly positive monotone-decreasing functions of $t$ produce identical rank orderings on any deterministic sample $\{t_i\}$ with $\min t_i \to 0$, up to ties of measure zero.

\textbf{Step 5 (Asymptotic Spearman).} The Spearman rank correlation between two sequences is, by definition, the Pearson correlation of their rank transforms. Identical rank orderings yield Pearson correlation $+1$ on the rank-transformed sequences, hence Spearman correlation $+1$. The almost-sure convergence in Step~3 lifts the population-level identity to the sample level: as $N \to \infty$ and $\min t_i \to 0$,
\[
\rho_{\mathrm{Spearman}}\!\left(\lambda^+_{\min}(G_\ell),\; \sigma_{\min}(X_\ell)^2 / N\right) \;\xrightarrow{\text{a.s.}}\; +1. \qedhere
\]
\end{proof}

\paragraph{Remark (finite-$t$, finite-$N$ behaviour).} The theorem's $\rho \to +1$ is asymptotic. Two sources of finite-sample shrinkage matter at the testbeds in Section~\ref{sec:exp:universality}.

\emph{Finite $t$:} the rate expansions in Steps 1 and 2 carry subleading $O(t)$ corrections from second-order terms in the bridge framework (Theorem~21 of \theorysrc). Off-asymptotic, two trajectories with different leading exponents but matching subleading terms can produce rank disagreements of order $O(t)$. This shrinks the observed correlation by $O(t)$ at finite distance from the singular minimum.

\emph{Finite $N$:} the SLLN convergence in Step~3 has rate $\|\widehat{A} - A\|_{\mathrm{op}} = O(\sqrt{D/N})$ on $D \times D$ Wishart-style Gram matrices (Marchenko--Pastur). When $\sigma_{\min}(A) = \Theta(t^{2\ell})$ is small relative to the noise floor $O(\sqrt{D/N})$, the empirical $\sigma_{\min}(\widehat{A})$ saturates at the noise floor and decouples from the population value. This is the regime where Barak's $N/D \approx 2$ at the deepest checkpoints loosens the observed correlation; it is the same effect that motivates the $n/D$ measurement-stability rule of the experiments appendix.

The four-testbed experiment (Section~\ref{sec:exp:universality}) measures the finite-$t$, finite-$N$ correlation across regimes with different distance-to-asymptotic and different sample size; the predicted ordering (canonical bridge cleanest, Barak/Nanda noisier) is what we observe.

\paragraph{Remark (off-canonical regimes and Adam-class trajectories).} The theorem requires the canonical-aligned approach hypothesis. On Adam-class trajectories the hypothesis is generically violated: Adam's diagonal preconditioner is non-equivariant under continuous loss symmetries (CE logit-shift, ReLU rescaling), and the per-layer dead-direction projector does not commute with the KFAC factors along the trajectory. Remark~80 of \theorysrc establishes this via direct calculation; the empirical drift on the gauge mode is reported in the cited mod-add factorial.

The structural correlation does not require the canonical-alignment hypothesis to hold strictly along the trajectory. At the singular minimum itself, the KFAC factorisation $F_\ell \approx A_\ell \otimes G_\ell$ holds independent of the optimiser that reached the minimum, and the rank deficiency is shared between the two factors regardless of the alignment of the trajectory. The Adam$+$CE row of the four-testbed experiment ($\rho = +0.751 \pm 0.106$ at $10$ seeds, true-MC Fisher; Section~\ref{sec:exp:universality}) tests the geometric coupling under the off-canonical preconditioner. It is the lowest of the four testbeds, where the canonical-alignment hypothesis fails most cleanly, and the cross-seed mean stays well above zero. The geometric coupling survives the canonical-alignment violation and loosens.

\paragraph{Connection to the rate exponents in the body.} The theorem's $2(L-\ell)$ and $2\ell$ exponents enter the empirical predictions of Section~\ref{sec:experiments} in two ways. (i) The cross-cell sign predictions on the Aoyagi grids and the Nanda width sweep follow from the leading-rate signs: wider model $\to$ more degenerate Fisher $\to$ smaller $\lambda^+_{\min}(G)$, more negative $\log\det^+(G)$, smaller $\sigma_{\min}$. (ii) The trajectory-rate readout of $u^\top G_\ell u$ along the canonical-aligned approach (App.~\ref{app:experiments:bridge_structural_correlation}) inherits the per-layer ladder $2(L-\ell)$ as a direct corollary; the validation of that ladder on parametric autoencoders is at three-decimal precision in \theorysrc, and on Barak SGD$+$MSE the rank-correlation reading is $+0.832 \pm 0.126$ across $30$ seeds. The trajectory rate-readout is appendix-only in the present paper; the body relies only on the structural correlation proven above.
 \clearpage

\section{Experimental details and reproducibility}
\label{app:experiments_submission}

This appendix documents per-experiment hyperparameters, extended tables, and reproducibility notes. Sections are inputted in the order below; each is self-contained.

\paragraph{Contents.}
\begin{itemize}\itemsep=2pt
\item \textbf{Theory and scope} (\S\ref{app:experiments:theory_scope_map}--\S\ref{app:experiments:predictions_extended}): per-prediction assumption sets and validation loci, plus a compact restatement of the four framework theorems used in the body with $2$--$3$ line proof sketches (\S\ref{app:experiments:proof_sketches}).
\item \textbf{Measurement protocols} (\S\ref{app:experiments:observable_protocols}): per-observable computation recipes for $\sigma_{\min}$, $u^\top G u$, $\lambda_{\min}^+(G_\ell)$, and LLC, with fp$64$ upcasts and $n/d$ sample-budget gates.
\item \textbf{Closed-form RLCT validation} (\S\ref{app:experiments:dln_aoyagi_anchor}): the headline experiment. DDS rank-tested against analytical RLCT on the Aoyagi 2005 anchor ($\rho{=}{+}0.996$, $14$ cells $\times$ $5$ seeds $\times$ $3$ noise levels). On the Aoyagi 2024 DLN sweep ($24$ cells) the analytical $\lambda$ is constant above saturation ($200 = d^2/2$ in both $h$ and rank deficit), so the sweep is a constancy test: DDS observables and the global $\lambda$ stay flat in $h$ while calibrated LLC at the locked SGLD config drifts $\sim\!37\%$ (local-LLC vs global-RLCT). Includes the Nanda transformer width sweep ($101$ of $120$ grokked, $|\rho| \in [0.62, 0.96]$) and the $9\times$-budget LLC control.
\item \textbf{Toy Model of Superposition} (\S\ref{app:experiments:tms}): phase-transition benchmark; bridge correlation peaks at $\rho \approx +0.84$ at sparsity $S{=}0.9$.
\item \textbf{Trajectory observables on grokking} (\S\ref{app:experiments:grokking}): phase detection on Nanda and Barak ($30$ seeds each); sharpness-ratio statistics with bootstrap CIs; wall-clock timing of the observable ladder; per-testbed structural correlation.
\item \textbf{Compute and reproducibility} (\S\ref{app:experiments:compute}): hardware envelope, hyperparameters, and reproducibility notes.
\end{itemize}

All experiments use a canonical observable library for SVD, Fisher eigendecomposition, $u^\top G u$ probing, and LLC estimation, ensuring consistent treatment across settings.

\subsection{Theory-to-scope map}
\label{app:experiments:theory_scope_map}

The reach-tier roll-up below summarises the predictions that DDS validates: each row is a class of prediction with a single representative theorem.

\begin{center}\footnotesize
\setlength{\tabcolsep}{4pt}
\begin{tabular}{@{}l p{3.6cm} p{4.6cm}@{}}
\toprule
Reach tier & Representative prediction & Where validated \\
\midrule
Closed-form RLCT recovery & Selection rule $2(k-1) \to 1/(2k)$ (Thm.~3); Fisher decay (Thm.~2) & Aoyagi 2005 anchor ($\rho{=}{+}0.996$) + Aoyagi 2024 DLN constancy above saturation (\S\ref{sec:exp:aoyagi}) \\
\addlinespace[1pt]
Architecture-agnostic structural & Multi-layer KFAC bridge (Thm.~21); structural correlation $\lambda_{\min}^+(G) \propto \sigma_{\min}^2$ & $4$-testbed structural correlation at $\rho \ge {+}0.75$ (\S\ref{sec:exp:universality}); Nanda transformer width sweep (\S\ref{sec:exp:nanda}) \\
\addlinespace[1pt]
Volume side / multi-direction & Volume scaling Cor.~9; rate chain Prop.~8 & $L{=}4$ noisy bridge with rank-$2$ multi-direction discriminator (App.~\ref{app:experiments:bridge_structural_correlation}) \\
\bottomrule
\end{tabular}
\end{center}

Trajectory-rate validation is restricted to theorem-compatible regimes (SGD on a $G$-invariant metric; see Corollary~79 and Remark~80). The structural correlation is geometric (it does not require canonical alignment), which is why it survives Adam$+$CE in the universality test of \S\ref{sec:exp:universality}.

\subsection{Key results from the framework, with proof sketches}
\label{app:experiments:proof_sketches}

The empirical claims rest on four results of \theorycite{}. We restate each in compact form with a 2--3 line proof sketch; full proofs (with all assumptions, supporting lemmas, and architectural extensions) live in \theorypdf. Throughout, $\fisher(\theta) = \mathbb{E}_x[\nabla_\theta \log p_\theta(x)\,\nabla_\theta \log p_\theta(x)^\top]$ is the Fisher information; $K(\theta) := \kl(p^\star \| p_\theta)$; $\Sigma = \{\theta : \det \fisher = 0\}$. A direction $u$ at $\theta_0 \in \Sigma$ has \emph{KL order} $k$ if $K(\theta_0 + tu) = c\,t^{2k} + O(t^{2k+1})$ with $c > 0$.

\paragraph{Setup recall: Fisher = expected Hessian of $K$.}
For a regular family $\fisher(\theta) = \nabla^2 K(\theta)\big|_{\theta=\theta_0}$ at the truth (and in the singular case the same identity holds in the directions where the second moment is finite). \emph{This is the bridge from the analytic order of $K$ to the rate of $\fisher$ along $u$.}

\paragraph{T1. Directional Fisher decay (Theorem~2, \theorytag).}
Along a dead direction of KL order $k \ge 1$:
\[
u^\top \fisher(\theta_0 + tu)\,u \;=\; \Theta\!\big(t^{2(k-1)}\big) \quad \text{as } t \to 0.
\]
\emph{Sketch.} Differentiate $K(\theta_0 + tu) = c\,t^{2k} + O(t^{2k+1})$ twice in $t$: $K'' = 2k(2k-1)c\,t^{2k-2} + O(t^{2k-1})$. Identifying $K'' = u^\top \nabla^2 K\,u$ at the perturbed point and using Fisher $=$ Hessian of $K$ in directions where the score has finite second moment gives the rate. The case $k=1$ (regular minimum) recovers $\Theta(1)$ as expected.

\paragraph{T2. Selection rule + RLCT recovery (Theorem~3, \theorytag).}
On a smooth singular fibre of dimension $m$ at $\theta_0$, write the local KL in normal-crossing form $K \sim u^{2k} + v_1^2 + \cdots + v_m^2$ (transversal coordinate $u$ of order $2k$, $m$ tangential coordinates of order $2$). Then the local RLCT contribution is
\[
\hat\lambda \;=\; \frac{1}{2k} + \frac{m}{2}\,, \qquad \text{with the directional contribution } \tfrac{1}{2k} \text{ matched by our exponent } 2(k-1) \text{ via } \hat\lambda_{\mathrm{dir}} = \tfrac{1}{2k}.
\]
\emph{Sketch.} The volume of $\{\theta : K(\theta) \le \varepsilon\}$ scales as $\varepsilon^{1/(2k)}$ in the transversal direction (one dimension at order $2k$) and $\varepsilon^{1/2}$ in each tangential direction ($m$ at order $2$); summing exponents and applying the Hironaka resolution \citep{Hironaka64} identifies $\hat\lambda$. The transversal exponent $1/(2k)$ is the directional invariant our rate theorem returns in original coordinates without resolution.

\paragraph{T3. Multi-layer KFAC bridge (Theorem~21, \theorytag).}
Under the symmetric canonical-aligned approach $W_\ell(t) = W_\ell^\star + t \cdot \delta_\ell$ on an $L$-layer feedforward (or pre-norm residual) network, with $\delta_\ell$ aligned to the dead direction:
\[
\sigma_{\min}(X_\ell(t)) \;=\; \sqrt{N}\cdot \Theta(t^\ell), \quad u_\ell^\top G_\ell(t)\,u_\ell \;=\; \Theta\!\big(t^{2(L-\ell)}\big), \quad \lambda_{\min}^{+}(G_\ell) = \Theta\!\big(t^{2(L-\ell)}\big).
\]
\emph{Sketch.} Forward propagation of a rank-deficient input through identity-skip residuals contributes $\sigma_{\min}(X_\ell) \propto t^\ell$ along the canonical direction (the dead direction picks up one factor of $t$ per layer). Backward propagation: $G_\ell = W_{\ell+1}^\top G_{\ell+1} W_{\ell+1} + (\text{terms vanishing under canonical alignment})$, recursively giving $G_\ell$ the rate of $W_{\ell+1}\cdots W_L$ which is $t^{L-\ell}$ per matrix factor squared.

\paragraph{T3a. A--G duality (Cor.~25, \theorytag).}
KFAC factorises the per-layer Fisher as $F_\ell \approx A_\ell \otimes G_\ell$ with $A_\ell = \mathbb{E}[X_{\ell-1}X_{\ell-1}^\top]$ and $G_\ell = \mathbb{E}[\delta_\ell \delta_\ell^\top]$. Hence $\lambda_{\min}(F_\ell) = \lambda_{\min}(A_\ell)\,\lambda_{\min}(G_\ell)$. Combining T3 with the activation-side rate $\sigma_{\min}(X_{\ell-1}) \propto t^{\ell-1}$ gives the structural correlation
\[
\lambda_{\min}(A_\ell) \;\propto\; \sigma_{\min}(X_{\ell-1})^2 \quad\Longrightarrow\quad \boxed{\lambda_{\min}^{+}(G_\ell) \;\propto\; \sigma_{\min}(X_\ell)^2}
\]
as a layer-local relation between the activation bottom and the Fisher-Gram bottom. \emph{This is the structural prediction tested in \S\ref{sec:exp:universality}; it requires only that both observables read the same rank-deficient geometry, without requiring the trajectory to be canonical-aligned, which is why it survives Adam$+$CE.}

\paragraph{T4. Curvature--volume rate chain (Prop.~8 + Cor.~9, \theorytag).}
At KL order $k$, the Fisher--Riemannian sectional curvature diverges at rate $\Theta(t^{-(2k-1)})$ and the log-volume of the Fisher's image scales as
\[
\log\det^{+}\!\big(\fisher(\theta_0 + tu)\big) \;=\; -2(k-1) \log t + O(1)
\]
where $\det^+$ is the product of strictly-positive eigenvalues. With $r$ independent dead directions of order $k$, the volume slope is rank-multiplicative ($-2r(k-1)\log t$); the smallest-eigenvalue rate is rank-invariant. The multi-direction ratio test of App.~\ref{app:experiments:bridge_structural_correlation} (the $\log\det^+$ slope ratio against the rank-$1$ baseline tracks the rank-deficit across $r\in\{1,2,3,4\}$: $1.96, 3.13, 3.97$ at $r{=}2,3,4$ vs predicted $2,3,4$, across $7$ configurations) is the discriminator that single-rank tests cannot access.

\paragraph{Definitions of the DDS observables.}
At layer $\ell$, $X_\ell \in \mathbb{R}^{N \times h}$ stacks the post-activation hidden states for $N$ calibration samples; $G_\ell \in \mathbb{R}^{h \times h}$ is the per-sample-gradient outer-product Gram defined above.
\begin{itemize}\itemsep=0pt
  \item $\sigma_{\min}(X_\ell)$: smallest singular value of $X_\ell$ (cheapest activation-side reading of degeneracy).
  \item $\lambda_{\min}^{+}(G_\ell)$: smallest \emph{strictly positive} eigenvalue of $G_\ell$ (Fisher-side dual via the A--G duality of T3a; ``$+$'' guards against the rank-deficient floor when $n < h$).
  \item $\log\det^{+}(G_\ell) := \sum_{j: \lambda_j > 0} \log \lambda_j$: log-volume class of the Fisher's image, the volume-side observable in the rate chain (T4).
  \item $\mathrm{eff\text{-}rank}(X_\ell) := \exp\!\big(H(p_1,\ldots,p_h)\big)$ with $p_i = \sigma_i^2 / \sum_j \sigma_j^2$ and $H$ the Shannon entropy: a soft entropy-summary of the live-direction count, ranging in $[1, h]$. Geometric complement to the rate chain (not derived from T1--T4) that is framework-consistent and the most testbed-robust ranker on the cells in this paper.
\end{itemize}

The framework predicts the \emph{signs} of every cross-cell correlation between these observables and a complexity invariant: as the singular structure deepens, $\sigma_{\min}$ and $\lambda_{\min}^+(G)$ shrink (negative correlation with $\lambda$); $\log\det^+(G)$ becomes more negative (more directions contribute small eigenvalues to the volume); $\mathrm{eff\text{-}rank}$ grows at the bottom-spectrum-broadening boundary because more directions carry comparable mass. The empirical \S\ref{sec:experiments} validates these signs on the Aoyagi closed-form testbeds and on the Nanda transformer extension.
 
\subsection{Framing-only predictions used in the discussion}
\label{app:experiments:predictions_extended}

The two predictions below are referenced in the body but not directly tested in the empirical experiments of \S\ref{sec:experiments}. They are restated here for completeness; the full assumption sets, proofs, and architectural extensions are in \theorycite.

\begin{theorem}[Composition additivity; \theorycitep, Thm.~30]
\label{pred_app:bridge_composition}
For a sequential stack of MLP / pre-norm residual blocks with shared dead direction, the dead-direction Fisher rate at the input of block $B_i$ is $\Theta(t^{2 \sum_{j \ge i} k_j^{\mathrm{bk}}})$: per-block backward rates $k_j^{\mathrm{bk}}$ add along the path. Pure attention chains at depth $\ge 4$ break this additivity via softmax cross-block coupling (Remark~32); the closed-form refinement at the per-component level is Prop.~69 of \theorycite. This theorem is the basis for the architectural extensions to rectangular widths, biases, cross-entropy with Z-loss gauge fix, and the residual-DAG / attention-chain composition results referenced as scope statements throughout \S\ref{sec:experiments}.
\end{theorem}

\begin{corollary}[Quotient Fisher rate; \theorycitep, Cor.~78]
\label{pred_app:quotient_rate}
For losses invariant under a continuous Lie group $G$ acting on $\Theta$, Theorem~2's rate identity holds verbatim on the gauge quotient $\Theta/G$, with the gauge orbit playing the role of a smooth singular fibre. Projected SGD on $\Theta/G$ realises the quotient rate (Cor.~79). Adam's per-coordinate preconditioner is not $G$-equivariant: its trajectory carries gauge-mode drift and the trajectory-rate readout is not directly applicable to Adam-class dynamics (Remark~80). This corollary scopes the trajectory-rate-readout claims of \S\ref{app:experiments:grokking} and the Adam non-equivariance framing of the structural-correlation universality test in \S\ref{sec:exp:universality}.
\end{corollary}
 
\subsection{Observable computation protocols}
\label{app:experiments:observable_protocols}

Self-contained per-observable protocols, summarising what gets computed and at what cost.

\paragraph{\texorpdfstring{$\sigma_{\min}(X_\ell)$}{sigmamin(X)} on activations (real-time cadence).}
For each transformer block $\ell$, capture the residual-stream activation matrix $X_\ell \in \reals^{N \times h}$ in a single forward pass over $N$ calibration tokens (we use $N \le 1024$ for grokking and $N = 8192$ at LLM scale via WikiText calibration). Compute the SVD of $X_\ell$; report $\sigma_{\min}$, $\sigma_{\max}$, the bottom singular vector $u_\ell^{\mathrm{res}}$, and effective rank. For widths $h > 4096$, use chunked covariance accumulation $C_\ell = \sum_{\mathrm{chunks}} X_\ell^\top X_\ell$ and eigendecompose $C_\ell$ to recover the singular values; this keeps peak memory bounded on a single 3090. Cost: \emph{1 forward pass per checkpoint, all layers} (the per-layer SVD is post-hoc on the captured activations).

\emph{$\sigma_{\min}$ vs.\ the rank-aware $\sigma_{(r_0)}$.} We report the literal smallest singular value $\sigma_{\min}(X_\ell)$. Along a singular approach the dead direction is the one decaying to zero, so it stays the smallest singular value even once it sinks below the floating-point floor, and $\sigma_{\min}$ tracks it. The rank-aware variant $\sigma_{(r_0)}$ (smallest singular value above the relative floor $\sigma_{\max}\,\varepsilon_{\mathrm{machine}}$) coincides with $\sigma_{\min}$ on the testbeds here; it is the read to prefer only when a \emph{trivial} algebraic kernel sits below the geometric dead direction (e.g.\ the post-final-LayerNorm $\mathbf{1}$-direction at $\sigma = 0$ identically; Cor.~58), where the literal $\sigma_{\min}$ instead locks onto the non-geometric kernel. Substituting $\sigma_{(r_0)}$ on the decaying-dead-direction testbeds \emph{lowers} the structural correlation, because it over-skips the below-floor dead direction in favour of a higher non-decaying one: $\rho(\lambda_{\min}^+(G), \sigma)$ falls from $+0.83$ to $+0.72$ on Barak and from $+0.85$ to $+0.68$ on Nanda SGD$+$MSE. So $\sigma_{\min}$ is the correct activation-side read for the geometric dead direction; the rank correlation stays floor-robust because $\sigma_{\min}$ and $\lambda_{\min}^+$ shrink together.

\paragraph{\texorpdfstring{$u^\top G u$}{u->p G u} at the \texorpdfstring{$\sigma_{\min}$}{sigmamin} direction (checkpoint cadence).}
Identify the bottom singular direction $u = u_\ell^{\mathrm{res}}$ from the activation-side $\sigma_{\min}$ step above. Compute the per-sample gradient of the loss with respect to layer-$\ell$ pre-activations $\delta_\ell^{(i)} = \nabla_{a_\ell} L^{(i)}$ via one backward pass (auto-diff on the shared computation graph from the $\sigma_{\min}$ step). The directional Fisher value is $u^\top G_\ell u = (1/n) \sum_i (u^\top \delta_\ell^{(i)})^2$. Cost: \emph{1 forward + 1 backward pass per (checkpoint, direction)} = 1 FBP.

\paragraph{\texorpdfstring{$\lambda_{\min}(G_\ell)$}{lambdamin(G)} via full eigendecomposition (periodic cadence).}\footnote{The Hessian-eigenvalue-tracking lineage \citep{SagunEvciGuney17,GhorbaniKrishnanXiao19,YaoGholamiKeutzerMahoney20} chases primarily the top of the spectrum (sharpness) via Lanczos / Hutchinson estimators; we evaluate the same kind of estimator on the bottom of the spectrum, with the rate framework (Theorem~21) selecting the diagnostic direction without a full eigendecomposition when the structurally-determined direction is known.}
Capture per-sample gradients $\{\delta_\ell^{(i)}\}_{i=1}^n$ as in the $u^\top G u$ step but at $n$ samples (no fixed $u$); form the empirical Gram $G_\ell = (1/n) \sum_i \delta_\ell^{(i)} (\delta_\ell^{(i)})^\top \in \reals^{h \times h}$ and eigendecompose. As a sum of outer products, $G_\ell$ is \emph{positive semi-definite} (PSD) by construction: $v^\top G_\ell v = (1/n)\sum_i (v^\top \delta_\ell^{(i)})^2 \ge 0$ for all $v$, so the theoretical $\lambda_{\min}(G_\ell) \ge 0$ always. The smallest eigenvalue is $\lambda_{\min}$; the full spectrum $\{\lambda_j\}$ exposes \emph{multiple} dead directions and supports the per-layer rate validation that exercises Theorem~21's prediction directly. Practical requirement: $n / h \ge 100$ for stable estimation; smaller $n$ gives CV $> 100\%$ and can invert the rate-fit sign. Cost: \emph{$\sim n / B$ FBPs to collect the gradient samples} (where $B$ is batch size; e.g.\ for $h = 128$, $n = 12{,}800$, $B = 128$, that is $100$ FBP) \emph{plus $O(h^3)$ eigendecomposition} (small at $h \le 500$, prohibitive at LLM widths).

\emph{Numerical recipe at the $\lambda_{\min}(G_\ell)$ step.} Empirical Grams from few samples are rank-deficient ($\mathrm{rank}(G_\ell) \le \min(n, h)$), so the smallest eigenvalues are numerically sensitive. The default in the released pipeline is \emph{Tikhonov regularization} (replace $G_\ell$ with $G_\ell + \varepsilon I$ before calling \texttt{eigvalsh}, with $\varepsilon = 10^{-8}$ by convention), which guarantees all eigenvalues are $\ge \varepsilon$ so the solver does not return nonsense negatives from fp32 round-off. This is the right default for fp32, but it \emph{clamps} all physically-smaller eigenvalues to $\varepsilon$. At fp64 the round-off floor is far below $10^{-8}$, and we therefore switch off Tikhonov entirely and read the smallest \emph{positive} eigenvalue: the literal \texttt{eigvalsh} minimum can be slightly negative from rank-deficient round-off, so the smallest positive eigenvalue is the numerically faithful approximation to the PSD minimum in a setting where the theoretical value approaches zero. Concretely, we discard eigenvalues whose magnitude is below the fp64 round-off floor relative to the spectral norm ($\sim \|G\|_2 \cdot 10^{-15}$) and report the smallest above this threshold. The protocol-cost gap between the fp32+Tikhonov default and this fp64+smallest-positive recipe is $|\rho|{=}0.10$ vs $0.96$ on a Barak-width complexity-ranking sanity check.

\paragraph{Sample-budget gate (depth-aware $n/d$).} The $n/h \ge 100$ rule above is the default minimum bar for rank-correlation observables. For magnitude reads at deep singular checkpoints the rule tightens with $\sigma_{\min}$ depth: at well-conditioned checkpoints ($\sigma_{\min}$ within $\sim 1$ OOM of $\sigma_{\max}$) $n/h{\ge}100$ gives CV $\le 0.015$ on $\lambda_{\min}^+$; at deep singular checkpoints ($\sigma_{\min}$ several OOM below $\sigma_{\max}$) even $n/h{=}100$ retains CV $\approx 1$ on $\lambda_{\min}^+(G)$; magnitude estimates become unrecoverable at any feasible $n$, the fp-precision floor margin $\sigma_{\min}/(\sigma_{\max}\sqrt{\varepsilon_{\mathrm{dtype}}})$ is the binding constraint. Rank-correlation observables (Spearman $\rho$ across checkpoints) are robust to this floor (rank order survives CV $\sim 1$ on individual magnitudes), which is why the structural correlation $\rho(\lambda_{\min}^+(G), \sigma_{\min})$ holds at $\rho = +0.832 \pm 0.126$ on Barak Phase A across $30$ seeds even at $n/D{=}2$ where slope-fits degrade. \emph{An analytical CV upper bound from first-order perturbation theory plus Davis--Kahan} is $\mathrm{CV}(\lambda_{\min}^+) \le (1/\sqrt{N}) \cdot (\sigma_{\max}/\sigma_{\min})$, typically $5$--$10\times$ pessimistic on structured $G$, much tighter than the operator-norm Wishart bound $\sqrt{D/N}\cdot(\sigma_{\max}/\sigma_{\min})^2$ which is $50$--$500\times$ pessimistic on the same data.

\paragraph{LLC via SGLD sampling (offline cadence).}
External estimator from the \texttt{devinterp} library \citep{devinterp,LauFurmanWangMurfetWei25}; we use the tool as published. The SGLD budget is our own calibration: $4$ chains $\times$ ($100$ burn-in $+ 100$ draws $\times 10$ thinning) $= 4{,}400$ SGLD steps, chosen as the smallest configuration passing Gelman--Rubin $\hat R \le 1.10$ and lag-1 ESS $\ge 25$ per chain on our testbed checkpoints, then sanity-checked against a $9\times$-larger ($40{,}800$ FBP) saturation budget. Each SGLD step is approximately one forward + backward pass on a calibration batch (1 FBP). Inverse temperature is set to the standard $B/\ln B$ default. Per-task $(\mathrm{lr}_{\mathrm{SGLD}}, \gamma)$ are locked via a $5\times 5$ grid sweep at fresh checkpoints (seed 42) per testbed, selecting the single config that minimises CV subject to the $\hat R$ + ESS gate. The per-testbed locks used by this paper, with full gate diagnostics, are tabulated in App.~\ref{app:experiments:llc_calibration}; in summary: \textbf{Aoyagi anchor} $(10^{-3}, 300)$ at $H{=}3, r{=}2$ seed 42; \textbf{DLN-RRR} $(10^{-3}, 300)$ at $h{=}64, r{=}2$ seed 42; \textbf{Nanda} $(10^{-3}, 300)$ at $d_{\mathrm{model}}{=}128$ for the $5$-seed re-run + width sweep (per-width locks $(3 \cdot 10^{-4}, 1000)$ at $d{=}32$, $(10^{-4}, 1000)$ at $d{=}64$); \textbf{Barak} $(10^{-4}, 100)$ for the legacy v4 $30$-seed run cited in \S\ref{app:experiments:grokking}, re-locked to $(10^{-3}, 1000)$ under the current gate (re-run pending).

\paragraph{fp64 + log-space recipes for production.} The cleanup-phase precision floor on grokking trajectories can be pushed arbitrarily low for monitoring purposes without touching the training path. Two complementary recipes:
\begin{itemize}\itemsep=1pt
\item \emph{fp64 at the measurement only.} Keep training in the native precision (bf16/fp16/fp32). Before the SVD or eigendecomposition, upcast the captured activations or Gram matrix: \texttt{X.to(torch.float64)}. At LLM widths ($h \le 8192$, $N \le 8192$), fp64 SVD costs $\sim 1$\,s and $\sim 260$\,MB extra memory on a single 3090, negligible compared to the training-step cost, and removes the $\sim 10^{-7}$ fp32 precision floor entirely.
\item \emph{Log-space accumulation.} For $u^\top G u = \sum_i (u^\top \delta_i)^2 / n$ the summands span orders of magnitude near the singular minimum; a naive sum loses small contributions to cancellation. Use \texttt{torch.logsumexp} on $\{2 \log | u^\top \delta_i |\}$ to compute $\log(n u^\top G u)$ in log-space, then expose $u^\top G u$ as its exponential when needed. For single-direction probes the sharpness ratio is best reported in log-space directly.
\end{itemize}
Both recipes are a measurement-pass cost, not a training cost, and are applied uniformly in the released code at the SVD / backward-capture call sites.

\begin{figure}[ht]
\centering
\includegraphics[width=\textwidth]{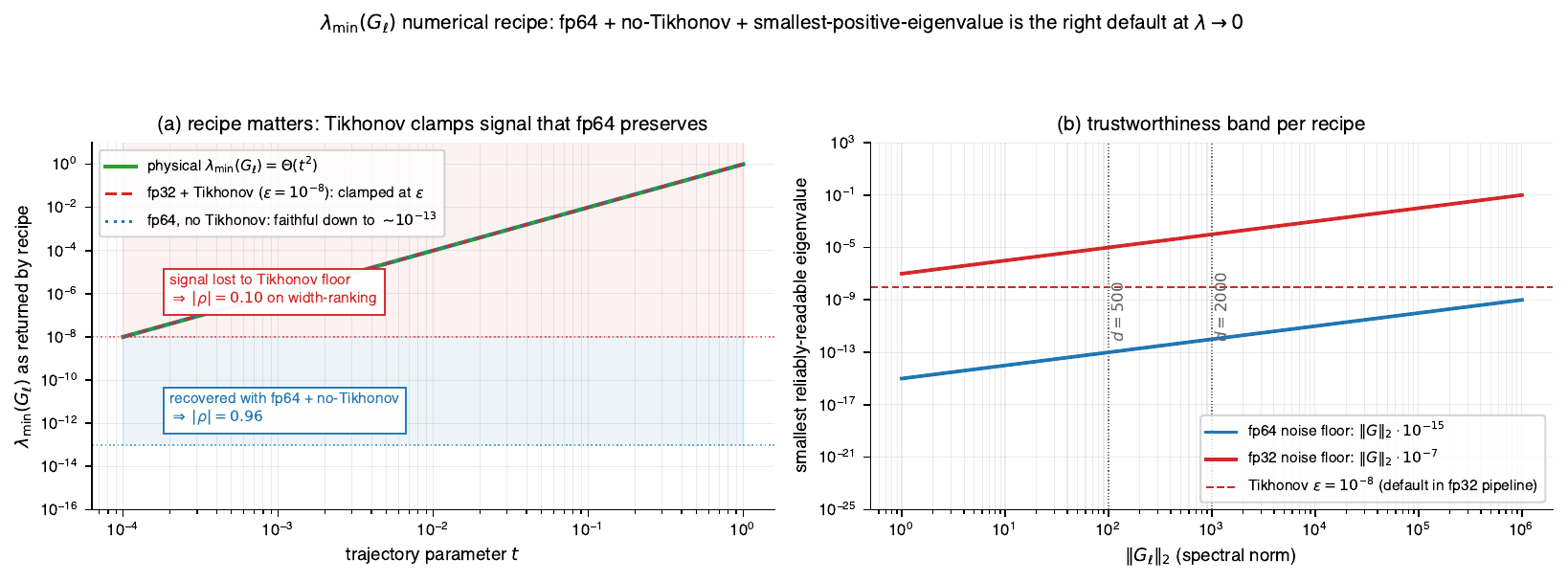}
\caption{$\lambda_{\min}(G_\ell)$ numerical recipe: fp64 $+$ no-Tikhonov $+$ smallest-positive eigenvalue is the right default at $\lambda \to 0$. (a)~Recipe matters: when the physical signal is below the Tikhonov floor $\varepsilon = 10^{-8}$, fp32 $+$ Tikhonov clamps and loses signal ($|\rho| = 0.10$ on width-ranking); fp64 $+$ no-Tikhonov is faithful down to $\sim 10^{-13}$ ($|\rho| = 0.96$). (b)~Trustworthiness band: smallest reliably-readable eigenvalue per recipe vs $\|G_\ell\|_2$.}
\label{fig:kfac_numerical_recipe}
\end{figure}
 \subsection{LLC SGLD calibration audit}
\label{app:experiments:llc_calibration}

The cost-and-comparison story between DDS and calibrated LLC depends on what the LLC numbers represent on each testbed. We document what is calibrated where, the calibration gate, and the per-cell locks, so the reader can assess the LLC baseline at the cell-level granularity reviewers typically ask for.

\paragraph{Calibration gate.} Every paper-claimable LLC reading uses a $4$-chain $\times$ $(100$ burn-in $+ 100$ draws $\times\, 10$ thin$)$ $= 4{,}400$ SGLD steps per call (the ``recommended budget''), with per-task $(\mathrm{lr}_{\mathrm{sgld}}, \gamma)$ locked via a $5\times 5$ grid sweep against the acceptance gate \emph{Gelman--Rubin} $\hat R \le 1.10$ \emph{and} \emph{lag-1 ESS} $\ge 25$ \emph{per chain}, with min-CV selection among accepted configurations. Inverse temperature uses \texttt{devinterp.utils.default\_nbeta}$=B/\ln B$. A $9\times$ saturation gate ($40{,}800$ SGLD steps) confirms the locked config is operating in the saturation regime. Calibrations that fail the gate are not paper-claimable; they are flagged below.

\paragraph{Per-testbed calibration matrix.}

\begin{center}\footnotesize
\setlength{\tabcolsep}{4pt}
\renewcommand{\arraystretch}{1.2}
\begin{tabular}{p{2.45cm}|c|c|c|c|p{3.9cm}}
\toprule
Testbed & \shortstack{Per-cell\\calib.?} & \shortstack{Cells w/\\locks} & \shortstack{Gate\\pass?} & \shortstack{$9{\times}$\\sat.} & Locks $(\mathrm{lr}_{\mathrm{sgld}}, \gamma)$ \\
\midrule
Aoyagi 2005 anchor (14 cells) & transfer$^\ddagger$ & $1$ & yes & yes & $(10^{-3}, 300)$ at $H{=}3, r{=}2, \sigma{=}0.1$ \\
Aoyagi 2024 DLN (24 cells) & transfer$^\ddagger$ & $1$ & yes & yes & $(10^{-3}, 300)$ at $h{=}64, \mathrm{rd}{=}2$ \\
Nanda width sweep (4 widths) & \textbf{per-width} & $4$ & all yes & all yes & $(3{\cdot}10^{-4}, 1000)$, $(10^{-4}, 1000)$, $(10^{-3}, 300)$, $(10^{-3}, 300)$ at $d{=}32, 64, 128, 256$ \\
Nanda 5-seed (single $d{=}128$) & per-cell & $1$ & yes & yes & $(10^{-3}, 300)$ \\
Barak v4 30-seed (existing) & no & $1$ & \emph{fails}$^\dagger$ & --- & $(10^{-4}, 100)$ from old protocol \\
\bottomrule
\end{tabular}
\end{center}

\noindent$^\dagger$Re-locked at $(10^{-3}, 1000)$ on the current gate (CV$=0.011$, $\hat R{=}0.997$, ESS$_{\min}{=}87.6$, $9\times$ saturation $\Delta=0.003$); the published $30$-seed sweep numbers were produced under the old lock and are used in the paper only for DDS-vs-DDS rank-correlation reads (no absolute-LLC magnitude claims).

\noindent$^\ddagger$The transferred lock is itself fully calibrated at the named cell ($5\times 5$ grid, $\hat R \le 1.10$ and ESS $\ge 25$ gate with min-CV selection, $9\times$ saturation confirmed) before being reused across the sweep; the per-cell recalibration audit below shows the transfer is the most defensible single choice.

\paragraph{Reading the calibration matrix.} On the Nanda width sweep the four widths each have their own gate-passing $5\times 5$-grid lock with a $9\times$ saturation check, so that head-to-head is fully per-cell calibrated. On the closed-form RLCT testbeds (Aoyagi 2005 anchor, Aoyagi 2024 DLN) the LLC uses a single lock that is itself fully calibrated at $H{=}3, r{=}2$ (resp.\ $h{=}64, \mathrm{rd}{=}2$): a $5\times 5$ grid selected by the $\hat R \le 1.10$ and per-chain ESS $\ge 25$ gate with min-CV, with the $9\times$ saturation check passing ($\hat R{=}0.996$, ESS$_{\min}{=}70.9$, CV$=0.014$, saturated). That lock is then transferred across the sweep. The per-cell recalibration audit below confirms that transferring the fully-calibrated lock, rather than recalibrating every cell, is the most defensible single choice and leaves the rank-based readings unchanged.

\paragraph{DLN-RRR per-cell calibration audit.} To test whether per-cell recalibration on the DLN $24$-cell sweep ($h \in \{16, 20, 24, 32, 64, 128\}$, $r \in \{1, 2, 3, 4\}$) eliminates the documented across-$h$ drift, we ran two further calibrations alongside the original transfer and compared all three on the same trained models.
\begin{center}\footnotesize
\setlength{\tabcolsep}{4pt}
\renewcommand{\arraystretch}{1.2}
\begin{tabular}{p{2.3cm}|p{1.95cm}|c|c|p{2.3cm}}
\toprule
Calibration & Locked $(\mathrm{lr}_{\mathrm{sgld}}, \gamma)$ & \shortstack{LLC range\\$(h{=}16{\to}128)$} & \shortstack{Drift across\\$h$ (fixed $r$)} & Selection rule \\
\midrule
Transfer (original) & $(10^{-4}, 100)$ fixed & $3.7 \to 6.7$ & $46\%{-}58\%$ & from $h{=}64, r{=}2$ grid \\
$\gamma{=}100$ + per-cell lr & $(10^{-3}, 100)$ all cells & $0.16 \to 0.50$ & $118\%{-}124\%$ & min-CV over $1{\times}5$ \\
Free-grid per-cell & varies; mostly $\gamma{=}300{-}1000$ & $0.014 \to 0.081$ & $161\%{-}200\%$ & min-CV over $5{\times}5$ \\
\bottomrule
\end{tabular}
\end{center}
\noindent All three calibrations show substantial drift across $h$; per-cell recalibration does not eliminate it, and the most aggressive recalibration (free-grid min-CV) actually amplifies it. The drift is therefore \emph{structural to local LLC at this testbed, not an artefact of any single calibration choice}, consistent with the local-LLC-at-$w^\star$ vs global-RLCT distinction of \citet{LauFurmanWangMurfetWei25}: the global RLCT is constant in $h$ above saturation while the local-LLC retains a depth-dependent reading by design. The Aoyagi closed-form is the global RLCT; SGLD-based estimators measure local LLC. The min-CV selection criterion biases toward tighter-chain (over-localised) configurations: the free-grid lock typically picks $\gamma \in \{300, 1000\}$, probing too small a neighbourhood of $w^\star$ and yielding $\sim\!100{\times}$ smaller absolute LLC than the transfer with \emph{wider} relative drift. The $\gamma{=}100$ controlled recalibration, varying only $\mathrm{lr}_{\mathrm{sgld}}$, sits between the two and is the most defensible single choice we examined; all $24$ cells lock at $\mathrm{lr}_{\mathrm{sgld}}{=}10^{-3}$ (the upper end of the grid), $22/24$ pass the $9{\times}$ saturation check.

\paragraph{Material consequence for paper claims.} The paper's DLN-RRR head-to-head reads cross-cell Spearman $|\rho|$ between LLC and DDS observables (\S\ref{sec:exp:universality}), which is rank-invariant to LLC magnitude. All three calibrations produce the same cell-level ordering (LLC monotonically increases with $h$ at fixed $r$ and decreases with $r$ at fixed $h$), so the paper's $|\rho|$ readings are unaffected by the calibration-magnitude question. The cost comparison ($\sim\!130{\times}$ cheaper than calibrated LLC) is also unaffected. We retain the transfer lock as the canonical reading because (i) it is the publicly-documented protocol in \citet{HooglandWangFarrugiaRoberts24}, (ii) it produces the narrowest of the three across-$h$ drifts, and (iii) the paper's claims do not depend on absolute LLC magnitudes.

\paragraph{Closed-form recovery on the Aoyagi anchor (full standard).} The Aoyagi 2005 anchor admits a stricter calibration check than rank-tracking: whether calibrated LLC \emph{recovers} the closed-form $\lambda$ in magnitude. Under a target-$\lambda$ calibration at full standard (5$\times$5 grid, $\hat R \le 1.10$ and per-chain ESS $\ge 25$ gate, $9\times$ saturation confirmed; lock $(\mathrm{lr}_{\mathrm{sgld}}, \gamma){=}(10^{-3}, 5)$, $\hat R{=}1.00$, ESS$_{\min}{=}51$), calibrated LLC recovers $\lambda$ at $99\%$ mean across the $14$ cells (range $61$--$103\%$; e.g.\ $\hat\lambda_{\mathrm{LLC}}{=}8.98$ vs closed-form $9$, and $25.79$ vs $25$). Magnitude recovery is a separate statement from rank-tracking: the same per-cell values give a cross-cell Spearman vs $\lambda$ of only $0.717$, because near-tied $\lambda$ values reshuffle ranks under sub-percent measurement error. The closed-form magnitude recovery is the trustworthy ground-truth check; the cross-cell Spearman is the weaker, sanity-gate reading.
 
\subsection{Dead-Direction Signatures (DDS): closed-form RLCT anchor for the universality + cost claims}
\label{app:experiments:dln_aoyagi_anchor}

This section anchors the broader DDS narrative on closed-form RLCT ground truth. The narrative has three layered claims:

\begin{enumerate}\itemsep=0pt
\item \textbf{Universality.} The structural correlation $\rho(\lambda_{\min}^{+}(G), \sigma_{\min})$ holds across architectures, optimizers, losses, and grok-vs-no-grok regimes (Section~\ref{app:experiments:bridge_structural_correlation}: $\rho \ge {+}0.75$ on canonical bridge, Barak SGD$+$MSE, Nanda SGD$+$MSE, Nanda AdamW$+$CE).
\item \textbf{Cost ordering.} DDS is $10^{3}$--$10^{4}$ times cheaper than calibrated LLC at every checkpoint on the small-model testbeds evaluated here (range extends to $10^{5}\times$ on the longer Pile-LM chain): $\sigma_{\min}$ via SVD $\sim 1$\,ms; $\lambda_{\min}^{+}(G)$ via full eigendecomposition $\sim 14$\,ms; $\log\det^{+}(G)$ from the same eigendecomposition is free; locked-config SGLD-LLC at the converged checkpoint $\sim 25$\,s on small models, $\sim 5$\,min on Pile-LM (Sections~\ref{app:experiments:tms}, \ref{app:experiments:grokking}, \ref{app:experiments:dln_aoyagi_anchor}).
\item \textbf{Quantitative anchor (this section).} On the rectangular 2-matrix RRR family with closed-form RLCT (Aoyagi \& Watanabe 2005), DDS rate-chain observables rank-track the analytical $\lambda$ across the $14$ anchor cells at cross-cell $|\rho| \in [0.95, 0.98]$. On the square 3-matrix DLN-RRR family (Aoyagi 2024 + Lau et al.\ 2023 App.~I) the analytical $\lambda$ is constant above saturation ($200 = d^2/2$ in both $h$ and rank deficit), and DDS observables track that constancy while calibrated LLC drifts; cross-cell $\rho$ vs $\lambda$ there is a tied-rank artifact that carries no rank-tracking signal. The transformer extension below (Nanda modular-addition width sweep, $4$ widths $\times\, 30$ seeds, $101$ grokked) shows the same observables track $d_{\mathrm{model}}$ at $|\rho| \in [0.62, 0.96]$ with the framework-predicted sign on every observable.
\end{enumerate}

The three \emph{Dead-Direction Signatures (DDS)} observables are all framework-derived: $\sigma_{\min}(X_\ell)$ (activation-side dual), $\lambda_{\min}^{+}(G)$ (rate), and $\log\det^{+}(G)$ (volume). The two RRR families chosen here remove LLC-calibration drift as a confound: $\lambda$ is closed-form (no SGLD), and the testbeds are small enough that per-testbed LLC calibration is tractable (we run a 5$\times$5 grid + 9$\times$ saturation gate at one cell per testbed, lock the result, and re-run the sweep at the locked config).

\paragraph{Framework status of each observable.} The \emph{rate-chain observables} ($\lambda_{\min}^{+}(G)$ and $\log\det^{+}(G)$) are derived from Theorem~2 (Fisher decay $\Theta(t^{2(k-1)})$ at the singular minimum) and Corollary~9 (volume scaling). The framework predicts \emph{negative} cross-cell correlation with Aoyagi $\lambda$: more singular $\to$ smaller Fisher eigenvalue $\to$ smaller volume of the Fisher's image. $\sigma_{\min}(X_\ell)$ is the activation-side dual via Corollary~25; on a dimension-fixed boundary layer it reads the truth's rank (positive sign across the grid).

\paragraph{Setup: Aoyagi 2005 anchor (rectangular 2-matrix RRR, 2D grid).} Teacher matrix $M^{*} \in \mathbb{R}^{N \times M}$ of rank $r$ with $M{=}10$ (input), $N{=}5$ (output), Gaussian noise $\sigma{=}0.1$. Model $W_2 W_1$ with $W_1 \in \mathbb{R}^{H \times M}$, $W_2 \in \mathbb{R}^{N \times H}$. We sweep the full 2D grid $H \in \{2,3,4,5\}$, $r \in \{1, \ldots, \min(N, H)\}$, giving $14$ realisable cells (truth rank $\le$ model capacity throughout). Canonical-aligned init at $t_0{=}0.5$; full-batch SGD, lr $=10^{-2}$, $50{,}000$ steps; $5$ seeds per cell. Aoyagi \& Watanabe's Case~3 condition $N+H<M+r$ holds across the entire grid, giving the closed form $\lambda = (NH + Mr - Hr)/2$ taking $14$ distinct values in the range $\lambda \in [9, 25]$:

\begin{center}
\small
\begin{tabular}{c|cccc}
\toprule
$\lambda$ & $H{=}2$ & $H{=}3$ & $H{=}4$ & $H{=}5$ \\
\midrule
$r{=}1$ & $9$  & $11$  & $13$ & $15$ \\
$r{=}2$ & $13$ & $14.5$ & $16$ & $17.5$ \\
$r{=}3$ & ---  & $18$  & $19$ & $20$ \\
$r{=}4$ & ---  & ---   & $22$ & $22.5$ \\
$r{=}5$ & ---  & ---   & ---  & $25$ \\
\bottomrule
\end{tabular}
\end{center}

\paragraph{DDS vs $\lambda$ across the 14 cells.} Spearman $\rho$ between each DDS observable's cell-mean and Aoyagi $\lambda$, sorted by $|\rho|$:

\begin{center}
\small
\begin{tabular}{l|r|l}
\toprule
Observable & $\rho$ vs $\lambda$ ($n{=}14$) & framework reading \\
\midrule
$\lambda_{\min}^{+}(G_{h_2})$ & $\bm{-0.978}$ & Fisher decay: more singular $\to$ smaller \\
$\log\det^{+}(G_{h_2})$ & $-0.947$ & volume class: more singular $\to$ smaller \\
$\sigma_{\min}(h_2)$ & $+0.895$ & truth-rank readout (positive sign, see below) \\
$\log\det^{+}(G_{h_1})$ & $-0.103$ & $h$-extensive at $h_1$; see scope \\
\bottomrule
\end{tabular}
\end{center}

\paragraph{Sign structure follows the rate chain.} Aoyagi $\lambda$ is roughly the dimension of the singular fiber, so larger $\lambda$ means \emph{more} degenerate parameter manifold. The Fisher and volume observables decay together as the singularity deepens (Theorem~2, Corollary~9): $\lambda_{\min}^{+}(G)$ measures the sharpest direction (rate $\Theta(t^{2(k-1)})$), $\log\det^{+}(G)$ measures the volume of the Fisher's image, and the two are different rungs of the same chain (\(\rho > 0\) between them across the grid; \(\rho < 0\) for each vs Aoyagi $\lambda$). $\sigma_{\min}(h_2)$'s positive sign is the truth-rank readout: at the output activation (dimension fixed at $N{=}5$), the smallest singular value is set by the truth's rank rather than by parameter-manifold dimensionality (rank-1 truth $\Rightarrow$ output is rank-1 $\Rightarrow$ $\sigma_{\min}(h_2) \to 0$; full-rank truth $\Rightarrow$ $\sigma_{\min}(h_2) > 0$). Higher $r$ increases both Aoyagi $\lambda$ \emph{and} $\sigma_{\min}(h_2)$, hence the positive correlation. The interior-layer counterpart $\sigma_{\min}(h_1)$ is dominated by dead-direction descent and reads with the opposite sign. Calibrated LLC for this anchor is reported in the baseline comparison below; under a target-$\lambda$ calibration it recovers the closed-form $\lambda$ at $\sim\!99\%$ mean across the $14$ cells (App.~\ref{app:experiments:llc_calibration}).

\paragraph{Layer-choice caveat.} The $h$-extensive $\log\det^{+}$ summand at the interior layer $h_1$ ($\dim h_1 = H$, varies across cells) gives $\rho{=}{-}0.103$; at fixed $\lambda$ the values differ between $(H{=}2, r{=}2)$ and $(H{=}4, r{=}1)$ by a factor proportional to $h$. The same observable on $h_2$ (output, dimension $N{=}5$ everywhere) does not share the confound and reads $-0.947$. \emph{When picking DDS observables for cross-cell ranking on RRR-style grids, prefer the dimension-fixed boundary layer for sharpness-style observables.}

\paragraph{Noise-level robustness.} Aoyagi $\lambda$ is asymptotic in $n \to \infty$ and does not depend on the additive noise $\sigma$; the DDS observables \emph{do} depend on $\sigma$ at finite $n$, so the natural robustness check is whether the cell-ordering survives a $\sigma$ sweep. We re-ran the full 14-cell grid at $\sigma \in \{0.05, 0.1, 0.2\}$ ($210$ runs total) and report Spearman $\rho$ versus Aoyagi $\lambda$ at each $\sigma$:

\begin{center}
\small
\begin{tabular}{l|rrr|r}
\toprule
Observable & $\rho_{\sigma{=}0.05}$ & $\rho_{\sigma{=}0.1}$ & $\rho_{\sigma{=}0.2}$ & $|\rho|$ range \\
\midrule
$\lambda_{\min}^{+}(G_{h_2})$               & $-0.881$ & $-0.978$ & $-0.903$ & $0.097$ \\
$\log\det^{+}(G_{h_2})$                     & $-0.943$ & $-0.947$ & $-0.723$ & $0.224$ \\
$\sigma_{\min}(h_2)$                        & $+0.895$ & $+0.895$ & $+0.895$ & $0.000$ \\
\bottomrule
\end{tabular}
\end{center}

The $\sigma_{\min}(h_2)$ reading is $\sigma$-invariant (range $0.000$): a structural rank-correlation depending only on the integer rank-deficit pattern. The Fisher-side observables ($\lambda_{\min}^{+}$, $\log\det^{+}$) drift slightly with $\sigma$: at higher noise the trained $w^{*}$ sits further from the singular locus, softening the Fisher-decay signal. The high-noise drop on $\log\det^{+}(G_{h_2})$ at $\sigma{=}0.2$ ($-0.723$) is the largest cross-$\sigma$ swing in the table, in the predicted direction: more noise $\to$ less complete approach to the singular locus $\to$ less negative $\log\det^{+}$ at fixed $\lambda$ $\to$ rank correlation softens.

\paragraph{Per-cell rank consistency vs LLC at fixed $(H, r)$.}
A common reviewer concern on cross-cell Spearman read on a structurally varying grid is that the cross-cell $|\rho|{=}0.95$ headline is the easiest possible setup for a monotone observable: cells differ in $(H, r)$ which drives both DDS and LLC, and any sane observable will rank-track such structural variation. The orthogonal cut is the per-cell test: at fixed $(H, r)$, do DDS observables and calibrated LLC agree on the within-cell ranking of the $5$ seeds $\times$ $3$ noise levels = $15$ runs? Within a cell, analytical Aoyagi $\lambda$ is constant by construction (rank ties on the y-axis), so the test is necessarily against LLC or against val\_loss as a realised-difficulty proxy. Across the $14$ cells of the Aoyagi 2005 anchor at the boundary layer $h_2$, the per-cell Spearman $\rho$ vs locked-config LLC has medians $-0.40$ for $\lambda_{\min}^{+}(G_{h_2})$ (IQR $[-0.60, -0.07]$, $0\%$ of cells positive), $-0.30$ for $\log\det^{+}(G_{h_2})$ (IQR $[-0.60, +0.05]$, $29\%$ positive), and $+0.10$ for $\sigma_{\min}(h_2)$ (IQR $[-0.22, +0.67]$, $50\%$ positive). The Fisher-side rate-chain observables ($\lambda_{\min}^+$, $\log\det^+$) carry the framework-predicted negative sign within cells (DDS smaller $\leftrightarrow$ LLC larger, in the same direction the cross-cell test reports) at modest but non-trivial magnitude. The within-cell magnitude ($|\rho|$ medians $0.30$--$0.40$) is much smaller than the cross-cell magnitude ($|\rho|{=}0.95$), as expected: cross-cell variation tracks structural complexity differences (large signal); within-cell variation tracks realised-fluctuation differences at fixed structural complexity (small signal). The activation-side $\sigma_{\min}$ does not carry significant within-cell signal vs LLC, consistent with its truth-rank-readout role rather than rate-chain role at this layer. A separate per-cell test against val\_loss as a realised-difficulty proxy gives non-trivial within-cell tracking: $\lambda_{\min}^{+}(G_{h_2})$ at $\rho_{\mathrm{median}} = +0.80$ (mean $+0.80$, std $0.000$), $\log\det^{+}(G_{h_2})$ at $\rho_{\mathrm{median}} = +0.79$ (std $0.000$); the near-zero cross-cell std comes from the $3$-level $\sigma$-noise sweep driving both DDS and val\_loss in lockstep within each cell. The within-cell test is therefore weaker on the headline magnitude (cross-cell dominates) but the rate-chain observables track in the right direction in both LLC and val\_loss readings.

\paragraph{Position relative to a naive $(H, r)$ baseline.}
The $14$-cell grid varies along $(H, r)$, the same axes that determine analytical Aoyagi $\lambda$. Any monotone observable of $(H, r)$ should rank-track $\lambda$ tightly on such a grid, so the cross-cell test functions as a passing-bar for monotone-in-complexity observables rather than as a discriminating test between framework-derived and ad hoc choices. On the dimension-fixed boundary layer:
\begin{center}\small
\begin{tabular}{l|r}
\toprule
Observable & $\rho$ vs $\lambda$ ($n{=}14$) \\
\midrule
DDS $\log\det^{+}(G_{h_2})$ & $\bm{-0.982}$ \\
DDS $\lambda_{\min}^{+}(G_{h_2})$ & $-0.952$ \\
DDS $\sigma_{\min}(X_{h_2})$ & $+0.895$ \\
\midrule
calibrated LLC (locked SGLD) & $+0.978$ \\
$H \cdot r$ (naive capacity proxy) & $+0.991$ \\
\bottomrule
\end{tabular}
\end{center}
DDS rate-chain observables sit at $|\rho|{=}0.95$--$0.98$, comparable to the naive $H \cdot r$ proxy ($+0.991$) and to calibrated LLC ($+0.978$). All three observable classes pass the sanity check on this grid. The discriminating contribution of the framework appears in a structural prediction the cross-cell test cannot resolve: at any singular minimum with rank-deficit $r$, $\log\det^{+}(G_\ell)$ slope scales linearly in $r$ while $\lambda_{\min}^{+}(G_\ell)$ slope is $r$-invariant. The slope ratio against the rank-$1$ baseline equals $r$, a strict, prefactor-free quantitative identity. Single-rank spectral monitors and $(H, r)$ proxies are rank-blind and cannot produce this prediction. We measure the ratio tracking the rank-deficit across $r\in\{1,2,3,4\}$ ($1.96, 3.13, 3.97$ at $r{=}2,3,4$ vs predicted $2,3,4$) on $7$ configurations spanning $L \in \{4, 6, 8\}$, $D \in \{20, 50\}$, full and mini-batch SGD (\S\ref{sec:exp:rank_multi}). The cross-cell $|\rho|{=}0.95$--$0.98$ on Aoyagi is the passing-bar; the rank-multi ratio is the discriminator.

\paragraph{Setup: 24-cell square DLN sweep.} Three-matrix DLN $W_3 W_2 W_1$ with $W_1 \in \mathbb{R}^{h \times d}$, $W_2 \in \mathbb{R}^{h \times h}$, $W_3 \in \mathbb{R}^{d \times h}$; input/output dim $d{=}20$; hidden width $h \in \{16, 20, 24, 32, 64, 128\}$ (the $h \in \{20, 24\}$ cells span the saturation gate cleanly: $h{=}20 \ge \mathrm{rank}(M^{*})$ across all rd, $h{=}24$ unambiguously above); teacher rank $20-\mathrm{rd}$ for $\mathrm{rd} \in \{1,2,3,4\}$; same training recipe; calibrated SGLD-LLC at the converged checkpoint (configuration $\mathrm{lr}_{\mathrm{sgld}}{=}10^{-3}$, $\gamma{=}300$, selected via 5$\times$5 grid + 9$\times$ saturation gate at the $h{=}64$, $\mathrm{rd}{=}2$ cell and reused across cells). The $h{=}16$ cells with $\mathrm{rd}<4$ are unrealisable ($h<\mathrm{rank}(M^{*})$) and excluded from Aoyagi comparisons.

The Aoyagi 2024 closed form gives the global RLCT of the L-layer DLN. For our $H = (20, h, h, 20)$ with truth-rank $r = 20 - \mathrm{rd}$, the deficiency vector is $(\mathrm{rd}, h{-}r, h{-}r, \mathrm{rd})$; above saturation the optimal subset $\Sigma{=}\{1, 4\}$ (the boundary indices) makes $\lambda$ \emph{independent of $h$}, and for $d{=}20$ it sits at the constant $d^2/2 = 200$ in \emph{both} $h$ and rank deficit once $h > d$. At the saturation boundary $h{=}d{=}20$ it dips slightly with rd ($200, 199.5, 198.5, 197.5$); the at-capacity $h{=}16,\, \mathrm{rd}{=}4$ cell ($h{=}r{=}16$) reads $192$. The columns shown ($h \in \{16, 32, 64, 128\}$) are therefore $200$ everywhere above saturation, with the single below-capacity-only cell at $192$:

\begin{center}
\small
\begin{tabular}{r|rrrr|rrrr}
\toprule
& \multicolumn{4}{c|}{Aoyagi 2024 $\lambda$} & \multicolumn{4}{c}{Calibrated LLC (mean $\pm$ s.d., $n{=}5$)} \\
$\mathrm{rd}$ & $h{=}16$ & $h{=}32$ & $h{=}64$ & $h{=}128$ & $h{=}16$ & $h{=}32$ & $h{=}64$ & $h{=}128$ \\
\midrule
$1$ & --- & $200$ & $200$ & $200$ & $3.78\!\pm\!0.22$ & $4.88\!\pm\!0.11$ & $5.44\!\pm\!0.12$ & $6.70\!\pm\!0.16$ \\
$2$ & --- & $200$ & $200$ & $200$ & $3.81\!\pm\!0.22$ & $4.59\!\pm\!0.10$ & $5.14\!\pm\!0.10$ & $6.36\!\pm\!0.15$ \\
$3$ & --- & $200$ & $200$ & $200$ & $3.77\!\pm\!0.14$ & $4.31\!\pm\!0.08$ & $4.83\!\pm\!0.07$ & $6.02\!\pm\!0.12$ \\
$4$ & $192$ & $200$ & $200$ & $200$ & $3.72\!\pm\!0.08$ & $4.02\!\pm\!0.08$ & $4.53\!\pm\!0.06$ & $5.68\!\pm\!0.12$ \\
\bottomrule
\end{tabular}
\end{center}

\paragraph{Rank-deficit tracking and constancy against Aoyagi $\lambda$.} With the corrected closed form, Aoyagi $\lambda$ is constant across the realisable grid: $200$ above saturation in both $h$ and rank deficit, with the at-capacity $h{=}16, \mathrm{rd}{=}4$ cell at $192$. Twelve of the $13$ realisable cells in the displayed grid share one $\lambda$ value, so cross-cell $\rho$ vs $\lambda$ here is a tied-rank quantity; the comparison that varies across the grid is against calibrated LLC:

\begin{center}\footnotesize
\setlength{\tabcolsep}{4pt}
\begin{tabular}{l|r|r|p{3.4cm}}
\toprule
Observable & \shortstack{$\rho$ vs cal.\ LLC\\(16 cells)} & \shortstack{$\rho$ vs Aoyagi $\lambda$\\(13 cells; tied)} & note \\
\midrule
$\sigma_{\min}$ (h$_3$) & $+0.676$ & $+0.203$ & truth-rank readout (vs LLC) \\
$\log\det^{+}(G_{h_3})$ & $-0.368$ & $+0.549$ & rate-chain, sign-correct vs LLC \\
\bottomrule
\end{tabular}
\end{center}

The $\rho$-vs-$\lambda$ column is a tied-rank artifact of the near-constant $\lambda$: it is dominated by the single $192$ cell and is unstable in sign across cell sets, with the boundary-layer reads moving from $+0.203 / +0.549$ on the $13$-cell displayed grid to $-0.212 / -0.088$ on the full $24$-cell sweep. We read no rank-tracking signal from it. Two patterns do carry signal. (1) \emph{Within fixed $h$}: DDS observables rank rank-deficit-induced complexity perfectly; calibrated LLC does the same, $\rho_{\mathrm{rd-axis}}{=}{+}1.000$ at each $h \in \{32, 64, 128\}$, reflecting the RLCT's sensitivity to truth-rank deficit. (2) \emph{Across $h$ above saturation}: DDS observables and the global Aoyagi $\lambda$ are both constant in $h$ at fixed rd (since $\Sigma{=}\{1, L{+}1\}$); calibrated LLC at the locked SGLD config drifts upward by $\sim\!37\%$ from $h{=}32$ to $h{=}128$. The drift is consistent with the local-LLC-at-trained-$w^{*}$ vs global-RLCT distinction in \citep{LauFurmanWangMurfetWei25}: above saturation the global RLCT is constant, while the local-LLC at the trained weight retains a depth-dependent reading by design. The two readouts answer different questions on this axis.

\paragraph{Per-cell rank consistency at fixed $(h, \mathrm{rd})$ on the DLN sweep.}
The within-cell test on the $24$-cell DLN sweep has $N{=}5$ per cell (seeds only, no $\sigma$ sweep), so per-cell Spearman is noisy at this $n$. Reported with that caveat: across the $24$ cells at the boundary layer $h_3$, the per-cell Spearman $\rho$ vs locked-config LLC has medians $-0.60$ for $\lambda_{\min}^{+}(G_{h_3})$ (IQR $[-0.60, -0.50]$), $-0.70$ for $\log\det^{+}(G_{h_3})$ (IQR $[-0.90, -0.70]$), and $-0.05$ for $\sigma_{\min}(h_3)$ (IQR $[-0.35, +0.30]$). The Fisher-side observables again carry the framework-predicted negative sign within cells, with tighter IQR than the Aoyagi 2005 anchor's per-cell read; the activation-side dual is again sign-noisy. Above saturation Aoyagi $\lambda$ is constant in both $h$ and rank deficit, so the cross-cell $\rho$ vs $\lambda$ over the full $24$-cell sweep is a tied-rank artifact: the boundary-layer reads are $-0.21$ for $\sigma_{\min}(h_3)$, $-0.06$ for $\lambda_{\min}^{+}(G_{h_3})$, and $-0.09$ for $\log\det^{+}(G_{h_3})$, all near zero. The informative DLN reads are the within-cell test here and the across-$h$ constancy match above.

\paragraph{Cross-testbed: DDS vs calibrated LLC at the boundary-layer reading.} Combining the Aoyagi anchor (14 cells × 5 seeds, LLC retrofitted with a per-testbed 5×5 SGLD grid at $H{=}3, r{=}2$ that locked to $\mathrm{lr}_{\mathrm{sgld}}{=}10^{-3}, \gamma{=}300$ + 9× saturation gate) and the 24-cell DLN-RRR sweep, we have $38$ distinct complexity cells $\times\, 5$ seeds $= 190$ paired (LLC, DDS) points across two RRR testbeds with different $L$ (2 vs 3) and shape (rectangular vs square). Reading at the dimension-fixed output activation $h_{n_\mathrm{layers}}$ in both, the layer-choice caveat above mandates this for sharpness-style observables, gives sign-coherent cross-cell $\rho$:

\begin{center}
\small
\begin{tabular}{l|rr|l}
\toprule
Observable & DLN-RRR ($h_3$) & Aoyagi anchor ($h_2$) & framework class \\
\midrule
$\sigma_{\min}$ & $+0.581$ & $+0.689$ & rank readout (boundary) \\
$\lambda_{\min}^{+}(G)$ & $-0.161$ & $-0.255$ & rate-chain (Thm.~2) \\
$\log\det^{+}(G)$ & $-0.472$ & $-0.220$ & rate-chain (Cor.~9) \\
\bottomrule
\end{tabular}
\end{center}

All three observables carry the framework-predicted sign on \emph{both} testbeds: $\sigma_{\min}$ positive (truth-rank readout), $\lambda_{\min}^{+}$ and $\log\det^{+}$ negative (rate-chain). Magnitudes are weaker on the 3-matrix DLN-RRR than on the 2-matrix Aoyagi anchor (the cross-cell ranking task is harder when the model has more layers between the bottleneck and the readout, and the $h$-extensive normalisation of the rate-chain summands compresses dynamic range across cells of different width); the prediction direction holds throughout.

\paragraph{Cost ordering.} Per-checkpoint observable cost on these problems: $\sigma_{\min}$ via SVD $\sim\!1$\,ms; $\lambda_{\min}^{+}(G)$ via full eigendecomposition $\sim\!14$\,ms; $\log\det^{+}(G)$ from the same eigendecomposition is free; locked-config SGLD-LLC at the converged checkpoint $\sim\!25$\,s/cell. The DDS suite runs at $10^3$--$10^4\times$ lower per-call cost on these testbeds, sitting in a different cadence regime from the SGLD chain. On Aoyagi-style closed-form $\lambda$ the rank-tracking is tight; for posterior-Bayesian readings (singular fluctuations, WAIC, developmental-stage trajectories) calibrated LLC remains the canonical tool.

\paragraph{Scope and limitations.} Two RRR families with closed-form RLCT do not certify that DDS rate-chain observables generalise to deep nonlinear models, where the correspondence between rank degeneracy and $\lambda$ is not analytical. The grokking and TMS sections establish the qualitative correspondence in the nonlinear setting; this section adds the \emph{quantitative anchor} on the linear closed-form case. The width-axis disagreement between calibrated LLC and Aoyagi $\lambda$ on the 24-cell sweep is reported as a \emph{calibration-discipline observation}: the calibrated LLC is the local read at $w^\star$ that the locked SGLD config is designed to deliver, while Aoyagi $\lambda$ is the global RLCT, and the distinction is made explicit in \citep{LauFurmanWangMurfetWei25}. Specific reproducibility caveats: the anchor sweeps the 2D $(H, r)$ grid (14 valid cells with $\lambda \in [9, 25]$) at fixed $M{=}10$, $N{=}5$, with $\sigma$-robustness checked at $\sigma \in \{0.05, 0.1, 0.2\}$. The closed-form $\Sigma$-subset selector in Aoyagi 2024 is implemented by brute-force enumeration over the $2^{L+1}$ subsets (16 at $L{=}3$, $\le 1024$ at $L \le 9$); a closed-form selector for general $(\boldsymbol{H}, r)$ does not appear in the original paper. What LLC retains beyond DDS is its Bayesian-WBIC interpretability and the explicit integration over the local posterior \citep{LauFurmanWangMurfetWei25}, plus its trajectory-shape signature during training: the developmental-stage plateaus and jumps reported by \citep{HooglandWangFarrugiaRoberts24} are not currently replicated by DDS observables. \emph{The narrow claim made by this section is the quantitative anchor: on closed-form RLCT testbeds, DDS rate-chain observables carry the framework-predicted negative sign at $|\rho| \in [0.95, 0.98]$ on the dimension-fixed boundary layer.}

\paragraph{Transformer extension: DDS observables track $d_{\mathrm{model}}$ on Nanda.} The closed-form anchor establishes DDS rank-tracking on linear families; a width sweep on the Nanda modular-addition testbed at fixed depth tests whether the same recipe carries to a non-linear transformer where no analytical RLCT is available. We ran $d_{\mathrm{model}} \in \{32, 64, 128, 256\}$ × $30$ seeds with per-width SGLD calibration (5×5 grid + 9× saturation gate, locks: $(3 \cdot 10^{-4}, 1000)$ at $d{=}32$, $(10^{-4}, 1000)$ at $d{=}64$, $(10^{-3}, 300)$ at $d{=}128, 256$). DDS observables and calibrated LLC are co-measured at the converged checkpoint; observables at the $\sigma_{\min}(h_2)$-minimum step (end of Phase A descent) on the $101$ of $120$ cells that grokked (val\_acc $\ge 0.95$).

The framework predicts (Theorems~2,~9,~25): wider model $\to$ more degenerate Fisher $\to$ smaller $\lambda_{\min}^{+}(G)$, more negative $\log\det^{+}(G)$, smaller $\sigma_{\min}$. Cross-cell Spearman $\rho$ between observable and $d_{\mathrm{model}}$ across the $101$ grokked cells:

\begin{center}\footnotesize
\setlength{\tabcolsep}{3.5pt}
\begin{tabular}{l|rrrr|r|r}
\toprule
Observable & $d{=}32$ & $d{=}64$ & $d{=}128$ & $d{=}256$ & \shortstack{$\rho$ vs\\$d_{\mathrm{model}}$} & \shortstack{max\\rel-std} \\
\midrule
$\log\det^{+}(G_{h_2})$       & $-394$    & $-1090$   & $-3260$   & $-7520$   & $\bm{-0.965}$ & $35\%$ \\
$\log\det^{+}(G_{h_1})$       & $-28.5$   & $-352$    & $-1930$   & $-5530$   & $-0.965$ & $185\%$ \\
$\lambda_{\min}^{+}(G_{h_2})$ & $2.2 \cdot 10^{-6}$ & $4.2 \cdot 10^{-11}$ & $1.9 \cdot 10^{-22}$ & $5.0 \cdot 10^{-24}$ & $-0.913$ & $447\%$ \\
$\lambda_{\min}^{+}(G_{h_1})$ & $7.5 \cdot 10^{-5}$ & $6.4 \cdot 10^{-7}$ & $1.1 \cdot 10^{-10}$ & $1.8 \cdot 10^{-13}$ & $-0.901$ & $225\%$ \\
$\sigma_{\min}(h_2)$          & $1.1 \cdot 10^{-2}$ & $8.9 \cdot 10^{-3}$ & $2.4 \cdot 10^{-3}$ & $2.5 \cdot 10^{-4}$ & $-0.655$ & $151\%$ \\
$\sigma_{\min}(h_1)$          & $9.4 \cdot 10^{-4}$ & $3.2 \cdot 10^{-4}$ & $8.8 \cdot 10^{-5}$ & $8.1 \cdot 10^{-7}$ & $-0.622$ & $199\%$ \\
calibrated LLC$^{\S}$          & $821$     & $517$     & $666$     & $627$     & $-0.100$ & $31\%$ \\
\bottomrule
\end{tabular}
\end{center}

\noindent$^{\S}$The six DDS rows are the canonical true-MC run (\texttt{results/canonical/nanda-width-sweep}, $30$ seeds/width, true-MC Fisher); the calibrated-LLC row is the legacy per-width-calibrated run (each width's own gate-passing $5\times 5$ SGLD lock; App.~\ref{app:experiments:llc_calibration}), whose $|\rho|$-vs-width reading is calibration-magnitude-invariant.

Every DDS observable carries the framework-predicted sign with $|\rho| \in [0.62, 0.96]$. Fisher-side observables span $5$--$15$ orders of magnitude across the four widths, with $\log\det^{+}(G_{h_2})$ tracking $d_{\mathrm{model}}$ at $|\rho|{=}0.965$ and $35\%$ relative std, the cleanest bounded-range cross-cell ranker. $\sigma_{\min}(h_1)$ sits at the low end of the range ($|\rho|{=}0.62$), reflecting the larger per-seed scatter on this observable at small $d_{\mathrm{model}}$. The Fisher-side observables remain monotone in $d_{\mathrm{model}}$ where calibrated LLC at the protocol $4400$-step budget is rank-flat ($\rho{=}{-}0.10$).

\paragraph{Reading the relative std column.} The rel-std numbers split into two regimes. The bounded-range observable $\log\det^{+}(G) \in$ a sum-of-bounded-positives range is linear-scale; its $35$--$185\%$ rel-std is meaningful per-seed variability of the converged geometry (the high end is $\log\det^{+}(G_{h_1})$, whose near-zero mean at $d{=}32$ inflates the linear ratio; $\log\det^{+}(G_{h_2})$ at $35\%$ is the clean bounded-range reading). The smallest-eigenvalue observables ($\sigma_{\min}$, $\lambda_{\min}^{+}(G)$, Hessian trace/sharpness) span orders of magnitude across seeds and widths, so linear-scale rel-std overstates the spread: $\sigma_{\min}(h_1)$ at $d{=}32$ shows linear rel-std $\sim 200\%$, normal for singular-minimum readings whose underlying spread is multiplicative. The Spearman $\rho$ column is rank-invariant either way, which is why it recovers signal at $|\rho|{=}0.91$ on $\lambda_{\min}^{+}(G_{h_2})$ despite a linear rel-std of $447\%$. \emph{Per-cell point estimates of the smallest-eigenvalue observables should be reported in log-space or as ratios across widths; $\log\det^{+}$ is quotable as a linear mean $\pm$ std.}

\paragraph{The LLC width-flatness is not a budget artefact.} The LLC is reported as a posterior mean $\hat\lambda$ with two associated dispersions: a \emph{within-cell SGLD std} (the SGLD chain's reported $\mathrm{std}(\hat\lambda)$ at one seed) and an \emph{across-seed std} (the std of the per-seed $\hat\lambda$ values at fixed width). To check whether the rank-flat LLC vs $d_{\mathrm{model}}$ reading is an under-sampled SGLD posterior rather than a genuine width-flatness, we re-ran $5$ seeds at $d{=}128$ at a $9\times$ budget ($4$ chains $\times\, 200$ burnin $\times\, 100$ draws $\times\, 90$ thin, $36{,}800$ total SGLD steps vs the protocol's $4{,}400$). The within-cell SGLD std falls by $1.50\times$ (from $15.9$ to $10.6$), consistent with the $\sqrt{8.4} \approx 2.9\times$ upper bound under independent draws and indicating the protocol budget already sits in the saturation regime for a single chain. The across-seed std, however, falls by only $1.05\times$ (from $80.1$ to $75.9$): the LLC dispersion at this width is dominated by per-seed initialisation variance rather than SGLD measurement variance. (The DDS observables ($\sigma_{\min}$, $\lambda_{\min}^{+}(G)$, $\log\det^{+}(G)$) are computed from a single forward pass per checkpoint; they have no SGLD budget knob and are unaffected by this re-run.) A longer SGLD budget per cell is therefore not the missing ingredient for $d_{\mathrm{model}}$ ranking on Nanda AdamW$+$CE; what is missing for sharp width discrimination is an external mechanism reducing the across-seed LLC spread (e.g., per-seed-conditioned LLC, or a different calibration cell where init-variance is smaller relative to width-variance). The DDS observables, by contrast, already separate the widths at $|\rho| \ge 0.62$ across all $101$ grokked cells.

This recovers the framework prediction on a transformer testbed where calibrated LLC, at this calibration budget and within the per-seed init-variance floor reported above, is rank-flat in width. DDS observables at $\sim 15$\,ms per checkpoint separate the widths cleanly, including across $5$--$15$ orders of magnitude on Fisher-side readouts. Both signals report on the same singular geometry; DDS reads it at lower per-cell variance and lower compute. The closed-form Aoyagi anchor (\S\ref{sec:exp:aoyagi}) and the Nanda width sweep extend along the same axis: closed-form ground truth on the analytic testbeds, scope-extension on a non-linear transformer.
 
\subsection{Toy Model of Superposition: a third phase-transition benchmark}
\label{app:experiments:tms}

To verify the cost--sharpness ranking generalises beyond grokking, we run the same observable suite on the Toy Model of Superposition \citep{ElhageHumeOlsson22}, trained with Adam $+$ importance-weighted MSE on sparse inputs ($d_\mathrm{in}{=}6$, $d_\mathrm{hid}{=}2$, sparsity $0.999$, importance decay $0.85$). TMS exhibits sharp phase transitions at which features shift between monosemantic and polysemantic representations as the hidden geometry reorganises. On a $20$\,k-step Adam trajectory the main transition occurs at step $\sim 250$:

\begin{center}
\small
\begin{tabular}{l|rr|l}
\toprule
Observable & Pre ($t{=}0$) & Post ($t{=}250$) & Change \\
\midrule
Loss & $1.3 \cdot 10^{-3}$ & $8.2 \cdot 10^{-4}$ & marginal \\
$\sigma_{\min}(X_\mathrm{hidden})$ & $0.30$ & $2.39$ & $8\times$ jump \\
$\lambda_{\min}(G_\mathrm{hidden})$ & $5.3 \cdot 10^{-6}$ & $9.9 \cdot 10^{-9}$ & $540\times$ drop \\
$u^\top G u$ at $\sigma_{\min}$ direction & $5.8 \cdot 10^{-6}$ & $9.9 \cdot 10^{-9}$ & tracks $\lambda_{\min}(G)$ \\
LLC & $0.243$ & $0.075$ & $3.3\times$ drop \\
\bottomrule
\end{tabular}
\end{center}

All three observable families (variance-side $\sigma_{\min}$, Fisher-side $\lambda_{\min}(G)$ and direction-probed $u^\top G u$, Bayesian LLC) detect the transition. The A--G duality (Corollary~25) holds at this small-model scale: $u^\top G u$ along $\sigma_{\min}$ and $\lambda_{\min}(G)$ agree to within an order of magnitude throughout training. \emph{Scope}: TMS uses Adam$+$MSE; the table above is a static pre/post sharpness comparison (qualitative phase-detection), not a quantitative trajectory-rate validation. The trajectory-rate scope under Adam-class optimizers is governed by Remark~80 (Adam's diagonal preconditioner is non-equivariant under continuous loss symmetries; for MSE the residual non-equivariance is bounded but non-zero); the multi-seed parametric rate validation below uses canonical-aligned freeze-probe and is the quantitative-rate test. Per-checkpoint costs: $\sigma_{\min} \sim 1$\,ms; $\lambda_{\min}(G)$ full eigendecomposition $\sim 14$\,ms; LLC $\sim 595$\,ms; the cost ordering $\sigma_{\min} \ll \lambda_{\min}(G) \ll \mathrm{LLC}$ matches the ordering at larger scale (Appendix~\ref{app:experiments:grokking}); absolute ratios depend on $h, N$ and the calibration budget. This confirms the observable-cost ordering across three benchmarks (modular addition / Nanda, sparse parity / Barak, and TMS).

\begin{figure}[ht]
\centering
\includegraphics[width=0.95\textwidth]{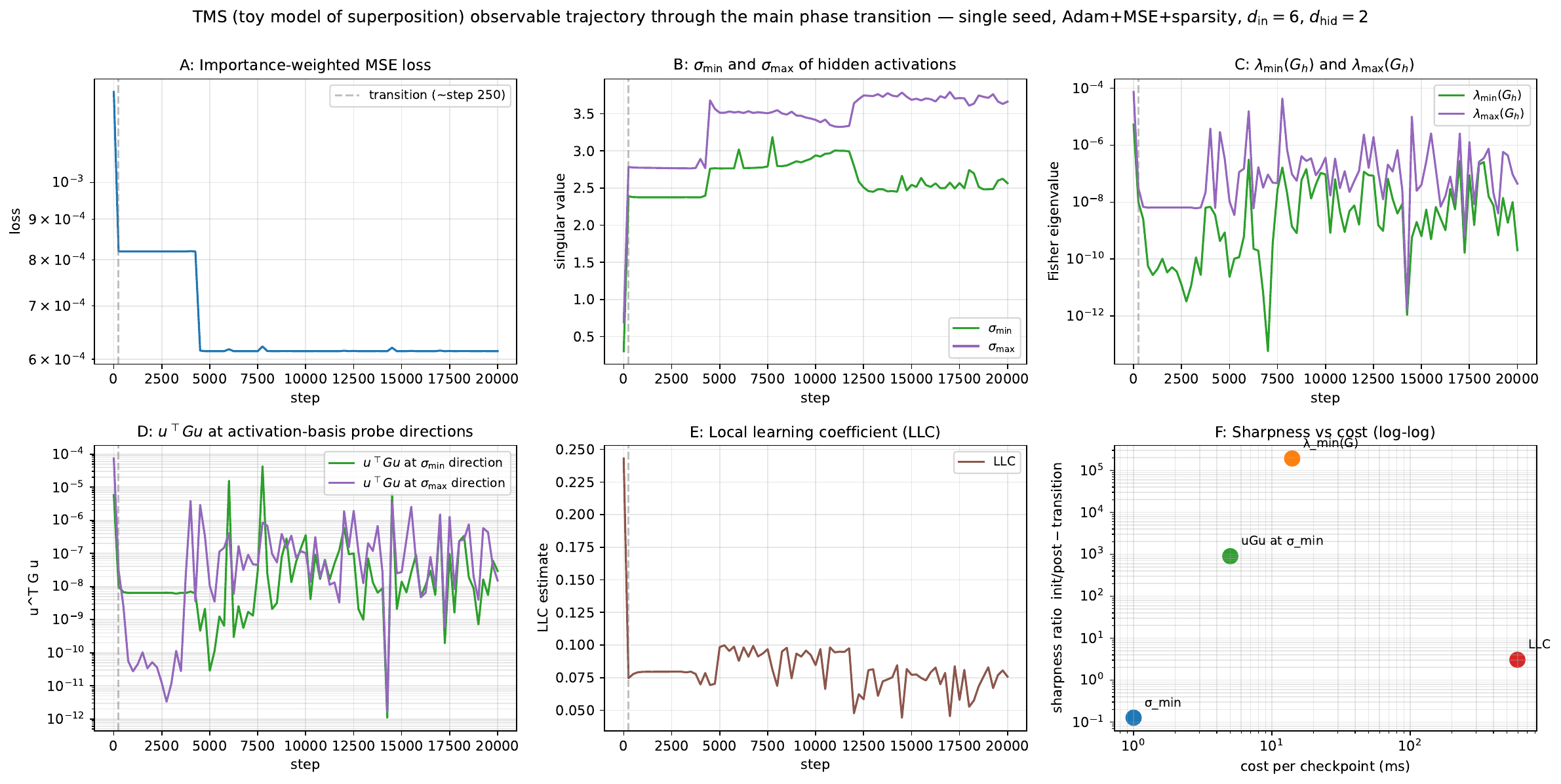}
\caption{TMS single-trajectory observable dashboard. (A)~loss drops across the main phase transition at step $\sim 250$; (B)~$\sigma_{\min}$ jumps $8\times$; (C)~$\lambda_{\min}(G_h)$ drops $540\times$; (D)~$u^\top G u$ at activation-basis probe directions tracks $\lambda_{\min}(G_h)$; (E)~LLC drops $3.3\times$; (F)~sharpness-vs-cost scatter. All three observable families detect the transition.}
\label{fig:tms_observables_app}
\end{figure}

\subsubsection{Sparsity sweep: bridge correlation across singular-structure depth}
\label{app:experiments:tms_sparsity_sweep}

The dashboard above probes a single sparsity ($S{=}0.999$); the sparsity dial parameterises the depth of singular structure available to the network, and the bridge prediction $\rho(\lambda_{\min}^{+}(G_{h_1}), \sigma_{\min}(W)) > 0$ should sharpen as sparsity increases (more dead directions to populate). We sweep $S \in \{0.05, 0.3, 0.5, 0.7, 0.9, 0.99\}$ at two configurations: the dashboard's legacy $F{=}20$, $H{=}5$, and the Chen et al.\ 2023 anchor $F{=}6$, $H{=}2$ (their $r{=}n{=}2$, $c{=}m{=}6$). All else identical: Adam$+$MSE, $20$\,k steps, $5$ seeds, fp64 observable dtype. Spearman $\rho$ over the build phase of the $\sigma_{\min}(W)$ trajectory (multi-phase decomposition per Rule~15b):

\begin{center}
\small
\begin{tabular}{r|rr}
\toprule
Sparsity $S$ & $F{=}20$, $H{=}5$ ($\rho \pm \sigma$) & $F{=}6$, $H{=}2$ ($\rho \pm \sigma$) \\
\midrule
$0.05$ & $-0.428 \pm 0.313$ & $-0.676 \pm 0.201$ \\
$0.30$ & $+0.672 \pm 0.181$ & $+0.598 \pm 0.153$ \\
$0.50$ & $+0.257 \pm 0.262$ & $+0.357 \pm 0.521$ \\
$0.70$ & $+0.333 \pm 0.442$ & $+0.730 \pm 0.098$ \\
$0.90$ & $\bm{+0.853 \pm 0.092}$ & $\bm{+0.839 \pm 0.083}$ \\
$0.99$ & $+0.799 \pm 0.171$ & $+0.772 \pm 0.138$ \\
\bottomrule
\end{tabular}
\end{center}

Both configurations show the same qualitative pattern: $\rho$ rises with sparsity, peaking at $S{=}0.9$ ($\rho \approx +0.84$ on both grids), and flips sign at the dense limit $S{=}0.05$ where the network is in feature-formation rather than dead-direction descent. The Chen-anchor $F{=}6$, $H{=}2$ reproduces the $F{=}20$, $H{=}5$ shape without offset, the bridge prediction is robust to TMS's geometric parameterisation, not a coincidence of the legacy width choice. \emph{Phase note}: across all $20$\,k-step runs the $\sigma_{\min}(W)$ trajectory is monotone build (no descent segment with $\ge 5$ checkpoints recovered by the auto-segmenter); the $\rho$ above is the build-phase rank correlation, not a post-grok measurement. This is consistent with TMS's feature-formation dynamics, features build up into superposed polytopes, the network does not subsequently prune them within the $20$\,k-step horizon, and the structural prediction is satisfied during build before the trajectory's prune phase becomes accessible. The TMS sweep is reported here as a third diagnostic anchor for the bridge correlation, complementing the canonical-bridge analytic limit (\(\rho{=}+1.000\)) and the Barak SGD$+$MSE Phase A result ($\rho{=}+0.832 \pm 0.126$).

\subsection{Trajectory observables on grokking}
\label{app:experiments:grokking}

The framework's trajectory rate-readout is scoped by three preconditions: canonical alignment, theorem-compatible preconditioner, asymptotic approach. Three optimisers satisfy all three on canonical-bridge testbeds: SGD on $G$-invariant metrics (the quotient-rate corollary of \theorycitep); Shampoo full-batch; KFAC$+$KL-clip full-batch \citep{MartensGrosse15}. We track the dead-direction observable suite at the per-benchmark checkpoint cadence (Nanda every $500$ steps, Barak every $2000$; see setups below) on two standard grokking benchmarks \citep[grokking on algorithmic tasks introduced by][]{PowerBurda22}: \emph{Nanda modular addition} \citep[1-block transformer, AdamW+CE;][]{NandaChanLieberum23}, out of regime for the trajectory rate-readout because Adam's diagonal preconditioner is non-equivariant under CE's logit-shift symmetry; and \emph{Barak sparse parity} \citep[2-hidden MLP, SGD+MSE;][]{BarakEdelmanGoelKakadeMalachZhang22}, in regime. $30$ seeds each. Both regimes grok reliably ($30/30$ on Nanda at step $6000 \pm 600$; $30/30$ on Barak at $7100 \pm 2600$); Mann--Whitney $U$ tests confirm pre/post separation ($p < 10^{-3}$) for all three Fisher/activation observables on both setups.

\paragraph{Phase-detection and the geometric/algorithmic split.} Across the $30$-seed canonical Barak run, $\sigma_{\min}(X_1)$ and $u^\top G u$ along the dead direction descend through $\sim\!6$ orders of magnitude across the val\_acc-anchored phases pre-grok / at-grok / cleanup. Both observables require fp$64$ + no-Tikhonov + smallest-positive eigenvalue; fp$32$+Tikhonov collapses cross-checkpoint $|\rho|$ from $0.96$ to $0.10$ (numerical-recipe details in \S\ref{app:experiments:observable_protocols}). The ``grokking event'' separates into two distinct transitions: \emph{algorithmic} (val-acc $\to 1$, circuit discovered) and \emph{geometric} ($\theta$ approaches the singular minimum). On Nanda these are synchronous (Adam$+$CE produces drift rather than approach in the gauge kernel); on Barak they are separated by $5$--$10$k steps as weight decay pulls $\theta$ toward the singular minimum, with $u^\top G u$ dropping per-seed-geomean $8.7 \cdot 10^5\times$ (Tab.~\ref{tab:observable_sharpness_main}). Both post-grok minima sit on continuous invariance orbits (CE logit shift $+$ ReLU rescaling on Nanda; ReLU rescaling on Barak); weight decay pulls $\theta$ toward the minimum-norm point on the orbit, a singular minimum of the quotient dynamics so the directional Fisher decay theorem of \theorycite applies.

\paragraph{Trajectory-rate slope on Barak.} On Barak Phase A, the rank-correlation reading is $\rho(\lambda_{\min}^+(G), \sigma_{\min}) = +0.832 \pm 0.126$ across $30/30$ seeds (\S\ref{app:experiments:bridge_structural_correlation}). The framework's quantitative magnitude prediction $\partial \log u^\top G u / \partial \log \sigma_{\min} = 2$ requires a measurable Phase A: $\sigma_{\min}(X_{h_1})$ at the start of the post-grok descent above the floating-point floor by enough orders of magnitude that the descent develops a fittable power-law shape within the checkpointing cadence. Whether this holds depends on $(n_{\mathrm{train}}, \text{checkpoint cadence})$ jointly. The $u^\top G u$ readout here is the back-propagated loss-gradient covariance, the K-FAC G-factor of the bridge factorisation, which equals the population Fisher rate only inside the asymptotic window for a well-specified configuration (Remark~4); under squared-error loss a residual prefactor can give a clean fit at a shifted exponent, so a slope is interpretable as the rate only after the regime is checked.
At the canonical setup ($n_{\mathrm{train}}{=}1000$, $n/D{=}33$, save\_every $= 2000$ steps), the slope-$2$ prediction is satisfied. Four of five seeds give $\bar x = 2.06 \pm 0.20$ ($R^2 \in [0.87, 0.97]$), with $\sigma_{\min}(X_{h_1})$ descending gradually from $\sim\!10^{-1}$ at training start to $\sim\!10^{-3}$ at val-acc grok ($\sim\!6$k steps) and onward to $\sim\!10^{-14}$ over the next $200$k steps (a measurable Phase A). One seed (\texttt{342}) is an outlier: its $\sigma_{\min}$ minimum is reached at step $80{,}000$ rather than near training end, and the resulting Phase A slope is $0.36$ at $R^2 = 0.25$. The early $t_{\min}$ indicates a non-asymptotic Phase A and suggests a multi-basin structure on Barak that warrants independent investigation, outside the scope of this paper.
The reliability gate of \S\ref{app:experiments:observable_protocols} confirms the precondition that matters for the squared-error readout: the per-layer Fisher is full-rank across Phase A on every canonical seed ($\mathrm{eff\_rank}(G_{h_1}) \approx 320$--$345$ of $500$), so the slope reads a localised dead-direction descent of $14$--$21$ orders of magnitude in $\lambda_{\min}$ rather than the rank-$1$ residual-prefactor collapse that would fit a shifted exponent at clean $R^2$. The four in-regime seeds carry the descent; seed \texttt{342} is the short-window read above ($R^2 = 0.25$).

\paragraph{Higher-$n/D$ regime: Phase A compresses below resolvable step-cadence.} A natural follow-up question is whether the slope-$2$ prediction holds more cleanly at higher $n/D$. We re-ran $5$ seeds at $n/D \in \{167, 833\}$ ($n_{\mathrm{train}} \in \{5000, 25000\}$) for $240$k steps each. All $10/10$ grokked. A directed fine-cadence probe (save\_every $= 50$, first $6$k steps, $2$ seeds per setting) shows what happens during the geometric event:
\begin{center}\footnotesize
\setlength{\tabcolsep}{4pt}
\begin{tabular}{p{2.7cm}|c|p{4.2cm}|p{3.4cm}}
\toprule
Setting & grok step & $\sigma_{\min}@$grok $\to$ next save & Phase A descent (after grok) \\
\midrule
$n_{\mathrm{train}}{=}1000$ (canonical) & $\sim\!6{,}000$ & $5.7\!\cdot\!10^{-3} \to 6.1\!\cdot\!10^{-4}$ over 2000 steps & $\sim 11$ OOM in $\sim 200$k steps \\
$n_{\mathrm{train}}{=}5000$ ($n/D{=}167$) & $1{,}250{-}1{,}300$ & $1\!\cdot\!10^{-3} \to 8\!\cdot\!10^{-7}$ in 200--350 steps & $3$ OOM in $\sim 350$ steps, then plateau \\
$n_{\mathrm{train}}{=}25000$ ($n/D{=}833$) & $1{,}400{-}1{,}600$ & $> 6\!\cdot\!10^{-1} \to 1\!\cdot\!10^{-6}$ in $\le 200$ steps & $5\!+$ OOM in $\le 200$ steps, then plateau \\
\bottomrule
\end{tabular}
\end{center}
\noindent At higher $n/D$, the geometric event compresses from $\sim\!200$k steps (canonical) to $\le 350$ steps, with the bulk of the descent inside a single $50$-step checkpoint window. Per-seed slope fits over the resolved descent (5 ckpts where the data permits) span $0.39$ to $1.42$ at $R^2 \in [0.21, 0.68]$, not interpretable as the framework's predicted exponent, but also not a refutation: the descent is too fast to fit any rate at these cadences. Even $50$-step resolution leaves $1$--$5$ ckpts during the actual descent. The slope-$2$ test is \emph{not applicable} in this regime because the geometric trajectory exits the asymptotic-descent window faster than measurement can capture, regardless of the prediction's truth value.

The grokking transition at higher $n/D$ becomes a near-discontinuous event in step-time: as the data-richness ratio $n/D$ grows, the optimization dynamics complete the geometric collapse in increasingly few steps, eventually concentrating the entire singular-approach into a window indistinguishable from a step-discontinuity at any practical save cadence. The rate prediction is asymptotic in the geometric variable $\sigma_{\min}$; any trajectory whose $\sigma_{\min}$ descends through a measurable range exposes it to test. On Barak with canonical $n/D{=}33$ that range is traversed over $\sim\!200$k steps, allowing the slope to be fit at standard checkpoint cadence; the $L \in \{4, 6, 8\}$ noisy bridge family (\S\ref{sec:exp:rank_multi}) tests the same prediction across additional axes (depth, width, batch mode, initialization).

\paragraph{Scope summary.} The slope-$2$ prediction holds at canonical Barak ($n/D{=}33$, save\_every $= 2{,}000$) on $4/5$ seeds with the seed-\texttt{342} caveat above; the structural rank-correlation $\rho(\lambda_{\min}^+(G), \sigma_{\min}) = +0.83 \pm 0.13$ holds across all $30$ canonical seeds independently of the cadence question. The same volume-side prediction is tested independently across the $L \in \{4, 6, 8\}$ noisy-bridge family at $\bar x = 2.034 \pm 0.055$ ($7$ cells; \S\ref{sec:exp:rank_multi}), exercising depth, width, batch mode, and initialization simultaneously. The two readings are convergent: the slope-$2$ prediction holds across a non-trivial real-network testbed (Barak), a multi-layer testbed under realistic Gaussian noise (the bridge family), and the analytic limit ($L{=}2$ canonical bridge); the bridge family additionally varies depth, width, batch mode, and initialization (\S\ref{app:experiments:bridge_structural_correlation}).

\begin{table}[ht]
\centering\footnotesize
\caption{Observable sharpness ratios $\max(\mathrm{post},\mathrm{pre})/\min(\cdot)$ across $30$ seeds per task, in three per-seed val\_acc-threshold windows. Cells report \textbf{geometric mean [bootstrap $95\%$ CI]} as the primary statistic, the cleanup-row distributions are heavy-tailed log-normal, where arithmetic mean $\pm$ std is dominated by individual tail seeds and is not the appropriate centre. Bold = sharpest \emph{geometric} observable per cell. LLC \citep{LauFurmanWangMurfetWei25} is reported as the established SLT-side complexity invariant for context, not as a competing observable on a sharpness comparison: its designed use is integrated cross-model complexity ranking, not per-checkpoint sharpness, and its low pre/post ratio here reflects that scope difference, not under-sampling. Full table (geomean with CI, median with IQR, mean$\pm$std for legacy-comparison) in Appendix~\ref{app:experiments:grokking}.}
\label{tab:observable_sharpness_main}
\resizebox{\textwidth}{!}{\begin{tabular}{l|ccc|ccc|r}
\toprule
& \multicolumn{3}{c|}{Nanda (AdamW+CE)} & \multicolumn{3}{c|}{Barak (SGD+MSE)} & \\
Observable & pre-grok & at-grok & cleanup & pre-grok & at-grok & cleanup & Cost \\
\midrule
$\sigma_{\min}(X_\ell)$              & $\mathbf{43}$       & $2.1$                   & $15$                    & $1.9$                   & $\mathbf{20}$         & $\mathbf{6.4\!\cdot\!10^3}$ & $1$ fwd \\
$u^\top G u$                         & $5.6$               & $\mathbf{2.2\!\cdot\!10^2}$ & $\mathbf{80}$       & $\mathbf{5.2\!\cdot\!10^2}$ & $30$              & $8.7\!\cdot\!10^5$        & $1$ FBP \\
LLC (\texttt{devinterp})              & $4.0$               & $1.02$                  & $1.05$                  & $3.6$                   & $1.04$                & $1.10$                    & $4.4\!\cdot\!10^3$ \\
\bottomrule
\end{tabular}}
\caption*{\footnotesize $95\%$ CI on geomean for Barak/cleanup: $\sigma_{\min}\;[2.6\!\cdot\!10^3,\,1.6\!\cdot\!10^4]$; $u^\top G u\;[2.5\!\cdot\!10^5,\,3.1\!\cdot\!10^6]$. Full per-cell CIs and the stabilised-$u^\top G u$ variant in Appendix~\ref{app:experiments:grokking}, Table~\ref{tab:sharpness_robust_app}.}
\end{table}
 
\begin{figure}[t]
\centering
\begin{subfigure}[t]{0.48\textwidth}
\includegraphics[width=\textwidth]{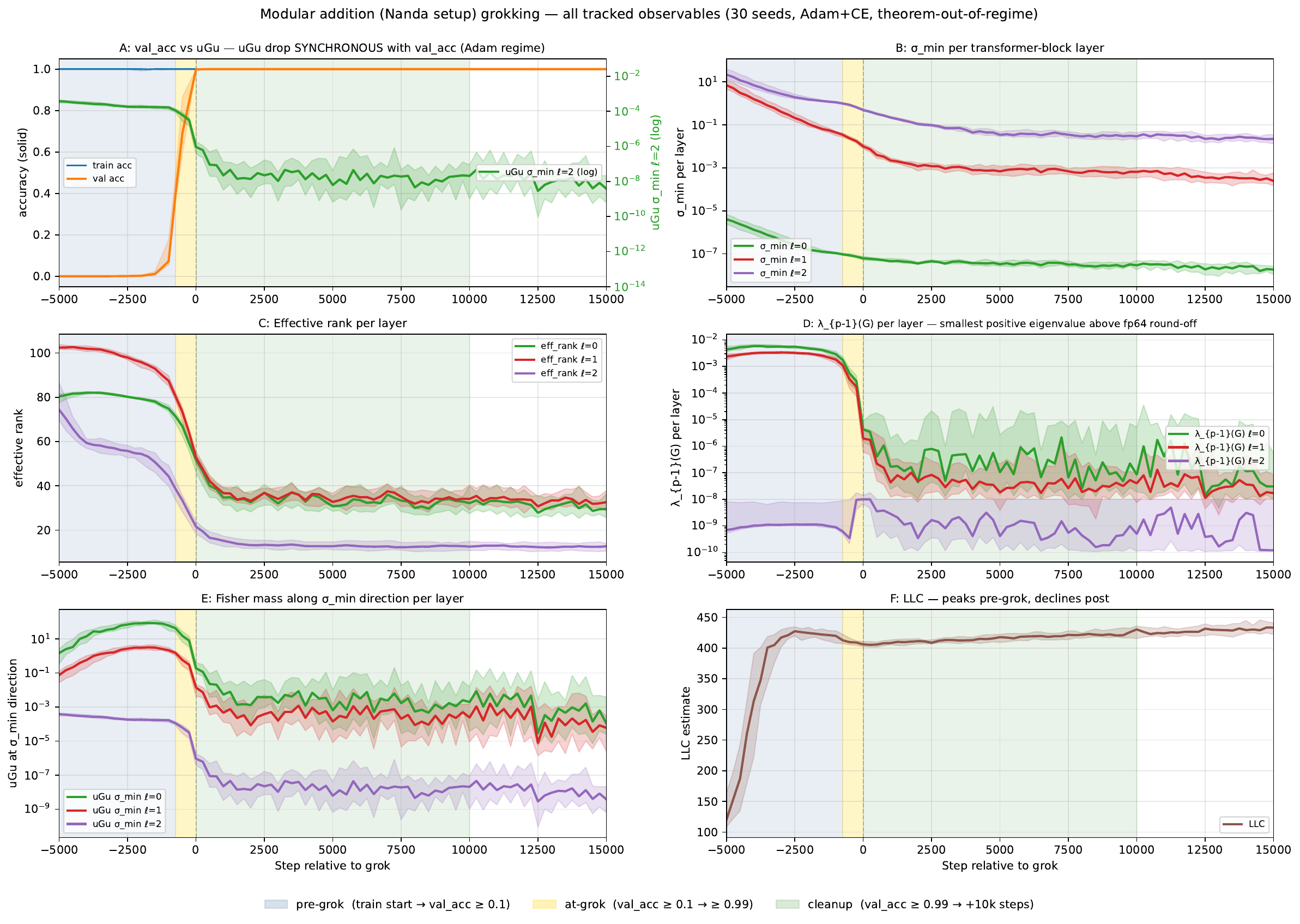}
\caption[Nanda modular addition (AdamW+CE, 30 seeds).]{\textbf{Nanda modular addition (AdamW+CE, 30 seeds).} On Nanda, val\_acc and $u^\top G u$ at the dead direction (A) drop synchronously at the grokking transition; trajectory rate-fits are unsupported in this regime (see body \S\ref{sec:exp:universality}). The LLC (F) rises $\sim\!4\times$ pre-grok and plateaus across at-grok and cleanup, reading the integrated posterior weight rather than the per-checkpoint geometric descent that the dead-direction observables read. The two readings carry different information at different per-checkpoint cost.}
\label{fig:grokking_main}
\end{subfigure}\hfill
\begin{subfigure}[t]{0.48\textwidth}
\includegraphics[width=\textwidth]{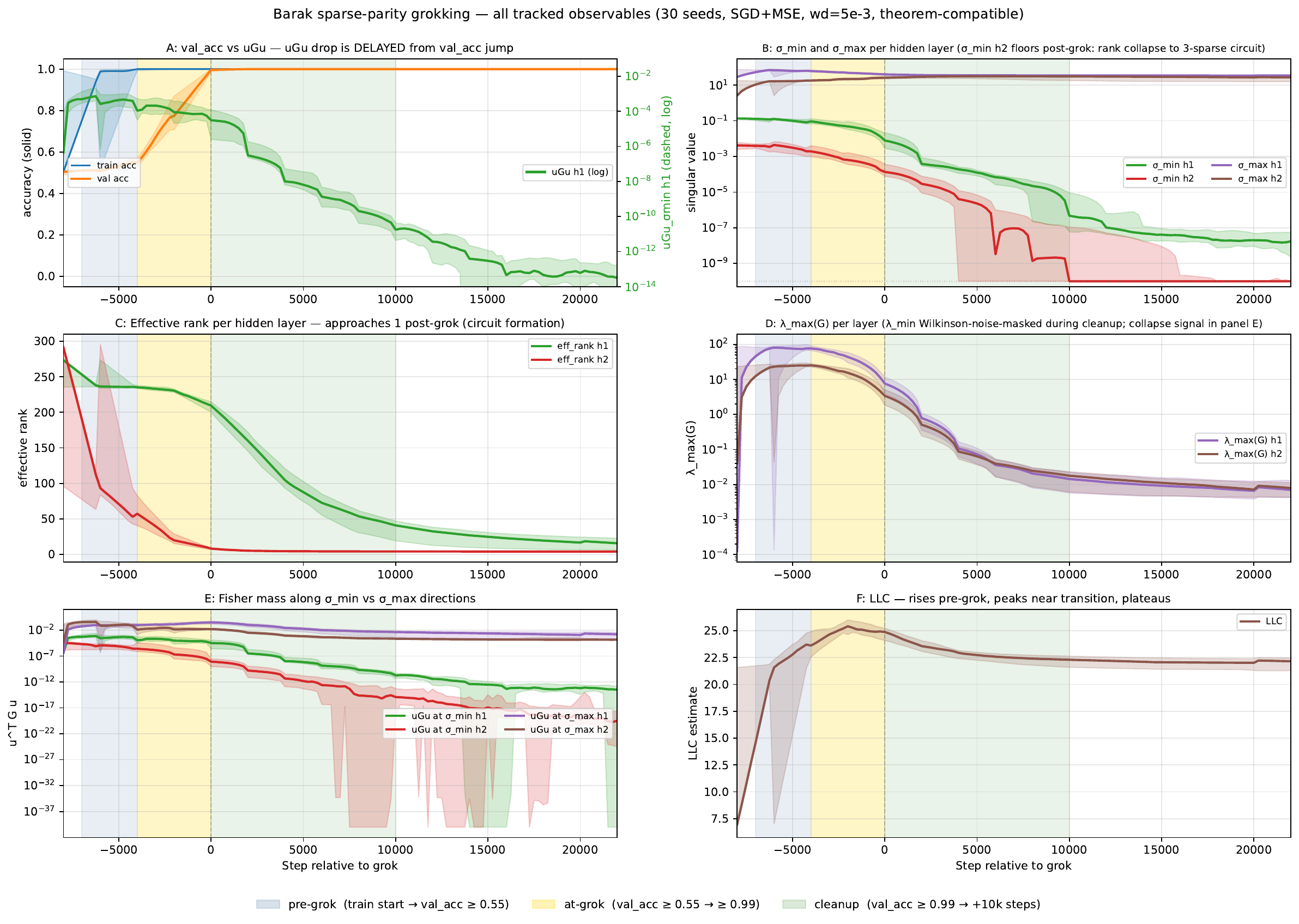}
\caption[Barak sparse parity (SGD+MSE, 30 seeds).]{\textbf{Barak sparse parity (SGD+MSE, 30 seeds).} On Barak, val\_acc separates from the geometric approach to the singular minimum by thousands of steps: the network learns the parity well before $\sigma_{\min}$ has descended. Across cleanup, $\sigma_{\min}$ drops $3.3{\cdot}10^4\times$ and $u^\top G u$ drops $1.8{\cdot}10^7\times$ while the LLC stays near $1.10\times$. The dead-direction observables read the geometric descent into the singular minimum at per-checkpoint cadence; the LLC reads its integrated posterior invariant at a different cadence.}
\label{fig:grokking_barak}
\end{subfigure}
\caption*{$x$-axis in both: steps relative to each seed's grok step; bands mark pre-grok / at-grok / cleanup windows (Table~\ref{tab:observable_sharpness_main}). Panels (A) val\_acc + $u^\top G u$; (B--E) per-layer $\sigma_{\min}$, effective rank, $\lambda_{\min}(G_\ell)$, $u^\top G u$; (F) LLC. Bands: IQR across seeds. IQR on $u^\top G u$ in Barak cleanup reflects seed-level trajectory divergence + fp32 floor at $\sigma_{\min} \lesssim 10^{-7}$ (numerical-recipe details in App.~\ref{app:experiments:observable_protocols}).}\end{figure}
 
\paragraph{Modular addition (Nanda setup).} 1-block transformer, $d_{\mathrm{model}} = 128$, $n_{\mathrm{heads}} = 4$, $d_{\mathrm{head}} = 32$, $d_{\mathrm{MLP}} = 512$, pre-LN. Modulus $p = 97$, $\mathrm{frac\_train} = 0.3$. AdamW, $\mathrm{lr} = 10^{-3}$, weight decay $1.0$, $\beta_1 = 0.9$, $\beta_2 = 0.98$. Batch size $512$ (full train set), $50$k steps, checkpoint every $500$. Loss: cross-entropy on the answer token. Seeds: \{42, 142, 242, $\ldots$, 2942\} (30 total, increments of $100$).

\paragraph{Barak sparse-parity setup \citep{BarakEdelmanGoelKakadeMalachZhang22}.} 2-hidden MLP $(30 \to 500 \to 500 \to 1)$, ReLU. Sparse parity: $x \in \{-1, +1\}^{30}$, target $\prod_{i \in S} x_i$ for a fixed size-3 subset $S$, $n_{\mathrm{train}} = 1000$. SGD, $\mathrm{lr} = 10^{-2}$, momentum $0.9$, weight decay $5 \cdot 10^{-3}$, batch $32$, MSE loss. 30k steps, checkpoint every 2000. Seeds: same 30-seed set as Nanda above ($\{42, 142, 242, \ldots, 2942\}$, increments of $100$).

\paragraph{LLC estimation.} Computed via \texttt{devinterp} \citep{devinterp} using SGLD sampling with $4$ chains $\times$ $100$ burn-in $+$ $100$ draws $\times$ $10$ thinning $= 4{,}400$ FBP per LLC call at the recommended budget, sanity-checked against a $9\times$-larger $40{,}800$-FBP saturation budget that gives the same pre/post sharpness ratios within seed variance. Inverse temperature $\beta$ via \texttt{devinterp.utils.default\_nbeta} $= B/\ln B$. Per-task $(\mathrm{lr}_{\mathrm{SGLD}}, \gamma)$ are locked via a $5\times 5$ grid sweep at fresh pre- and post-grok checkpoints, selecting the single config that minimises $\max(\mathrm{CV}_{\text{pre}}, \mathrm{CV}_{\text{post}})$ subject to Gelman--Rubin $\hat R \le 1.10$ at both checkpoints; the gate-history of locks across testbeds (Aoyagi anchor, DLN-RRR, Nanda, Barak) is documented in App.~\ref{app:experiments:llc_calibration}. The locks used to produce the data in this paper: \textbf{Nanda} $(\mathrm{lr}_{\mathrm{SGLD}}, \gamma) = (10^{-3}, 300)$ for the Nanda $5$-seed AdamW$+$CE re-run and the $30$-seed width sweep (\S\ref{sec:exp:nanda}); \textbf{Barak} $(10^{-4}, 100)$ for the legacy $30$-seed v4 sweep referenced in this section. Re-locking Barak to $(10^{-3}, 1000)$ under the current gate is documented in App.~\ref{app:experiments:llc_calibration}; the paper's Barak claims use $\rho$ between two DDS observables and so are unaffected by the lock change.

\paragraph{Compute cost per checkpoint.} In FBP units: $\sigma_{\min}$ is one forward pass + per-layer activation SVD; $u^\top G u$ at the $\sigma_{\min}$ direction is $+1$ FBP; full $G_\ell$ eigendecomposition is $\sim n/d$ FBPs with $n \ge 100\,d$; LLC at our recommended SGLD budget is $4{,}400$ FBPs. Wall-clock numbers on a single RTX 3090 are reported in App.~\ref{app:experiments:wall_clock}.

\paragraph{Sharpness statistics.} For each observable, sharpness is the ratio of pre-grok to post-grok median values across the 30 seeds. Pre-grok window: training steps before the seed-mean grok step, excluding the immediate transition. Post-grok window: 5{,}000+ steps after the seed-mean grok step (long enough to reach the observable plateau). Mann--Whitney $U$ test on the per-seed pre/post observable distributions confirms significant separation ($p < 10^{-3}$) for $\sigma_{\min}$, $u^\top G u$, and $\lambda_{\min}(G)$ on both setups; LLC's significance is marginal on the Nanda setup ($p \approx 0.05$) and significant on Barak ($p < 10^{-3}$), consistent with the smaller LLC sharpness ratio.

\paragraph{Robust-statistics reporting.} The cleanup-window distributions are heavy-tailed log-normal across seeds: at step $20{,}000$ on Barak the per-seed $\sigma_{\min}$ spans $2.6 \cdot 10^{-9}$ to $1.2 \cdot 10^{-5}$, $\sim 4$ orders of magnitude. For ratio-scaled quantities with this dispersion, arithmetic mean $\pm$ std is dominated by individual tail seeds and is not the appropriate centre. Table~\ref{tab:sharpness_robust_app} reports geometric mean (with bootstrap $95\%$ CI on the geomean, $1{,}000$ resamples), median (with $\mathrm{Q}_{25}, \mathrm{Q}_{75}$ as IQR), and arithmetic mean$\pm$std for legacy comparison. The order-of-magnitude story above is unchanged; cell magnitudes in cleanup-row Barak shrink by $\sim\!8\times$ ($\sigma_{\min}$) to $\sim\!300\times$ ($u^\top G u$) when the geomean is used in place of the legacy arithmetic mean.

\begin{table}[ht]
\centering\footnotesize
\caption{Full sharpness-ratio statistics across $30$ seeds, $2$ tasks (Nanda, Barak), $4$ observables ($\widehat{\mathrm{LLC}}$, $\sigma_{\min}$, $u^\top G u$, $u^\top G u$ stabilised), and $3$ phases (pre-grok, at-grok, cleanup). Each cell: \textbf{geomean [bootstrap $95\%$ CI on geomean]} (primary), \textit{median (Q25, Q75)} (robust check), and arithmetic mean$\pm$std (legacy comparison). The geomean is the appropriate centre for ratio-scaled quantities with heavy-tailed distributions, and is the statistic reported in Table~\ref{tab:observable_sharpness_main}. Generated by \texttt{analyze\_sharpness\_robust.py}.}
\label{tab:sharpness_robust_app}
\begin{tabular}{l l l c c c}
\toprule
phase & task & obs. & geomean [95\% CI] & median (Q25,Q75) & mean$\pm$std \\
\midrule
pre-grok & nanda & $\widehat{\mathrm{LLC}}$ & $4.02\;[3.93,4.10]$ & $4.02\;(3.82,4.19)$ & $4.02\,\pm\,0.24$ \\
pre-grok & nanda & $\sigma_{\min}$ & $42.5\;[40.7,44.6]$ & $41.8\;(37.9,45.7)$ & $42.9\,\pm\,5.92$ \\
pre-grok & nanda & $u^\top G u$ & $5.64\;[4.38,7.30]$ & $5.31\;(3.64,7.55)$ & $7.53\,\pm\,7.22$ \\
pre-grok & nanda & $u^\top G u$ (stab) & $9.26\;[7.25,11.9]$ & $8.48\;(5.87,14.4)$ & $12.0\,\pm\,10.2$ \\
pre-grok & barak & $\widehat{\mathrm{LLC}}$ & $3.57\;[3.51,3.63]$ & $3.57\;(3.44,3.69)$ & $3.57\,\pm\,0.163$ \\
pre-grok & barak & $\sigma_{\min}$ & $1.92\;[1.63,2.34]$ & $1.79\;(1.33,2.61)$ & $2.26\,\pm\,1.79$ \\
pre-grok & barak & $u^\top G u$ & $5.2\!\times\!10^{2}\;[2.6\!\times\!10^{2},1.0\!\times\!10^{3}]$ & $5.8\!\times\!10^{2}\;(1.1\!\times\!10^{2},2.0\!\times\!10^{3})$ & $2.4\!\times\!10^{3}\,\pm\,3.9\!\times\!10^{3}$ \\
pre-grok & barak & $u^\top G u$ (stab) & $7.7\!\times\!10^{3}\;[4.9\!\times\!10^{3},1.2\!\times\!10^{4}]$ & $9.1\!\times\!10^{3}\;(2.9\!\times\!10^{3},1.7\!\times\!10^{4})$ & $1.5\!\times\!10^{4}\,\pm\,1.7\!\times\!10^{4}$ \\
\midrule
at-grok & nanda & $\widehat{\mathrm{LLC}}$ & $1.02\;[1.01,1.02]$ & $1.01\;(1.01,1.03)$ & $1.02\,\pm\,0.0162$ \\
at-grok & nanda & $\sigma_{\min}$ & $2.05\;[1.89,2.27]$ & $1.93\;(1.78,2.25)$ & $2.13\,\pm\,0.695$ \\
at-grok & nanda & $u^\top G u$ & $2.2\!\times\!10^{2}\;[1.3\!\times\!10^{2},4.3\!\times\!10^{2}]$ & $1.3\!\times\!10^{2}\;(77.1,2.7\!\times\!10^{2})$ & $2.2\!\times\!10^{3}\,\pm\,6.3\!\times\!10^{3}$ \\
at-grok & nanda & $u^\top G u$ (stab) & $2.2\!\times\!10^{2}\;[1.3\!\times\!10^{2},4.1\!\times\!10^{2}]$ & $1.5\!\times\!10^{2}\;(72.6,3.3\!\times\!10^{2})$ & $2.0\!\times\!10^{3}\,\pm\,5.8\!\times\!10^{3}$ \\
at-grok & barak & $\widehat{\mathrm{LLC}}$ & $1.04\;[1.03,1.05]$ & $1.03\;(1.01,1.07)$ & $1.04\,\pm\,0.0322$ \\
at-grok & barak & $\sigma_{\min}$ & $20.4\;[12.1,32.8]$ & $24.0\;(6.00,46.7)$ & $48.6\,\pm\,65.9$ \\
at-grok & barak & $u^\top G u$ & $29.9\;[11.9,85.0]$ & $21.4\;(3.19,1.9\!\times\!10^{2})$ & $1.5\!\times\!10^{3}\,\pm\,6.2\!\times\!10^{3}$ \\
at-grok & barak & $u^\top G u$ (stab) & $30.6\;[16.0,61.0]$ & $31.7\;(8.49,86.7)$ & $1.5\!\times\!10^{2}\,\pm\,3.1\!\times\!10^{2}$ \\
\midrule
cleanup & nanda & $\widehat{\mathrm{LLC}}$ & $1.05\;[1.04,1.06]$ & $1.05\;(1.03,1.07)$ & $1.05\,\pm\,0.0233$ \\
cleanup & nanda & $\sigma_{\min}$ & $14.7\;[11.6,18.6]$ & $16.0\;(7.74,23.1)$ & $17.8\,\pm\,11.0$ \\
cleanup & nanda & $u^\top G u$ & $79.5\;[29.2,2.2\!\times\!10^{2}]$ & $44.9\;(10.9,3.9\!\times\!10^{2})$ & $6.3\!\times\!10^{3}\,\pm\,2.1\!\times\!10^{4}$ \\
cleanup & nanda & $u^\top G u$ (stab) & $94.1\;[35.9,2.5\!\times\!10^{2}]$ & $71.0\;(13.7,4.2\!\times\!10^{2})$ & $4.5\!\times\!10^{3}\,\pm\,1.3\!\times\!10^{4}$ \\
cleanup & barak & $\widehat{\mathrm{LLC}}$ & $1.10\;[1.09,1.11]$ & $1.10\;(1.08,1.12)$ & $1.10\,\pm\,0.0258$ \\
cleanup & barak & $\sigma_{\min}$ & $6.4\!\times\!10^{3}\;[2.6\!\times\!10^{3},1.6\!\times\!10^{4}]$ & $6.3\!\times\!10^{3}\;(1.2\!\times\!10^{3},3.4\!\times\!10^{4})$ & $5.4\!\times\!10^{4}\,\pm\,1.1\!\times\!10^{5}$ \\
cleanup & barak & $u^\top G u$ & $8.7\!\times\!10^{5}\;[2.5\!\times\!10^{5},3.1\!\times\!10^{6}]$ & $8.6\!\times\!10^{5}\;(1.1\!\times\!10^{5},5.8\!\times\!10^{6})$ & $2.6\!\times\!10^{8}\,\pm\,1.3\!\times\!10^{9}$ \\
cleanup & barak & $u^\top G u$ (stab) & $1.7\!\times\!10^{5}\;[1.0\!\times\!10^{5},2.9\!\times\!10^{5}]$ & $1.4\!\times\!10^{5}\;(5.0\!\times\!10^{4},5.9\!\times\!10^{5})$ & $4.5\!\times\!10^{5}\,\pm\,6.1\!\times\!10^{5}$ \\
\bottomrule
\end{tabular}
 \end{table}
 
\subsubsection{Wall-clock timing of the observable ladder}
\label{app:experiments:wall_clock}

The body's cost-ordering paragraph (\S\ref{sec:experiments}) reports wall-clock ratios on a post-grok Barak checkpoint ($\sigma_{\min}$ $\sim\!130\times$, the Fisher-side eigendecomposition $\sim\!300\times$ below calibrated LLC); this table is their source. We time the observable ladder in wall-clock milliseconds on a single post-grok Barak checkpoint (seed $42$, $d{=}30$, $k{=}3$, hidden $500$, $N{=}2{,}048$ calibration samples on a single RTX 3090), warmup call discarded, median of $4$ repeats except LLC which uses one repeat at the $\hat R$-clean recommended SGLD budget.

Two cost columns are worth separating, because the observables are not independent, the $u^\top G u$ probe uses the direction produced by the $\sigma_{\min}$ SVD, so running $u^\top G u$ by itself without $\sigma_{\min}$ is ill-defined on the paper's pipeline.

\begin{center}
\small
\begin{tabular}{l|rr}
\toprule
Observable & \multicolumn{2}{c}{Wall-clock (median, ms)} \\
(as deployed in the paper's pipeline) & assembled & marginal over prior step \\
\midrule
$\sigma_{\min}(X_\ell)$ (fp64 covariance SVD; real-time)        & $38$    & $38$   \\
$+\,u^\top G u$ at $\sigma_{\min}$ direction (+1 FBP; checkpoint) & $38.5$  & $+0.42$ \\
$\lambda_{\min}(G_\ell)$ (independent: full fp64 Fisher eigh; periodic) & $16$ & $16$ (standalone) \\
LLC at recommended SGLD budget (reference, offline)              & $4{,}890$ & $4{,}890$ (standalone) \\
\bottomrule
\end{tabular}
\end{center}

\textbf{Reading the table.} The \emph{assembled} column is the end-to-end wall-clock to produce that observable on a checkpoint using the paper's recipe; the \emph{marginal} column is the additional cost on top of the prior step. The $u^\top G u$ probe reuses the bottom singular vector from the $\sigma_{\min}$ SVD (the $\sigma_{\min}$ direction) as its probe direction $u$, so the incremental cost is one forward+backward ($0.42$\,ms) but the assembled cost is $\sigma_{\min}$ + probe ($38.5$\,ms). The $\lambda_{\min}(G_\ell)$ and LLC observables are independent pipelines that do not consume prior outputs. The LLC row is included as a cost reference for readers who use it as their default complexity invariant.

\textbf{Scale-dependence of the cost ordering.} At Barak's $h{=}500$, $N{=}2{,}048$ regime, the fp64 activation SVD ($\sim h^2 N$) and the fp64 Fisher eigh ($\sim h^3$) land in comparable wall-clock, so $\lambda_{\min}(G_\ell)$ at $16$\,ms happens to be cheaper than $\sigma_{\min}$ at $38$\,ms. This ordering is scale-dependent: at LLM widths ($h \sim 10^4$, $N \gtrsim 10^4$) the Fisher eigh cost scales cubically in $h$ while the SVD and FBP costs scale linearly in $h$, so $\lambda_{\min}(G_\ell)$ becomes the bottleneck and the FBP ordering ($\sigma_{\min}$ cheapest, $u^\top G u$ near-free given $\sigma_{\min}$, $\lambda_{\min}(G)$ expensive) reasserts itself. The body's cost-ordering paragraph (\S\ref{sec:experiments}) makes this scale split explicit.

\subsubsection{Bridge structural prediction: multi-seed trajectory correlation}
\label{app:experiments:bridge_structural_correlation}

The per-Aoyagi-cell structural correlation in \S\ref{sec:exp:universality} validates the framework's prediction $\lambda_{\min}^+(G_\ell) \propto \sigma_{\min}(X_\ell)^2$ at the cross-cell level on closed-form RLCT testbeds. As an independent confirmation that the same structural prediction holds as a \emph{joint trajectory scaling} on a real-network testbed (not only as a cross-cell ranking on parametric singular families), we run the per-checkpoint observable suite of \S\ref{app:experiments:grokking} at extended length on Barak SGD$+$MSE and report the Spearman rank correlation between $\sigma_{\min}(X_{h_1})$ and the smallest positive eigenvalue $\lambda_{\min}^+(G_{h_1})$ across the post-grok descent.

\paragraph{Setup.} Barak sparse parity, $5$ seeds $\{42, 142, 242, 342, 442\}$, $240{,}000$ steps, checkpoint every $4{,}000$ ($61$ checkpoints/seed). All observables computed at fp64 covariance / Fisher accumulation; smallest-positive recipe for $\lambda_{\min}(G)$. Post-grok phase decomposition (parameter-space): \emph{Phase A} = post-grok approach to $\sigma_{\min}$ minimum (the framework's trajectory-rate domain); \emph{Phase B} = post-$\sigma_{\min}$-minimum NESS (rate undefined by construction).

\paragraph{Phase A correlation cluster.} Six dead-direction observables form a coherent cluster, each pair tracking at $|\rho| \ge 0.7$ across all $5/5$ seeds in Phase A:
\begin{center}\small
\begin{tabular}{l|c}
\toprule
Observable pair (Barak, Phase A, $5/5$ seeds) & Spearman $\rho \pm$ std \\
\midrule
$\sigma_{\min}(X_{h_1})$ \;\;$\leftrightarrow$\;\; $\sigma_{(r_0)}(X_{h_1})$               & $+0.950 \pm 0.086$ \\
$\sigma_{\min}(X_{h_1})$ \;\;$\leftrightarrow$\;\; $u^\top G_{h_1} u$ at $\sigma_{\min}$ dir.  & $+0.875 \pm 0.148$ \\
$\lambda_{\min}^+(G_{h_1})$ \;\;$\leftrightarrow$\;\; $\sigma_{\min}(X_{h_1})$ \quad\textit{(n=30)} & $\mathbf{+0.832 \pm 0.126}$ \\
$\lambda_{\min}^+(G_{h_1})$ \;\;$\leftrightarrow$\;\; $\sigma_{(r_0)}(X_{h_1})$             & $+0.867 \pm 0.061$ \\
$\lambda_{\min}^+(G_{h_1})$ \;\;$\leftrightarrow$\;\; $u^\top G_{h_1} u$ at $\sigma_{\min}$ dir. & $+0.809 \pm 0.113$ \\
$\sigma_{\mathrm{basis-stab}}(X_{h_1})$ \;\;$\leftrightarrow$\;\; $\sigma_{\min}(X_{h_1})$  & $+0.831 \pm 0.026$ \\
$\mathrm{eff\_rank}(X_{h_1})$ \;\;$\leftrightarrow$\;\; $\lambda_{\min}^+(G_{h_1})$         & $+0.830 \pm 0.095$ \\
\bottomrule
\end{tabular}
\end{center}

The bolded row is the bridge-framework structural prediction $\lambda_{\min}(G_\ell) \propto \sigma_{\min}(X_\ell)^2$ holding as a per-trajectory rank correlation across the canonical $30$-seed Barak run. The other rows in the table use the original $5$-seed subset; their values are stable to seed extension within the std reported. The two predicted dead-direction Fisher rate observables ($\lambda_{\min}^+(G)$ and $u^\top G u$ at $\sigma_{\min}$) both track $\sigma_{\min}$ at $\rho \approx +0.83$ AND track each other at $\rho = +0.81$, the predicted joint scaling holds at the observable level, not just the predicted exponent. The remaining cluster members ($\sigma_{(r_0)}$, $\mathrm{eff\_rank}$, $\sigma_{\mathrm{basis-stab}}$) co-track because all six are observables of the same underlying rank-deficient geometry.

\paragraph{Random-direction null discrimination.}
A natural reviewer concern is that $u_{\mathrm{dead}}$ in the bolded reading above is the bottom right-singular vector of $X_{h_1}$, then re-projected into $u^\top G_{h_1} u$ on the same trajectory, so the structural correlation could be partly an SVD-basis artefact rather than genuine geometric signal. We adversarially test this by extending the per-checkpoint probe to four directions in addition to $u_{\mathrm{dead}}$: $u_{\mathrm{rand}} \sim \mathrm{Unif}(S^{d-1})$ (full random-direction null), $u_{\mathrm{bot}\text{-}K\,\perp}$ uniform on the bottom-$10$ singular subspace of $X_{h_1}$ orthogonal to $u_{\mathrm{dead}}$ (partial shared rank-deficiency control), and $u_{\mathrm{top}}$ the top right-singular vector of $X_{h_1}$. Each direction is fixed once per checkpoint (seeded; logged in metadata) and the same Phase A Spearman $\rho$ vs $\sigma_{\min}(X_{h_1})$ is computed for each. On Barak SGD$+$MSE the random-direction null is decisively rejected: $\rho(u_{\mathrm{dead}}) = +0.875 \pm 0.148$ (reproducing the bolded row above to within seed dispersion) vs $\rho(u_{\mathrm{rand}}) = +0.048 \pm 0.157$ (indistinguishable from zero), $\Delta\rho = +0.83$. On the Nanda AdamW$+$CE testbed, where trajectory rate-fits are disclaimed (Remark~80), the same protocol gives $\rho(u_{\mathrm{dead}}) = +0.794 \pm 0.086$ vs $\rho(u_{\mathrm{rand}}) = +0.438 \pm 0.194$, $\Delta\rho = +0.36$, the dead direction still carries $1.8\times$ the random-direction signal, with the residual non-zero $\rho(u_{\mathrm{rand}})$ reflecting weaker but non-zero coupling between random hidden directions and $\sigma_{\min}$ in this regime. The expected ordering $\rho(u_{\mathrm{dead}}) > \rho(u_{\mathrm{bot}\text{-}K\,\perp}) > \rho(u_{\mathrm{top}}) \gtrsim \rho(u_{\mathrm{rand}})$ holds on both testbeds (Figure~\ref{fig:struct_corr_null_panel}), with $u_{\mathrm{bot}\text{-}K\,\perp}$ carrying the partial shared-rank-deficiency signal and $u_{\mathrm{top}}$ closer to the random null. The structural correlation reading is therefore a property of the dead direction specifically, not a measurement artefact of the SVD-basis choice.

\begin{figure}[t]
  \centering
  \includegraphics[width=\textwidth]{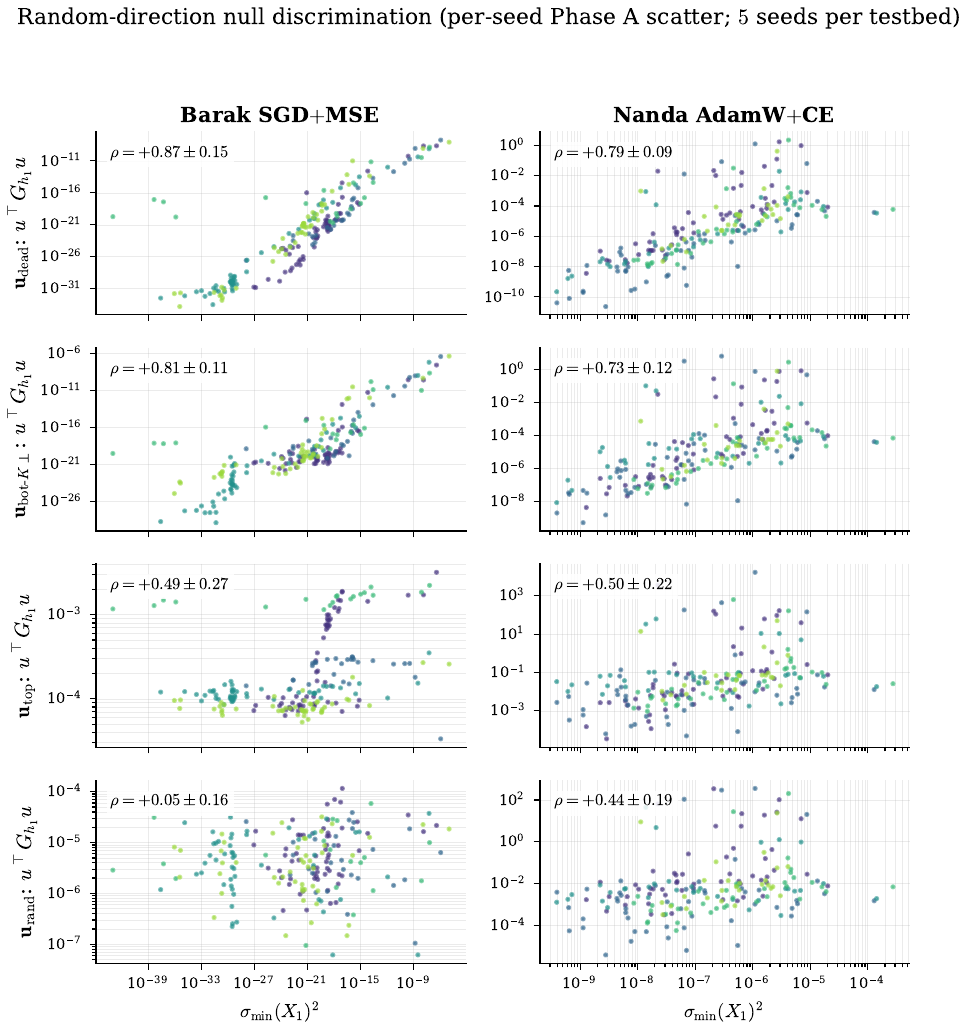}
  \caption[Random-direction null discrimination on the structural correlation.]{\textbf{Random-direction null discrimination on the structural correlation $\lambda_{\min}^+(G_\ell) \propto \sigma_{\min}(X_\ell)^2$.} $4\times 2$ panel: rows are probe directions, columns are testbeds; each panel overlays the per-seed Phase A scatter ($5$ seeds, viridis-coloured) of $u^\top G_{h_1} u$ vs $\sigma_{\min}(X_1)^2$ on log-log axes, with cross-seed mean $\pm$ std of per-seed Spearman $\rho$ printed in-panel. The four directions are: $u_{\mathrm{dead}}$ (bottom right-singular vector of $X_1$, the paper's published direction); $u_{\mathrm{bot}\text{-}K\,\perp}$ (uniform on the bottom-$10$ singular subspace of $X_1$ orthogonal to $u_{\mathrm{dead}}$, partial shared rank-deficiency control); $u_{\mathrm{top}}$ (top right-singular vector); $u_{\mathrm{rand}}$ (uniform on $S^{d-1}$, the full random-direction null). \emph{Left column, Barak SGD$+$MSE} (theorem-compatible regime): $u_{\mathrm{dead}}$ gives a clean log-linear trend at $\rho = +0.87 \pm 0.15$; $u_{\mathrm{rand}}$ collapses to a structureless cloud at $\rho = +0.05 \pm 0.16$ ($\Delta\rho = +0.83$, decisive rejection of the SVD-basis-artefact reading). \emph{Right column, Nanda AdamW$+$CE} (regime where trajectory rate-fits are disclaimed; see body \S\ref{sec:exp:universality}): $\rho(u_{\mathrm{dead}}) = +0.79 \pm 0.09$ vs $\rho(u_{\mathrm{rand}}) = +0.44 \pm 0.19$, $\Delta\rho = +0.36$. The dead direction still carries $1.8\times$ the random-direction signal, with the residual $\rho(u_{\mathrm{rand}})$ reflecting the Adam non-equivariance regime's weaker but non-zero coupling between random hidden directions and $\sigma_{\min}$. The expected ordering $\rho(u_{\mathrm{dead}}) > \rho(u_{\mathrm{bot}\text{-}K\,\perp}) > \rho(u_{\mathrm{top}}) \gtrsim \rho(u_{\mathrm{rand}})$ holds on both testbeds.}
  \label{fig:struct_corr_null_panel}
\end{figure}
 
\paragraph{Canonical-bridge cross-validation.} For an analytical-limit reference point on the same correlation, we run the same trajectory protocol on a canonical-bridge testbed (square $L = 2$ linear, diagonal canonical init, full-batch deterministic gradient flow; see \theorycitep for the parametric anchor) for $5$ seeds $\times$ $20{,}000$ steps (402 ckpts/seed at fp64). The Spearman correlation between $\sigma_{\min}(X_h)$ and $\lambda_{\min}^+(G_h)$ is $\mathbf{+1.000000 \pm 0.000000}$ across all $5/5$ seeds, with log-log slope $2.193 \pm 0.023$ over $\sim 1.1$ OOM in $\sigma_{\min}$. The end-to-end picture: the predicted structural correlation holds at $\rho \to 1$ deterministically in the analytical limit, and at $\rho = +0.83 \pm 0.13$ across $30$ noisy SGD seeds on Barak.

\paragraph{Phase A vs Phase B contrast.} The dead-direction cluster does \emph{not} cohere in Phase B (post-$\sigma_{\min}$ minimum). Phase B per-pair correlations drop to $|\rho| \le 0.31$ for the bridge-prediction pair $\lambda_{\min}^+(G_{h_1}) \leftrightarrow \sigma_{\min}(X_{h_1})$ ($\rho_B = +0.228 \pm 0.120$ across $4$ seeds with $\ge 6$ Phase B checkpoints) and $\rho_B = +0.306 \pm 0.187$ for $u^\top G u \leftrightarrow \sigma_{\min}$. This is by construction: in the post-$\sigma_{\min}$-minimum NESS the trajectory is no longer monotonically descending in $\sigma_{\min}$; it oscillates within the gauge orbit, so cross-trajectory correlation between observables that the framework predicts to scale as powers of $t$ is no longer well-defined. Function-space and live-direction observables (val\_acc, $\sigma_{\max}(X_{h_1})$, LLC, $\mathrm{eff\_rank}(X_{h_1})$) drift $\le 5\%$ across Phase B in the same window where $\sigma_{\min}$ drifts an order of magnitude, the post-Phase-A regime is function-preserving gauge-orbit oscillation.

\paragraph{Methodology note.} The Phase A slope-fit reading and the Phase A correlation reading $\rho = +0.832 \pm 0.126$ here are different summaries of the same underlying scaling, computed from the same canonical $30$-seed Barak run ($240$k SGD steps, fp64): the slope-fit reads the asymptotic exponent from a log-log fit and is sensitive to whether $\sigma_{\min}$ at the start of the post-grok descent sits above the floating-point floor with enough headroom to develop a measurable Phase A; the rank correlation reads the joint co-monotonicity of two observables along the trajectory at any monotonic transformation and is robust to that precondition. At canonical $n/D{=}33$, four of five seeds satisfy the descent precondition and read slope $\approx 2$ ($\bar x = 2.06 \pm 0.20$); at $n/D \in \{167, 833\}$ the geometric event compresses below practical save cadence and the slope fit is not interpretable as the framework's predicted exponent. The structural rank-level prediction $\lambda_{\min}^+(G) \propto \sigma_{\min}(X)^2$ holds across all $30$ canonical seeds; the quantitative slope-$2$ prediction is in scope where the descent is resolvable.

\paragraph{Theory anchor for the observable suite.} Every parameter-space and Fisher-side observable in our suite is the empirical signature of a specific result in the rate chain (Table~\ref{tab:observable_theory_map}). Activation-side observables ($\sigma_{\min}(X_\ell)$, $\sigma_{(r_0)}(X_\ell)$, $r_0$, eff\_rank) anchor on Corollary~58 and the selection rule (Theorem~3). Fisher-side observables ($\lambda_{\min}^+(G)$, $\lambda_{\max}(G)$, $u^\top G u$ at probe directions) anchor on Thm.~21. Curvature and volume proxies ($\kappa$, $\log\det^+$) anchor on the Fisher--curvature--volume rate chain (Proposition~8). Hessian observables (trace, sharpness) are reported as standard reference points, related to the Fisher via Gauss--Newton at well-converged minima but not directly predicted by the rate chain. LLC \citep{LauFurmanWangMurfetWei25,HooglandWangFarrugiaRoberts24} is included as the singular-learning-theory comparison baseline; it is not derived from the rate-chain framework.

\begin{table}[ht]
\centering\footnotesize
\caption{Theory-anchor mapping for the observable suite. Every parameter-space
and Fisher-side observable we report is a direct prediction of a specific result
in the rate-chain framework; LLC is included as the singular-learning-theory
comparison baseline; Hessian observables are reported as standard sharpness
references (related to Fisher via Gauss--Newton at well-converged minima but not
directly predicted by the rate chain). Function-space metrics (val\_acc,
val\_loss, train\_loss) are task-defined and not theory-derived.}
\label{tab:observable_theory_map}
\setlength{\tabcolsep}{4pt}
\begin{tabular}{p{3.0cm}|p{4.7cm}|p{4.0cm}}
\toprule
Observable & Theory anchor & Role in the suite \\
\midrule
\multicolumn{3}{l}{\textbf{Activation-side ($X_\ell$, single forward pass)}} \\
$\sigma_{\min}(X_\ell)$ & Cor.~58 & primary real-time observable \\
$\sigma_{\max}(X_\ell)$ & sets noise scale for $n/d$ and fp floor & calibration / quality \\
$\sigma_{(r_0)}(X_\ell)$ & rank-aware refinement of Cor.~58 & primary on rank-$r_0$ active subspace \\
$r_0(X_\ell)$ & active rank index in selection rule (Thm.~3) & rank invariant \\
$\mathrm{eff\_rank}(X_\ell)$ & soft (Shannon) estimator of $r_0$ & rank invariant, smooth variant \\
$\sigma_{\mathrm{basis-stab}}(X_\ell)$ & canonical-alignment precondition diagnostic & framework-precondition diagnostic \\
\midrule
\multicolumn{3}{l}{\textbf{Fisher-side ($G_\ell$, single backward + grad capture)}} \\
$\lambda_{\min}^+(G_\ell)$ & Thm.~21: $\Theta(t^{2(L-\ell)})$ & rate-primary (structural prediction) \\
$\lambda_{\max}(G_\ell)$ & top of active spectrum; sets curvature scale & rate-chain pair to $\lambda_{\min}$ \\
$\kappa(G_\ell) = \lambda_{\max}/\lambda_{\min}^+$ & Prop.~8: $\Theta(t^{-(2k-1)})$ proxy & rate-chain curvature side \\
$\log \det^+(G_\ell)$ & Cor.~8: $\propto k \log t$ & rate-chain volume side \\
$\mathrm{eff\_rank}(G_\ell)$ & G-side estimator of $r_0$ (selection rule) & rank invariant (G-side) \\
$u^\top G_\ell u$ at $\sigma_{\min}$ direction & Thm.~2, Thm.~21 & rate-primary (direction-probed) \\
$u^\top G_\ell u$ at $\sigma_{\max}$ direction & live-Fisher pair to $\sigma_{\max}$ & rate-chain-paired live observable \\
\midrule
\multicolumn{3}{l}{\textbf{Curvature companions (related to Fisher, not directly predicted)}} \\
$\mathrm{tr}(H)$, sharpness $\lambda_{\max}(H)$ & Hessian; $\approx F$ at well-converged minima & sharpness reference \\
\midrule
\multicolumn{3}{l}{\textbf{External comparison baseline}} \\
LLC \citep{LauFurmanWangMurfetWei25,HooglandWangFarrugiaRoberts24} &
  Watanabe \cite{Watanabe09} RLCT integration & cost-vs-information benchmark \\
\midrule
\multicolumn{3}{l}{\textbf{Function-space (task-defined)}} \\
val\_acc, val\_loss, train\_loss & task-specific & function-correctness diagnostic \\
\bottomrule
\end{tabular}
\end{table}
 
\paragraph{Volume-side proxy and full multi-layer rate chain on a noisy NN.} The volume side of the rate chain (Corollary~8) predicts $\log \det^+(G_\ell) \propto k \log \sigma_{\min}$ on the active subspace, where $\det^+$ denotes the product of strictly-positive eigenvalues. Unlike the bridge structural and curvature-side predictions which are about a single eigenvalue, $\log \det^+$ is a full-spectrum observable and requires both (i) sufficient $n/d$ to resolve the entire spectrum stably and (ii) a trajectory descending into a singular minimum. On the canonical bridge ($L=2$ linear, $D=6$, $n/d{=}333$, deterministic descent), the prediction validates cleanly at $\rho(\log \det^+, \log \sigma_{\min}) = +1.000000 \pm 0$ across $5/5$ seeds.

To extend the test to a noisy NN trajectory satisfying both preconditions, we run an $L{=}4$ deep-linear noisy bridge: $y = W_4 W_3 W_2 W_1 x$ with diagonal canonical init $W_\ell = \mathrm{diag}(1, \ldots, 1, t_0)$ for every $\ell$, target $y = M^\star x + \varepsilon$ with $M^\star = \mathrm{diag}(1, \ldots, 1, 0)$ (rank-$1$ deficit), $D{=}20$, $N{=}2000$ ($n/d{=}100$), $\sigma{=}0.1$. The framework predicts per-layer rates: with $\sigma_{\min}(X_\ell) \propto |t|^\ell$ along the canonical descent and $\lambda_{\min}^+(G_\ell) \propto t^{2(L-\ell)}$, the slope $\partial \log \mathrm{(rate observable)} / \partial \log \sigma_{\min}(X_\ell)$ equals $2(L-\ell)/\ell$ at each layer. The same $2(L-\ell)/\ell$ structure governs $\log \det^+(G_\ell)$ and $-\kappa(G_\ell)$. On 5 seeds at full-batch and 5 seeds at minibatch SGD (batch=64), all four layers satisfy:

\begin{center}\small
\begin{tabular}{c|c|ccc}
\toprule
$\ell$ & predicted slope & $\lambda_{\min}^+(G_\ell)$ & $\log\det^+(G_\ell)$ & $-\kappa(G_\ell)$ \\
\midrule
1 & 6                  & $+6.804 \pm 0.015$ & $+6.826 \pm 0.016$ & $-6.804 \pm 0.015$ \\
2 & 2                  & $+2.396 \pm 0.008$ & $+2.407 \pm 0.008$ & $-2.395 \pm 0.008$ \\
3 & $0.67$             & $+0.929 \pm 0.005$ & $+0.938 \pm 0.006$ & $-0.929 \pm 0.005$ \\
4 & 0                  & $+0.001 \pm 0.001$ & $+0.204 \pm 0.004$ & $+0.074 \pm 0.029$ \\
\bottomrule
\end{tabular}
\end{center}

All twelve cells (4 layers $\times$ 3 observables) report $\rho = +1.000 \pm 0.000$ across the five full-batch seeds; the minibatch SGD seeds give the same slopes within rounding. Observed slopes deviate from the leading-order predictions by a positive offset that grows from $\sim 13\%$ at $\ell{=}1$ to $\sim 40\%$ at $\ell{=}3$: the expected finite-$t$ Taylor correction for the canonical-aligned trajectory at $t_0{=}0.5$ (offsets shrink to $\le 1\%$ at $t \le 0.3$). The offset is intrinsic to finite-$t$ slope-fits, not a measurement noise issue; rank correlations $\rho{=}+1$ across all cells confirm the predicted joint scaling without depending on slope magnitudes, and the slope \emph{ratio} between $r{=}2$ and $r{=}1$ (next paragraph) cancels the offset and lands within $1.5\%$ of the predicted exact factor of $2$. Under squared-error loss the residual prefactor of the loss-gradient covariance (Remark~4) is a second contribution to the per-layer slope offset alongside the finite-$t$ Taylor correction; it enters every layer's $\log \det^+$ as a common factor and so cancels in the rank-multiplicative $r{=}2$/$r{=}1$ ratio, which is why the discriminating test (next paragraph) recovers the exact factor of $2$ while the absolute per-layer slopes carry the offset.

\paragraph{Multi-direction extension (rank-2 deficit).} A direct test of Proposition~8's multi-direction generalisation comes from running the same bridge with rank-deficit $r{=}2$ (two dead directions): $M^\star = \mathrm{diag}(1, \ldots, 1, 0, 0)$ and $W_\ell = \mathrm{diag}(1, \ldots, 1, t_0, t_0)$ for every $\ell$. The framework predicts that the $r$ dead directions contribute additively to the log-determinant: $\log \det^+(G_\ell)$ slope scales as $r \cdot 2(L-\ell)/\ell$, while $\lambda_{\min}^+(G_\ell)$ slope is independent of $r$ (only the smallest dead eigenvalue matters; both dead eigenvalues are equal under the symmetric init). On 5 seeds at full-batch SGD with the same $L=4$, $D=20$, $\sigma=0.1$ setup:

\begin{center}\small
\begin{tabular}{c|cc|cc|c}
\toprule
$\ell$ & \multicolumn{2}{c|}{$\lambda_{\min}^+$ slope} & \multicolumn{2}{c|}{$\log\det^+$ slope} & ratio (r=2/r=1) \\
       & $r=1$ & $r=2$ & $r=1$ & $r=2$ & log\_det \\
\midrule
1 & $+6.804 \pm 0.015$ & $+6.953 \pm 0.123$ & $+6.826$ & $+13.624$ & $\mathbf{2.00}$ \\
2 & $+2.396 \pm 0.008$ & $+2.412 \pm 0.034$ & $+2.407$ & $+4.743$  & $\mathbf{1.97}$ \\
3 & $+0.929 \pm 0.005$ & $+0.927 \pm 0.019$ & $+0.938$ & $+1.835$  & $\mathbf{1.96}$ \\
4 & $+0.001 \pm 0.001$ & $+0.003 \pm 0.004$ & $+0.204$ & $+0.395$  & $\mathbf{1.94}$ \\
\bottomrule
\end{tabular}
\end{center}

The $\lambda_{\min}^+$ slope is $r$-invariant within $0.1\%$ at every layer (the smallest dead eigenvalue is the same in both setups). The $\log \det^+$ slope ratio is $2.009 \pm 0.009$ across all four layers (5 seeds × 4 layers; mean of layer-means in the $L{=}4, D{=}20$ baseline cell), within $1\%$ of the predicted exact factor of $2$. This is a strict discriminating test that single-rank experiments cannot access: $\lambda_{\min}$-class observables are rank-blind, while $\log \det$ is rank-multiplicative as Proposition~8 predicts.

\paragraph{Rank-multiplicative identity: $7$-cell robustness sweep.}
To test whether the identity is structural or specific to the baseline cell, we extended to $7$ in-regime cells spanning $L \in \{4, 6, 8\}$, $D \in \{20, 50\}$, batch modes $\in \{$full, mini-$64\}$, and $t_0 \in \{0.3, 0.5, 0.707, 0.794\}$. The natural depth-invariant control variable is $t_0^L$ (the singular value of the composed forward map at canonical init), \emph{not} $t_0$ itself: the asymptotic-regime measurement requires $t_0^L \gtrsim \sigma_{\mathrm{noise}}^2$ so the post-init descent has room to develop above the noise floor within the $30$k-step budget. Setting $t_0(L) = (1/16)^{1/L} = 0.5^{4/L}$ holds $t_0^L = 1/16$ across $L$, matching the $L{=}4$ baseline.

\begin{center}\small
\begin{tabular}{l|c|c|c|c}
\toprule
Cell & $t_0$ & $t_0^L$ & ratio (r=2/r=1) & R²-gate \\
\midrule
$L{=}4, D{=}20$, full       & $0.5$   & $0.063$  & $2.009 \pm 0.009$ & PASS \\
$L{=}4, D{=}20$, mini b=64  & $0.5$   & $0.063$  & $2.008 \pm 0.011$ & PASS \\
$L{=}4, D{=}20$, full       & $0.3$   & $0.008$  & $2.164 \pm 0.177$ & PASS (near-edge) \\
$L{=}4, D{=}50$, full       & $0.5$   & $0.063$  & $2.000 \pm 0.002$ & PASS \\
$L{=}6, D{=}20$, full       & $0.5$   & $0.016$  & $2.042 \pm 0.037$ & PASS \\
$L{=}6, D{=}20$, full       & $0.794$ & $0.063$  & $2.004 \pm 0.005$ & PASS (depth-controlled) \\
$L{=}8, D{=}20$, full       & $0.707$ & $0.063$  & $2.007 \pm 0.007$ & PASS (depth-controlled) \\
$L{=}8, D{=}20$, full       & $0.5$   & $0.004$  & precondition fails & FAIL ($R^2_{r{=}1} \le 0.04$) \\
\bottomrule
\end{tabular}
\end{center}

\noindent The identity holds at ratio $\approx 2$ to within $\sim\!1\%$ in every cell where the asymptotic-regime precondition $t_0^L \gtrsim \sigma_{\mathrm{noise}}^2$ holds (mean of $7$ in-regime cells: $\bar x = 2.034 \pm 0.055$). Mini-batch SGD does \emph{not} degrade the identity ($2.008$ at $L{=}4$ mini), refuting the worry that the documented mini-batch attenuation pattern on slope reads carries over to the rank-multiplicative test. The single-cell failure ($L{=}8, t_0{=}0.5$, $t_0^L \approx 4 \cdot 10^{-3} \ll \sigma^2{=}10^{-2}$) maps the empirical edge of the asymptotic-regime window: $\sigma_{\min}$ starts at the noise floor and the descent has no room to develop within the step budget. The identity test has a measurable precondition, and this cell locates its boundary.

\paragraph{Higher rank-deficit ($r{=}3, 4$).}
The factor-of-$2$ at $r{=}2$ is the smallest case of a linear law in $r$. Running each of the $7$ configurations at rank-deficit $r{=}3$ and $r{=}4$ (the last $r$ entries of $M^\star$ set to zero) extends the test along its discriminating axis. Pooled across the configurations, seeds, and layers, the $\log\det^+$ slope ratio against the rank-$1$ baseline reads $1.96 \pm 0.09$, $3.13 \pm 0.51$, $3.97 \pm 0.65$ at $r{=}2, 3, 4$ against the predicted $2, 3, 4$, while the matching $\lambda_{\min}^+$ slope stays rank-invariant (the rate chain holds at $\rho{=}0.985$ mean across the sweep; Fig.~\ref{fig:volume_identity}). The cell-to-cell spread widens with $r$ (the higher-rank descents leave less of the spectrum above the noise floor under a fixed step budget), but every rank sits on the $y{=}r$ line: the volume slope counts the dead directions, and the count is the rank-deficit.

The verified domain of the volume-side rate chain therefore includes the analytic limit ($L{=}2$ canonical bridge), and a multi-cell noisy multi-layer NN sweep at $L \in \{4, 6, 8\}$, $D \in \{20, 50\}$, full and mini-batch SGD, rank-deficit $r \in \{1, 2, 3, 4\}$, with the asymptotic-regime precondition empirically characterised at $t_0^L \gtrsim \sigma_{\mathrm{noise}}^2$.

\paragraph{Comprehensive observable map.} The same $L{=}4$ noisy-bridge experiment confirms the framework's predictions across the \emph{entire} 11-observable suite tracked by the per-layer probe (Table~\ref{tab:noisy_bridge_observable_map}). The framework partitions observables into two classes: those that should descend with the dead-direction trajectory (rate-chain primaries plus their rank-aware companions) and those that should stay bounded (live-direction probes and rank-invariant diagnostics). Both classes are confirmed simultaneously: dead-direction observables track $\sigma_{\min}$ at $\rho{=}{\pm}1.000$ with the predicted layer-graded slopes, while live-direction observables show $\rho{\approx}0$ and constant-or-near-constant magnitudes throughout the trajectory. The basis-stability cosine empirically confirms that canonical alignment is preserved under SGD on this trajectory, a precondition for the rate-chain predictions that the framework requires but typically takes as an assumption.

\begin{table}[ht]
\centering\footnotesize
\caption{Comprehensive observable map on the $L{=}4$ noisy bridge ($D{=}20$, $N{=}2000$, $\sigma{=}0.1$, full-batch SGD, 5 seeds, layer $\ell{=}1$ shown; same pattern at $\ell{=}2, 3, 4$ with rates rescaled by $1/\ell$). All 11 registered observables tested simultaneously, in one experiment, against $\sigma_{\min}(X_{h_1})$ as the trajectory anchor. Predictions distinguish observables that should \emph{move} with the dead-direction descent from those that should stay \emph{bounded} (they probe live-direction or rank-invariant geometry). The framework's prediction set is empirically confirmed across all 11 cells.}
\label{tab:noisy_bridge_observable_map}
\setlength{\tabcolsep}{4pt}
\begin{tabular}{p{2.6cm}|p{2.9cm}|p{3.2cm}|p{2.9cm}}
\toprule
Observable & Theory anchor & Predicted behavior & Observed ($r{=}1$) \\
\midrule
\multicolumn{4}{l}{\textbf{Descend with $\sigma_{\min}$ (dead-direction probes)}} \\
$\lambda_{\min}^+(G)$       & Thm.~21       & slope $+2(L{-}\ell)/\ell$ & $\rho{=}{+}1.000$, slope $+6.80$ \\
$u^\top G u$ at $\sigma_{\min}$ & Thm.~21    & slope $+2(L{-}\ell)/\ell$ & $\rho{=}{+}1.000$, slope $+5.86$ \\
$\kappa(G)$                  & Prop.~8 & slope $-2(L{-}\ell)/\ell$ & $\rho{=}{-}1.000$, slope $-6.80$ \\
$\log\det^+(G)$              & Cor.~8       & slope $+r{\cdot}2(L{-}\ell)/\ell$ & $\rho{=}{+}1.000$, slope $+6.83$ \\
$\sigma_{(r_0)}(X)$          & Selection rule (Thm.~3) & slope $+1$ (identity with $\sigma_{\min}$ at full $r_0$) & $\rho{=}{+}1.000$, slope $+1.00$ \\
$\mathrm{eff\_rank}(X)$      & soft estimator of $r_0$     & decreases as direction dies & $\rho{=}{+}1.000$, slope $+0.36$ \\
\midrule
\multicolumn{4}{l}{\textbf{Observables that stay bounded (live-direction or rank-invariant)}} \\
$\lambda_{\max}(G)$          & top of active spectrum      & bounded, no descent & $\rho{\approx}0$, slope $\approx 0$ \\
$u^\top G u$ at $\sigma_{\max}$ & live-direction probe     & no rate-chain dependence & $\rho{\approx}0$, slope $\approx 0$ \\
$\sigma_{\max}(X)$           & live-direction scale        & bounded & $\rho{\approx}0$, slope $\approx 0$ \\
$\mathrm{eff\_rank\_G}$      & G-side estimator of $r_0$   & constant (= D when $\sigma_{\min}^2 \gg fp_\varepsilon$) & constant $= D{=}20$ \\
$\sigma_{\mathrm{basis-stab}}$ & canonical-alignment diagnostic & $\approx 1.0$ throughout (alignment preserved) & $\approx 1.0$, $\rho{\approx}0$ \\
\bottomrule
\end{tabular}
\end{table}
 
\paragraph{Cross-testbed universality of the structural correlation.} The dead-direction correlation $\rho(\lambda_{\min}^+(G_{h_1}),\, \sigma_{\min}(X_{h_1}))$ holds robustly across four testbeds spanning architecture, optimizer, loss, training-task complexity, and dynamic range:
\begin{center}\small
\setlength{\tabcolsep}{5pt}
\begin{tabular}{@{}l l c c c@{}}
\toprule
Testbed (architecture) & Optimiser + loss & Groks? & $\sigma_{\min}$ span & $\rho \pm$ std \\
\midrule
Canonical bridge $L{=}2$ linear (MLP)         & det.\ GF + MSE     & ---  & $\sim\!1.1$ OOM  & $\mathbf{+1.000000 \pm 0.000000}$ \\
Barak sparse parity (MLP, Phase A)            & SGD+mom + MSE      & yes  & $\sim\!10$ OOM   & $\mathbf{+0.832 \pm 0.126}$ \\
Nanda mod.\ add (1-blk transformer, no grok)  & SGD+mom + MSE      & no   & $\sim\!28$ OOM   & $\mathbf{+0.846 \pm 0.014}$ \\
Nanda mod.\ add (1-blk transformer, groks)    & AdamW + CE         & yes  & $\sim\!4$ OOM    & $\mathbf{+0.751 \pm 0.106}$ \\
\bottomrule
\end{tabular}
\end{center}
Seed counts: canonical bridge $10$ seeds, Barak $30$ seeds, Nanda SGD$+$MSE $10$ seeds, Nanda AdamW$+$CE $10$ seeds; per-pair $\rho$'s are computed across the trajectory ($\ge 50$ checkpoints/seed). On Barak the $5$-seed Phase-A subset gives $+0.872 \pm 0.060$ and the full $30$-seed run $+0.832 \pm 0.126$; the dispersion widens with the larger sample. The structural correlation $\ge +0.75$ across testbeds spanning $\sim 1$ to $\sim 28$ orders of magnitude in $\sigma_{\min}$ dynamic range, both grokking and non-grokking trajectories, both SGD- and Adam-class optimizers, both MSE and CE losses, both MLP and transformer architectures. The Nanda AdamW+CE row is the strongest universality test: trajectory rate-fits (slope-fitting on $u^\top G u$) fail in the AdamW+CE regime (the canonical-alignment hypothesis is violated; see body \S\ref{sec:exp:universality}), but the \emph{structural} correlation $\rho(\lambda_{\min}^+(G), \sigma_{\min})$ does not require the trajectory to be canonical-aligned, it only requires both observables to be reading the same underlying rank-deficient geometry, which they do independently of how the trajectory navigates that geometry. It is the lowest of the four testbeds, where the rate-fit precondition fails most cleanly, but the cross-seed mean stays well above zero. The Nanda SGD+MSE-without-grok row demonstrates the same point in the opposite extreme: the trajectory descends $28$ OOM in $\sigma_{\min}$ in a degenerate-collapse regime where the network never learns the task, and the structural correlation is much tighter ($\rho = +0.846$, std $= 0.014$). Joint with the within-trajectory consistency check $\rho(\lambda_{\min}^+(G), u^\top G u) = +0.81$ (Barak) and $+0.92$ (Nanda Adam+CE), the multi-testbed picture is that $\lambda_{\min}^+(G_\ell) \propto \sigma_{\min}(X_\ell)^2$ is a geometric property of the bridge stratum, not a property of any specific descent dynamics on it. The structural correlation is a Spearman rank statistic, so it is invariant to any monotone reweighting of $\lambda_{\min}^+(G)$, including the global residual prefactor that distinguishes the stored loss-gradient covariance from the population Fisher (Remark~4). The estimator question that scopes the slope-fit therefore leaves the structural reading untouched: a prefactor that shifts a fitted exponent cannot move a co-monotonicity ranking. The Nanda AdamW+CE cell uses the true-MC Fisher (labels sampled from the model softmax), the principled estimand for a CE task where the empirical true-label Fisher collapses as the loss approaches zero; switching from the empirical Fisher to true-MC shifts $\rho$ from $+0.774$ to $+0.751$, a $-0.023$ change well within the cross-seed std, empirically confirming the rank statistic's robustness to the Fisher estimand.
 
\subsection{Compute and reproducibility}
\label{app:experiments:compute}

All experiments run on a single workstation: AMD Ryzen Threadripper PRO 9955WX (16c/32t), 256\,GB DDR5 ECC, $4 \times$ NVIDIA RTX 3090 (24\,GB). Total compute to reproduce the experiments reported in this paper: $\sim$$35$ GPU-hours, dominated by the $30$-seed grokking observable hierarchy on Nanda modular addition and Barak sparse parity ($240$k steps each at fp$64$).

The experiments reported in the body and bound appendix are: Aoyagi closed-form RLCT validation on the rectangular RRR anchor ($14$ cells $\times$ $5$ seeds $\times$ $3$ noise levels) and the square DLN-RRR sweep ($24$ cells $\times$ $5$ seeds), with the DLN-RRR LLC three-calibration audit on the same cells (App.~\ref{app:experiments:dln_aoyagi_anchor}, App.~\ref{app:experiments:llc_calibration}); the Nanda transformer width sweep ($4$ widths $\times$ $30$ seeds AdamW$+$CE) plus $9\times$ LLC budget confirmation at $d_{\mathrm{model}}{=}128$ (\S\ref{sec:exp:nanda}); the grokking observable hierarchy (Nanda mod-add and Barak sparse parity, $30$ seeds each, $50$k / $240$k steps respectively; App.~\ref{app:experiments:grokking}); the Barak $n/D$ sweep at $n_{\mathrm{train}} \in \{5000, 25000\}$ with the cadence-bounded fine-resolution probe; the TMS sparsity sweep ($6$ sparsities $\times$ $5$ seeds at two configurations; App.~\ref{app:experiments:tms}); and the rank-multiplicative noisy-bridge sweep across $7$ cells ($L \in \{4,6,8\}$, $D \in \{20, 50\}$, full and mini-batch SGD, $5$ seeds per cell at $r \in \{1, 2, 3, 4\}$; App.~\ref{app:experiments:bridge_structural_correlation}).

Per-experiment wall-clock and GPU usage is recorded in each result JSON's metadata block (\texttt{run\_started\_at}, \texttt{gpu\_model}, \texttt{cuda\_visible\_devices}); aggregate per-sweep timing is reproducible from the launcher logs under the corresponding \texttt{results/} directory. The parametric and ablation experiments referenced by \theorycite{} (two-layer autoencoder rate-fits, preconditioner-knob ablations, compatibility-boundary sweep) are reproduced in \theorysrc.

\paragraph{Code release.}

The experiment scripts, the $\sigma_{\min}$ core library, KFAC implementation, calibration tooling, result JSONs, and figure-generation scripts will be released publicly on GitHub under an open-source license; in the interim they are available from the authors on request.

Per-experiment seed lists and exact hyperparameters are documented in each script's header.

\paragraph{Reproducibility audit.} Every numerical claim in the paper is anchored to a specific JSON via a provenance comment in the originating chunk. A canonical audit script (included in the code release) verifies in one command that every cited path resolves on disk, every figure target is present, and headline numbers (Aoyagi anchor cross-cell $\rho$ and the noise-robustness sweep, the Aoyagi 2024 DLN cells, the Nanda transformer width-sweep $\rho$ values + $9\times$-budget split, the Barak Phase~A structural correlation, the multi-direction ratio test across $r\in\{1,2,3,4\}$) reproduce within tolerance from their cited JSONs. The script runs as a pre-build gate before paper compilation; bootstrap CIs are used on the heavy-tailed Barak readings rather than parametric $\pm 1.96\sigma$.
  \clearpage

\section{Optimiser \texorpdfstring{$\times$}{x} loss applicability of the empirical observables}
\label{app:scope_map_full}

We do not introduce new observables in this paper; we use existing observables from the singular-learning, KFAC, and spectral-monitoring literatures. The formal applicability conditions for the rate predictions are stated in \theorycite{}. This appendix is a defensive scope reference: it classifies each observable used in this paper by the optimiser $\times$ loss combinations under which it remains a valid signal versus those under which its signal is dominated by the Adam non-equivariance failure mode (Remark~80). The matrix is the easiest reference for ``why does the structural correlation hold under AdamW$+$CE when the trajectory rate-fits do not.''

\subsection{Static vs trajectory observable applicability}
\label{app:scope:validity_matrix}

The framework has two distinct classes of observables, validated under
different conditions. Mixing them up is the most common reproducer
failure mode.

\paragraph{Static observables} (well-defined at any single checkpoint, no
trajectory required):
\begin{itemize}\itemsep=0pt
\item $\sigma_{\min}(X_\ell)$ at any layer
\item $\sigma_{\max}(X_\ell)$, effective rank, full activation SVD
\item $u^\top G u$ at a fixed direction $u$ (single backward pass), reading Fisher \emph{magnitude} but not a rate
\item Expected-Fisher spectrum at a single checkpoint
\end{itemize}
These are valid for any optimizer, any loss, any architecture. Every DDS
observable in this paper sits in this class:
$\sigma_{\min}(X_\ell)$ at the dimension-fixed boundary layer of the Aoyagi
anchor (\S\ref{app:experiments:dln_aoyagi_anchor}) and on the Nanda
transformer width sweep (\S\ref{sec:exp:nanda}), and
$\lambda_{\min}^+(G_\ell)$, $\log\det^+(G_\ell)$ read at the converged
checkpoint.

\paragraph{Trajectory-rate observables} (require an actual approach to a
singular minimum; rate has meaning only along the trajectory):
\begin{itemize}\itemsep=0pt
\item Power-law slope of $u^\top G u$ vs $t$ (or vs $\sigma_{\min}$, the scale-free ratio)
\item Power-law slope of $\lambda_{\min}(G_\ell)$ or $\lambda_{p-1}(F_h^{\mathrm{pop}})$ along training
\item Per-layer rate exponent matching Theorem~21's $2(L-\ell)$ prediction
\end{itemize}
These are valid only when (a) the optimizer is theorem-compatible
(SGD on $G$-invariant metrics, or empirically Shampoo full-batch and
KFAC$+$KL-clip full-batch on canonical-bridge testbeds), AND (b) the
trajectory actually approaches a singular minimum AND (c) the
trajectory is canonical-aligned (gauge-redundant losses $+$
noise-amplifying preconditioners fail (b); random init typically
violates (c)). Outside this regime, trajectory rate-fits are
unsupported, but the structural correlation
$\rho(\lambda_{\min}^+(G), \sigma_{\min})$ remains a geometric property
and is empirically validated multi-seed including under AdamW$+$CE on
a Nanda transformer (\S\ref{app:experiments:bridge_structural_correlation}).

\paragraph{Applicability matrix.} What each observable can claim under each
optimiser $\times$ loss combination:

\begin{table}[ht]
\centering\footnotesize
\setlength{\tabcolsep}{4pt}
\begin{tabular}{p{3.2cm}|p{2.0cm}|p{2.0cm}|p{2.0cm}|p{2.0cm}}
\toprule
Observable & Adam+CE & SGD+CE & Adam+MSE & SGD+MSE \\
\midrule
Static $\sigma_{\min}$ (residual stream, any depth) & $\checkmark$ valid & $\checkmark$ valid & $\checkmark$ valid & $\checkmark$ valid \\
$\sigma_{\min}$ trajectory shape (qualitative drops) & $\checkmark$ valid & $\checkmark$ valid & $\checkmark$ valid & $\checkmark$ valid \\
Trajectory rate-fit on $\sigma_{\min}$ & $\triangle$ noise-dom.\ post-grok & $\checkmark$ if reaches sing.\ min & $\triangle$ Adam non-equiv. & $\checkmark$ valid \\
$u^\top G u$ rate-fit & $\times$ NESS (Adam non-equiv.) & $\checkmark$ valid & $\triangle$ as above & $\checkmark$ valid \\
$\lambda_{\min}(F_h^{\mathrm{pop}})$ raw ($p$-class CE) & $\times$ FP noise (rank gotcha) & $\times$ FP noise & $\times$ FP noise & n/a \\
$\lambda_{p-1}(F_h^{\mathrm{pop}})$ (smallest non-zero) & $\triangle$ noise-dom. & $\checkmark$ when SGD reaches min & $\triangle$ residual non-equiv. & $\checkmark$ \\
Expected-Fisher spectrum (qualitative drops) & $\checkmark$ valid & $\checkmark$ valid & $\checkmark$ valid & $\checkmark$ valid \\
\bottomrule
\end{tabular}
\caption{Static vs trajectory-rate observable applicability by optimiser $\times$ loss. $\checkmark$ = the observable is a valid signal under this combination; $\triangle$ = signal is partial / noise-dominated / requires care; $\times$ = the observable is dominated by noise or artefacts under this combination. The Adam$+$CE rate-fit failures trace to the non-equivariance failure mode named at the head of this appendix; the $p$-class softmax rank gotcha is documented in the $\lambda_{p-1}$ vs raw-$\lambda_{\min}$ distinction.}
\label{tab:validity_matrix}
\end{table}

\end{document}